\documentclass{article}

\usepackage{microtype}
\usepackage{graphicx}
\usepackage{subcaption}
\usepackage{booktabs} 
\usepackage[normalem]{ulem}

\usepackage{hyperref}

\usepackage[accepted]{icml2026}

\usepackage{mathtools}
\usepackage{booktabs} 
\usepackage[T1]{fontenc}    
\usepackage{hyperref}       
\usepackage{url}            
\usepackage{booktabs}       
\usepackage{nicefrac}       
\usepackage{microtype}      
\usepackage{xcolor}         
\usepackage{mathtools}
\usepackage{amsmath,amsfonts,amssymb,dsfont,bm,amsthm}
\usepackage{csquotes}
\usepackage{float}
\usepackage{subcaption}
\usepackage{txfonts}
\usepackage{dsfont}
\usepackage{hyperref}
\usepackage{thm-restate}
\usepackage{soul}
\usepackage{fix-cm}
\usepackage{comment}
\usepackage{soul}

\usepackage{tikz} 
\usetikzlibrary{babel, arrows, decorations, decorations.pathmorphing, decorations.pathreplacing, decorations.shapes, calc,arrows.meta, positioning, decorations.markings,fit,shapes.geometric, backgrounds}
\definecolor{coolorange}{HTML}{F28E00}

\tikzset{
    obs/.style={
        circle, draw=black, 
        minimum size=9mm,
        line width=0.4pt,        
        inner sep=1pt,           
        execute at begin node=\scriptsize$\scriptstyle,
        execute at end node=$
    },
    doop/.style={
        circle, draw=black, 
                line width=0.4pt,        
        inner sep=1pt,           
        minimum size=9mm,
        execute at begin node=\scriptsize$\scriptstyle,
        execute at end node=$
    },
    target/.style={
        circle, draw=black,
                line width=0.4pt,        
        inner sep=1pt,           
        minimum size=9mm,
        execute at begin node=\scriptsize$\scriptstyle,
        execute at end node=$
    },
        edge/.style={-, thick, gray},
    elabel/.style={midway, fill=white, inner sep=1pt, font=\scriptsize},
        r1/.style={-, line width=0.5pt, draw=blue!70},
    r2/.style={-, line width=0.5pt, draw=red!70},
    r3/.style={-, line width=0.5pt, draw=green!60!black}
}

 \newtheorem{theorem}{Theorem}

 \newtheorem{definition}{Definition}
 \newtheorem{example}{Example}

\usepackage{slashed}
\newcommand{\ind}{{\perp\!\!\!\perp}}

\newcommand{\notind}{\slashed{\ind}}
\newcommand{\V}{\mathcal{V}}

\newcommand{\E}{\mathcal{E}}
\newcommand{\X}{\mathbf{X}}
\newcommand{\Y}{\mathbf{Y}}
\newcommand{\Z}{\mathbf{Z}}
\newcommand{\x}{\mathbf{x}}
\newcommand{\y}{\mathbf{y}}

\newcommand{\w}{\mathbf{w}}
\newcommand{\Zu}{\mathbf{Z}}
\newcommand{\zu}{\mathbf{z}}
\newcommand{\Zd}{\mathbf{\widetilde Z}}
\newcommand{\zd}{\mathbf{\widetilde z}}
\newcommand{\W}{\mathbf{W}}
\newcommand{\Pa}{\mathrm{Pa}}
\newcommand{\Anc}{\text{Anc}}
\newcommand{\Desc}{\text{Desc}}
\newcommand{\G}{{\mathcal{G}}}

\newcommand{\C}{\mathcal{C}}
\newcommand{\U}{\mathcal{U}}
\newcommand{\Do}{\mathrm{do}}

\newcommand{\Ga}{\G_{\text{a}}}
\newcommand{\Gb}{\G_{\text{b}}}
\newcommand{\Gc}{\G_{\text{c}}}
\newcommand{\Gd}{\G_{\text{d}}}
\newcommand{\Ge}{\G_{\text{e}}}
\newcommand{\Ru}{\mathcal{R}^\uparrow}
\newcommand{\Rd}{\mathcal{R}^\downarrow}

\usepackage[capitalize,noabbrev]{cleveref}

\theoremstyle{plain}

\usepackage[textsize=tiny]{todonotes}

\icmltitlerunning{Unveiling the Structure of Do-Calculus Reasoning via Derivation Graphs}

\begin{document}

\twocolumn[
  \icmltitle{Unveiling the Structure of Do-Calculus Reasoning via Derivation Graphs}



  \icmlsetsymbol{equal}{*}

  \begin{icmlauthorlist}
    \icmlauthor{Clément Yvernes}{UGA}
    \icmlauthor{Emilie Devijver}{UGA}
    \icmlauthor{Marianne Clausel}{Lor}
    \icmlauthor{Eric Gaussier}{UGA}
  \end{icmlauthorlist}

  \icmlaffiliation{UGA}{Univ Grenoble Alpes, CNRS, Grenoble INP, LIG, France}
  \icmlaffiliation{Lor}{Université de Lorraine, CNRS, CRAN, 54000 Nancy, France}

  \icmlcorrespondingauthor{Clément Yvernes}{clement.yvernes@univ-grenoble-alpes.fr}
  \icmlcorrespondingauthor{Emilie Devijver}{emilie.devijver@univ-grenoble-alpes.fr}  
  \icmlcorrespondingauthor{Marianne Clausel}{marianne.clausel@univ-lorraine.fr}  
  \icmlcorrespondingauthor{Eric Gaussier}{eric.gaussier@univ-grenoble-alpes.fr}
  \icmlkeywords{Machine Learning, ICML, Causal Inference, Do-calculus}

  \vskip 0.3in
]



\printAffiliationsAndNotice{}  

\begin{abstract}

The do-calculus defines a general system of inference for interventional queries, allowing causal quantities to be transformed through successive applications of its rules. This process induces a rich space of equivalent interventional expressions, but combining and ordering these rules remains challenging.
In this work, we introduce derivation graphs, which represent how do-calculus rules are applied and combined, and characterize the full space of observational and interventional probabilities which are equivalent under the do-calculus. The structure of these graphs yields a simple procedure that uses at most four applications of do-calculus rules. 
Finally, we show how applying identification algorithms to equivalent causal queries produces multiple valid estimands for the same causal quantity, eventually yielding more efficient estimators.

\end{abstract}

\section{Introduction}

Interventional queries play a central role in causal inference, as they enable
reasoning about the effects of external manipulations beyond mere associational
relationships.
Such queries are naturally expressed using Pearl’s $\Do$-operator
\citep{pearl_causality_2009}, which provides a formal language for representing
interventional distributions through modifications of the underlying
data-generating process. In this paper, we consider interventional and conditional probabilities of the form
\(
P(\mathbf{y} \mid \mathrm{do}(\mathbf{x}), \mathbf{w}),
\)
where $\mathbf x$ and $\mathbf w$ denote value assignments for the disjoint and possibly empty sets of variables $\mathbf X$ and $\mathbf W$, respectively. We refer to probabilities of this form as \emph{expressions} throughout the paper. 

Pearl's do-calculus introduces three graphical rewriting rules that govern how interventional and observational variables can be added, removed, or exchanged within such expressions. While these rules are grounded in conditional independences in appropriately mutilated causal graphs, they define a structured space of transformations over expressions. Despite the central role of the do-calculus in causal reasoning, the relationships
between different sequences of rule applications remain poorly characterized.
We propose a derivation graph that organizes do-calculus transformations and
makes explicit the connections between equivalent expressions. This structured view proves useful in several settings.

The first example concerns identifiability, that is,
whether a target interventional distribution can be
expressed only in terms of
the observational data-generating process.
This question has been extensively studied for adjustment-based estimands
\cite{JMLR:v18:16-319,pmlr-v124-perkovic20a}, including their statistical properties
\cite{JMLR:v21:19-1026,10.5555/3455716.3455962,10.1093/biomet/asac062,SmuclerRotnitzky+2022+174+189,christgau2025efficientadjustmentcomplexcovariates,10.1214/24-AOS2448},
and particularly in linear and Gaussian settings
\cite{10.5555/3540261.3541467,JMLR:v23:21-023,10.1111/rssb.12451}.
However, adjustment alone is not sufficient for identification in general.
The do-calculus is complete for identifying (conditional) interventional
distributions \citep{huang_pearls_2006}, and the ID algorithm
\citep{shpitser_identification_2006} applies the do-calculus rules to derive a valid
estimand whenever identification from observational data is possible; its output
can further be simplified using probability-based reasoning
\citep{Tikka1, Tikka2}.
While the ID algorithm returns a single identification formula, the do-calculus admits multiple equivalent representations of the same expression which may exhibit markedly different statistical behavior when estimated from finite data. We illustrate this issue in the following example.


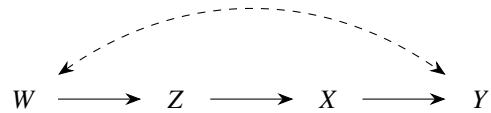
\begin{figure}[H]\centering
\begin{tikzpicture}[baseline=-0.5ex, node distance=1.1cm,
      node/.style = {inner sep=.4pt, minimum size=.9cm},
      >={Stealth[length=6pt]}, every edge/.style={->, thick}]
      
      \node[node] (W) {$W$};
      \node[node] (Z) [right=of W] {$Z$};
      \node[node] (X) [right=of Z] {$X$};
      \node[node] (Y) [right=of X] {$Y$};
      
      \draw[->] (W) -- (Z);
      \draw[->] (Z) -- (X);
      \draw[->] (X) -- (Y);
      \draw[dashed,<->,bend left=35] (W) to (Y);

    \end{tikzpicture}
\caption{A causal diagram with 4 variables used in Example~\ref{ex:1}.}
\label{fig:1}
\end{figure}
\begin{example}\label{ex:1}
Consider the ADMG shown in Figure~\ref{fig:1}, and the expression
$P(y \mid \Do(w,z))$. Using the front-door formula, one gets
\begin{align*}
    P(y\mid \Do(w,z)) = \sum_x P(x \mid w,z) \sum_{w^\prime,z^\prime}P(y \mid w^\prime,z^\prime,x) P(w^\prime, z^\prime).
\end{align*}
However, by using directly {usual }do-calculus rules and the backdoor formula, the same expression can be rewritten as: 
\begin{align*}
   P(y\mid \Do(w,z)) = P(y\mid \Do(z)) = \sum_{w} P(y\mid z,w) P(w),
\end{align*}
which is simpler to estimate and makes explicit the independence
of the query from the intervention on $W$.  
These two expressions are provably equal, yet arise from  different
derivations. {Moreover, the statistical properties of the associated estimators of the causal effect are different.} We provide more details on this example in Section \ref{sec:app}. 
\end{example}

A second important, yet less explored, question concerns the choice of
experimental data.
In many scientific domains, such as biology~\cite{lyle2023discobax} or medicine~\cite{bosdriesz2020evidence}, some interventions may
be feasible while others are costly or impossible to perform.
When {several } interventional queries correspond to the same causal effect,
understanding their equivalence structure enables principled choices between
experimental designs~\cite{ivanova2022efficient}.
Characterizing the full set of interventional queries that identify a given
causal effect allows one to reason about alternative experimental
strategies and to select interventions that are  statistically efficient 
and practically attainable.

Our main contributions are as follows:
\begin{itemize}
\item We introduce {the new concept of} derivation graphs, which explicitly represent the space of
expressions equivalent under the do-calculus.
\item We characterize the commutativity properties of do-calculus rules, leading to a structured and minimal operational calculus. {In particular, we prove that two expressions which can be derived from one another by the recursive application of any number of do-calculus rules\footnote{We say in this case that the two expressions are equivalent.} can in fact be derived by two applications of Rule $\mathcal{R}2$ and two applications of Rule $\mathcal{R}3$. We furthermore introduce an efficient graphical criterion allowing one to decide whether two expressions are equivalent or not.}
\item This framework directly reveals that the set of equivalent expressions can grow exponentially with the number of variables. It can furthermore be used to derive multiple valid estimands for the same causal effect using identification algorithms, {paving the way to insights about causal inference and experimental design}.
\end{itemize}

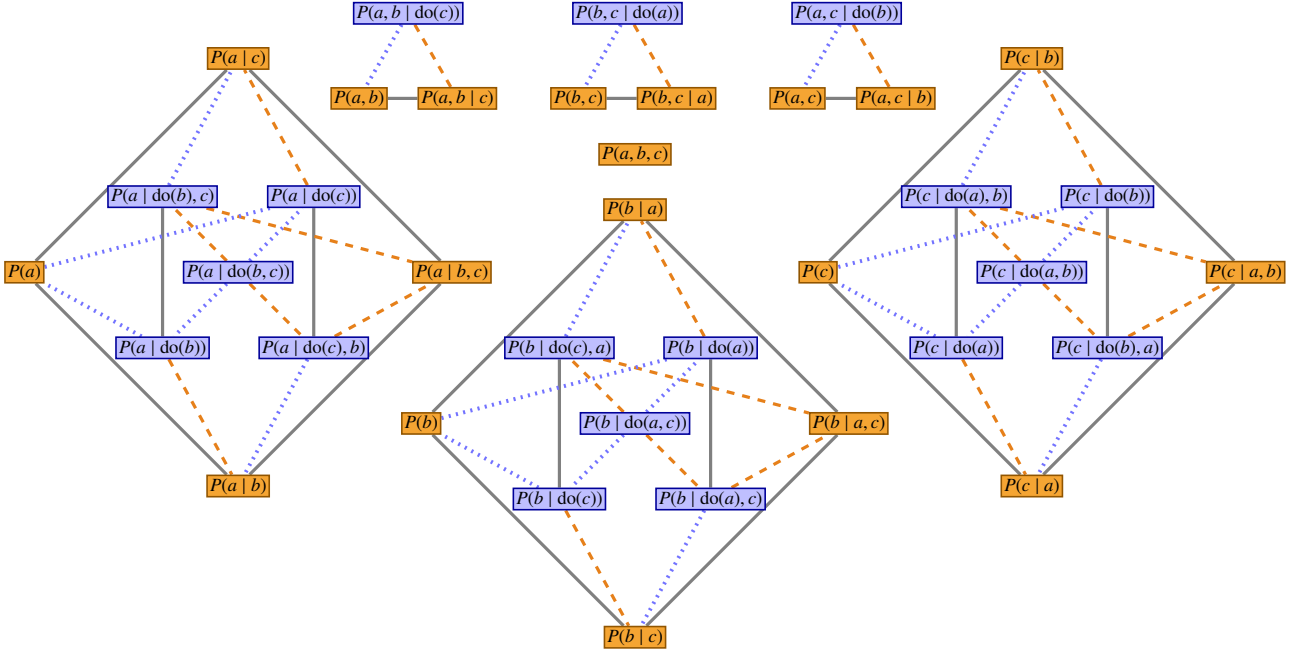
\begin{figure*}[htbp]
\centering
\begin{tikzpicture}[
    dofree/.style = {
      draw=coolorange!60!black,
      fill=coolorange!80,
      line width=0.6pt,
      inner sep=1pt,
      font=\scriptsize
    },
    withdo/.style = {
      draw=blue!60!black,
      fill=blue!25,
      line width=0.6pt,
      inner sep=1pt,
      font=\scriptsize
    },
    styleR1/.style = {thick, -, gray, line width=1.2pt},
    styleR2/.style = {thick, -, orange!60!brown, dashed, line width=1.2pt},
    styleR3/.style = {thick, -, blue!55, dotted, line width=1.5pt},
]

\begin{scope}[scale=2]
    \node[withdo] (c2n0) at (0, 0) {$P(a \mid \mathrm{do}(b,c))$};
    \node[withdo] (c2n1) at (0.5, -0.5) {$P(a \mid \mathrm{do}(c),b)$};
    \node[withdo] (c2n2) at (-0.5, 0.5) {$P(a \mid \mathrm{do}(b),c)$};
    \node[withdo] (c2n3) at (0.5, 0.5) {$P(a \mid \mathrm{do}(c))$};
    \node[withdo] (c2n4) at (-0.5, -0.5) {$P(a \mid \mathrm{do}(b))$};
    \node[dofree] (c2n5) at (0, -1.414) {$P(a \mid b)$};
    \node[dofree] (c2n6) at (-1.414, 0) {$P(a)$};

    \node[dofree] (c2n7) at (0, 1.414) {$P(a \mid c)$};
    \node[dofree] (c2n8) at (1.414, 0) {$P(a \mid b,c)$};
    
\begin{pgfonlayer}{background}
    \draw[styleR2] (c2n0) -- (c2n1);
    \draw[styleR2] (c2n0) -- (c2n2);
    \draw[styleR3] (c2n0) -- (c2n3);
    \draw[styleR3] (c2n0) -- (c2n4);
    \draw[styleR2] (c2n1) -- (c2n8);
    \draw[styleR3] (c2n1) -- (c2n5);
    \draw[styleR1] (c2n1) -- (c2n3);
    \draw[styleR1] (c2n2) -- (c2n4);
    \draw[styleR3] (c2n2) -- (c2n7);
    \draw[styleR2] (c2n2) -- (c2n8);
    \draw[styleR3] (c2n3) -- (c2n6);
    \draw[styleR2] (c2n3) -- (c2n7);
    \draw[styleR2] (c2n4) -- (c2n5);
    \draw[styleR3] (c2n4) -- (c2n6);
    \draw[styleR1] (c2n5) -- (c2n6);
    \draw[styleR1] (c2n5) -- (c2n8);
    \draw[styleR1] (c2n6) -- (c2n7);
    \draw[styleR1] (c2n7) -- (c2n8);
\end{pgfonlayer}
\end{scope}

\begin{scope}[scale=2, xshift=2.625cm, yshift=-1cm]
    \node[withdo] (c3n0) at (0, 0) {$P(b \mid \mathrm{do}(a,c))$};
    \node[withdo] (c3n1) at (0.5, -0.5) {$P(b \mid \mathrm{do}(a),c)$};
    \node[withdo] (c3n2) at (-0.5, 0.5) {$P(b \mid \mathrm{do}(c),a)$};
    \node[withdo] (c3n3) at (0.5, 0.5) {$P(b \mid \mathrm{do}(a))$};
    \node[withdo] (c3n4) at (-0.5, -0.5) {$P(b \mid \mathrm{do}(c))$};
    \node[dofree] (c3n5) at (0, -1.414) {$P(b \mid c)$};
    \node[dofree] (c3n6) at (-1.414, 0) {$P(b)$};

    \node[dofree] (c3n7) at (0, 1.414) {$P(b \mid a)$};
    \node[dofree] (c3n8) at (1.414, 0) {$P(b \mid a,c)$};
    \node[dofree] (solo) at (0, 1.78) {$P(a,b,c)$};

\begin{pgfonlayer}{background}
    \draw[styleR2] (c3n0) -- (c3n1);
    \draw[styleR2] (c3n0) -- (c3n2);
    \draw[styleR3] (c3n0) -- (c3n3);
    \draw[styleR3] (c3n0) -- (c3n4);
    \draw[styleR2] (c3n1) -- (c3n8);
    \draw[styleR3] (c3n1) -- (c3n5);
    \draw[styleR1] (c3n1) -- (c3n3);
    \draw[styleR1] (c3n2) -- (c3n4);
    \draw[styleR3] (c3n2) -- (c3n7);
    \draw[styleR2] (c3n2) -- (c3n8);
    \draw[styleR3] (c3n3) -- (c3n6);
    \draw[styleR2] (c3n3) -- (c3n7);
    \draw[styleR2] (c3n4) -- (c3n5);
    \draw[styleR3] (c3n4) -- (c3n6);
    \draw[styleR1] (c3n5) -- (c3n6);
    \draw[styleR1] (c3n5) -- (c3n8);
    \draw[styleR1] (c3n6) -- (c3n7);
    \draw[styleR1] (c3n7) -- (c3n8);
\end{pgfonlayer}
\end{scope}

\begin{scope}[scale=2, xshift=5.25cm]
    \node[withdo] (c4n0) at (0, 0) {$P(c \mid \mathrm{do}(a,b))$};
    \node[withdo] (c4n1) at (0.5, -0.5) {$P(c \mid \mathrm{do}(b),a)$};
    \node[withdo] (c4n2) at (-0.5, 0.5) {$P(c \mid \mathrm{do}(a),b)$};
    \node[withdo] (c4n3) at (0.5, 0.5) {$P(c \mid \mathrm{do}(b))$};
    \node[withdo] (c4n4) at (-0.5, -0.5) {$P(c \mid \mathrm{do}(a))$};
    \node[dofree] (c4n5) at (0, -1.414) {$P(c \mid a)$};
    \node[dofree] (c4n6) at (-1.414, 0) {$P(c)$};

    \node[dofree] (c4n7) at (0, 1.414) {$P(c \mid b)$};
    \node[dofree] (c4n8) at (1.414, 0) {$P(c \mid a,b)$};

\begin{pgfonlayer}{background}
    \draw[styleR2] (c4n0) -- (c4n1);
    \draw[styleR2] (c4n0) -- (c4n2);
    \draw[styleR3] (c4n0) -- (c4n3);
    \draw[styleR3] (c4n0) -- (c4n4);
    \draw[styleR2] (c4n1) -- (c4n8);
    \draw[styleR3] (c4n1) -- (c4n5);
    \draw[styleR1] (c4n1) -- (c4n3);
    \draw[styleR1] (c4n2) -- (c4n4);
    \draw[styleR3] (c4n2) -- (c4n7);
    \draw[styleR2] (c4n2) -- (c4n8);
    \draw[styleR3] (c4n3) -- (c4n6);
    \draw[styleR2] (c4n3) -- (c4n7);
    \draw[styleR2] (c4n4) -- (c4n5);
    \draw[styleR3] (c4n4) -- (c4n6);
    \draw[styleR1] (c4n5) -- (c4n6);
    \draw[styleR1] (c4n5) -- (c4n8);
    \draw[styleR1] (c4n6) -- (c4n7);
    \draw[styleR1] (c4n7) -- (c4n8);
\end{pgfonlayer}
\end{scope}

\begin{scope}[xshift=2.25cm, yshift=2.3cm, scale=1.3]
    \node[withdo] (doC) at (0,.866) {$P(a,b \mid \mathrm{do}(c))$};
    \node[dofree] (givenC) at (.5,0) {$P(a,b \mid c)$};
    \node[dofree] (marginal) at (-.5,0) {$P(a,b)$};

\begin{pgfonlayer}{background}
    \draw[styleR2] (doC) -- (givenC);
    \draw[styleR3] (doC) -- (marginal);
    \draw[styleR1] (givenC) -- (marginal);
\end{pgfonlayer}
\end{scope}

\begin{scope}[xshift=5.15cm, yshift=2.3cm, scale=1.3]
    \node[withdo] (doA) at (0,.866) {$P(b,c \mid \mathrm{do}(a))$};
    \node[dofree] (givenA) at (.5,0) {$P(b,c \mid a)$};
    \node[dofree] (marginal) at (-.5,0) {$P(b,c)$};

\begin{pgfonlayer}{background}
    \draw[styleR2] (doA) -- (givenA);
    \draw[styleR3] (doA) -- (marginal);
    \draw[styleR1] (givenA) -- (marginal);
\end{pgfonlayer}
\end{scope}

\begin{scope}[xshift=8.05cm, yshift=2.3cm, scale=1.3]
    \node[withdo] (doC) at (0,.866) {$P(a,c \mid \mathrm{do}(b))$};
    \node[dofree] (givenC) at (.5,0) {$P(a,c \mid b)$};
    \node[dofree] (marginal) at (-.5,0) {$P(a,c)$};

\begin{pgfonlayer}{background}
    \draw[styleR2] (doC) -- (givenC);
    \draw[styleR3] (doC) -- (marginal);
    \draw[styleR1] (givenC) -- (marginal);
\end{pgfonlayer}
\end{scope}

\end{tikzpicture}
\caption{Derivation graph for the  ADMG $\mathcal{G}_0$ over variables $A$, $B$, and $C$ with no edges. Blue nodes denote expressions involving interventions, while orange nodes denote observational (do-free) expressions. 
Edges represent atomic applications of the do-calculus rules (one variable at a time): gray solid lines for Rule $\mathcal{R}1$, orange dashed lines for Rule $\mathcal{R}2$, and blue dotted lines for Rule $\mathcal{R}3$.}
\label{fig:derivation_lattice_AfBfC}
\end{figure*}

The remainder of the paper is organized as follows.
Section~\ref{sec:notions} introduces the necessary preliminaries, while Section~\ref{sec:derivationGraph} presents the derivation graph and investigates its main properties. 
Section~\ref{sec:commutativity} analyzes the commutativity of the rules and
derives an operational calculus equivalent to the do-calculus, { which provides surprisingly simple interpretations on the derivation graph}. Section~\ref{sec:app}
discusses the application to the derivation of multiple estimands from the ID algorithm.
Section~\ref{sec:conc} concludes the paper. All the proofs are provided in Appendix, and Jupyter notebooks for building the derivation graph and reproducing the experiments are provided in \href{https://gricad-gitlab.univ-grenoble-alpes.fr/yvernesc/do-calculus-derivation-graphs}{this repository}.\footnote{https://gricad-gitlab.univ-grenoble-alpes.fr/yvernesc/do-calculus-derivation-graphs}

\section{Preliminaries}
\label{sec:notions}

%
A  variable is denoted by an uppercase letter $X$ and its 
value by a small letter $x$. 
A bold uppercase letter $\X$ denotes a set {of variables, the realizations of which are denoted by $\mathbf{x}$}.

\paragraph{Graphs.} We denote by $\Anc(\X, \G)$ and $\Desc(\X, \G)$ the sets of ancestors and descendants of $\X$ in the graph ${\G}$, respectively. By convention, each node is regarded as its own ancestor and descendant. 
%
A vertex $V$ is said to be \emph{active} on a path relative to a subset of variables $\Z$ if 1) $V$ is a collider and $V$ or any of its descendants are in $\Z$ or 2) $V$ is a non-collider and is not in $\Z$. A path $\pi$  is said to be \emph{active} given (or conditioned on) $\Z$ if every vertex on $\pi$ is active relative to $\Z$. Otherwise, $\pi$   is said to be \emph{inactive} given $\Z$. 
%
Given a graph $\mathcal{G}$, the sets $\X$ and $\Y$ are said to be d-separated by $\Z$ if every path between $\X$ and $\Y$ is inactive given $\Z$. We denote this by $\X \ind_{\G} \Y \mid \Z$. Otherwise, $\X$ and $\Y$ are d-connected given $\Z$, which we denote by $\X \notind_{\G} \Y \mid \Z$.
%
The mutilated graph 
 ${\mathcal{G}}_{\overline{\X}\underline{\Z}}$ is the result of
removing from a graph ${\mathcal{G}}$ edges 
with an arrowhead into $\X$ (e.g., $A \rightarrow \X$, $A \dashleftrightarrow \X$),  
and edges with a tail from $\Z$ (e.g., $A \leftarrow \Z$). 
For two graphs $\G_1 = (\V_1, \E_1)$ and $\G_2 = (\V_2, \E_2)$, the union is $\G_1 \cup \G_2 \coloneqq (\V_1 \cup \V_2, \E_1 \cup \E_2)$. 

\paragraph{Structural Causal Models.}  
A \emph{Structural Causal Model} (SCM) $\mathcal{M}$ is a 4-tuple $\langle \U, \V, \mathcal{F}, P(\U)\rangle$, where $\U$ is a set of mutually independent exogenous (latent) variables and $\V$ is a set of endogenous (observed) variables. $\mathcal{F} = \{f_i\}_{i=1}^{|\V|}$ assigns each $V_i \in \V$ a function $f_i(\U_i \cup \Pa(V_i))$, with $\U_i \subseteq \U$ and $\Pa(V_i) \subseteq \V \setminus V_i$, and uncertainty is encoded by $P(\U)$.  
Each SCM induces an acyclic directed mixed graph (ADMG) $\G(\V,\E)$, or \emph{causal diagram}, representing structural relations among $\V \cup \U$. Every $V_i \in \V$ is a vertex, directed edges $(V_j \to V_i)$ link each variable to its parents, and bidirected edges $(V_j \dashleftrightarrow V_i)$ connect variables sharing a common exogenous parent ($\U_i \cap \U_j \neq \emptyset$).

\paragraph{Interventions, Do-Operator and Do-Calculus}  
Performing an intervention $\X\!\!=\!\! \mathbf{x}$ is represented through the do-operator, \textit{do}($\X\!=\!\mathbf{x}$), which represents the operation of fixing a set $\X$ to a constant $\mathbf{x}$, and induces a submodel $\mathcal{M}_\X$, which is $\mathcal{M}$ with $f_X$ replaced to $x$ for every $X \in \X$. The post-interventional distribution induced by $\mathcal{M}_\X$ is denoted by $P(\mathbf{v} \setminus \mathbf{x} \mid \Do(\mathbf{x}))$. 
{A probability distribution $P$ is said to be \emph{compatible} with an ADMG $\mathcal{G}$ if it can be generated by a Structural Causal Model inducing $\mathcal{G}$. Then $P$ is Markovian with respect to $\G$ (Proposition 6.31 in \citet{PetJanSch17}).}
The three do-calculus rules are provided in the next theorem.

\begin{theorem}[Do-Calculus Rules, \cite{pearl_causality_2009}]
Let $\mathcal{G} = (\mathcal{V}, \mathcal{E})$ be an ADMG, and let $\mathbf{X,Y, Z, W}$ be pairwise disjoint subsets of $\mathcal{V}$, with $\Y \neq \emptyset$. Then the following rules hold for all distributions compatible with $\mathcal{G}$:

\textbf{Rule} $\mathcal{R}_1$ - $G \models \mathcal{R}_1(\Z)$: \text{insertion/deletion of observations}\newline
$\Y \ind_{\G_{\overline{\X}}} \Z \mid \X,\W 
    \Rightarrow P(\mathbf{y} \mid \Do(\mathbf{x}), \mathbf{z}, \mathbf{w})
= P(\mathbf{y} \mid \Do(\mathbf{x}), \mathbf{w})$,

\vspace{0.2cm}
\textbf{Rule} $\mathcal{R}_2$ - $\G \models \mathcal{R}_2(\Z)$: \text{action/observation exchange}\newline
$\mathbf{Y} \ind_{\mathcal{G}_{\overline{\mathbf{X}}\, \underline{\mathbf{Z}}}} \mathbf{Z} \mid \mathbf{X}, \mathbf{W} \Rightarrow P(\mathbf{y} \mid \Do(\mathbf{x},\mathbf{z}), \mathbf{w})
= P(\mathbf{y} \mid \Do(\mathbf{x}), \mathbf{z}, \mathbf{w})$,

\vspace{0.2cm}
\textbf{Rule} $\mathcal{R}_3$ - $\G \models \mathcal{R}_3(\Z)$: \text{insertion/deletion of actions }\newline
$\mathbf{Y} \ind_{\mathcal{G}_{\overline{\mathbf{X}}\, \overline{\mathbf{Z(W)}}}} \mathbf{Z} \mid \mathbf{X}, \mathbf{W} \Rightarrow P(\mathbf{y} \mid \Do(\mathbf{x},\mathbf{z}), \mathbf{w})
= P(\mathbf{y} \mid \Do(\mathbf{x}), \mathbf{w})$,

\vspace{0.2cm}
where $\mathbf{Z}(\mathbf{W})$ is the set of $\mathbf{Z}$-nodes non-ancestors of $\mathbf{W}$ in $\G_{\overline{\mathbf X}}$, and where $\mathcal{G}\models \mathcal{R}$ means that the rule $\mathcal{R}$ is valid in $\mathcal{G}$, i.e. the corresponding graphical condition holds in $\G$. 
\end{theorem}

The do-calculus operates by successively applying its rules to rewrite observational and interventional probabilities of the form \(P(\mathbf{y}\mid\mathrm{do}(\mathbf{x}),\mathbf{w})\). Starting from a causal expression of interest, each step replaces the
current quantity by an equivalent one, whose equality is certified by a
graphical condition in the graph. Given two rules $\mathcal{R}_i(\Z)$ and $\mathcal{R}_j(\widetilde\Z)$, we denote by $\mathcal{R}_i(\Z)\,\mathcal{R}_j(\widetilde\Z)$ the successive application of $\mathcal{R}_i(\Z)$ followed by $\mathcal{R}_j(\widetilde\Z)$.  A sequence of do-calculus rule application is valid if each rule in the sequence is valid when applied in order.

\section{Derivation Graphs and Equivalent Expressions} 
\label{sec:derivationGraph}

The main contribution of this section is to introduce the so--called derivation graph. It allows an extensive exploration of all expressions that are equivalent under the do-calculus. 

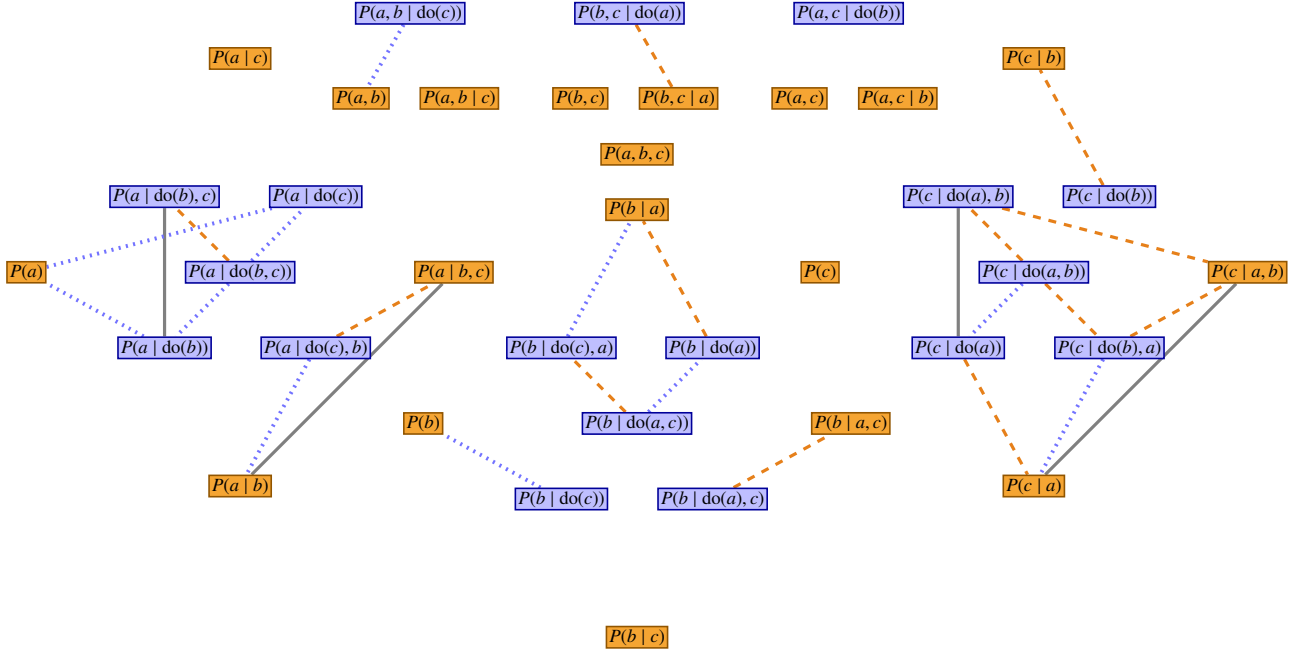
\begin{figure*}[htbp]
\centering
\begin{tikzpicture}[
    dofree/.style = {
      draw=coolorange!60!black,
      fill=coolorange!80,
      line width=0.6pt,
      inner sep=1pt,
      font=\scriptsize
    },
    withdo/.style = {
      draw=blue!60!black,
      fill=blue!25,
      line width=0.6pt,
      inner sep=1pt,
      font=\scriptsize
    },
    styleR1/.style = {thick, -, gray, line width=1.2pt},
    styleR2/.style = {thick, -, orange!60!brown, dashed, line width=1.2pt},
    styleR3/.style = {thick, -, blue!55, dotted, line width=1.5pt},
]

\begin{scope}[scale=2]
    \node[withdo] (c2n0) at (0, 0) {$P(a \mid \mathrm{do}(b,c))$};
    \node[withdo] (c2n1) at (0.5, -0.5) {$P(a \mid \mathrm{do}(c),b)$};
    \node[withdo] (c2n2) at (-0.5, 0.5) {$P(a \mid \mathrm{do}(b),c)$};
    \node[withdo] (c2n3) at (0.5, 0.5) {$P(a \mid \mathrm{do}(c))$};
    \node[withdo] (c2n4) at (-0.5, -0.5) {$P(a \mid \mathrm{do}(b))$};
    \node[dofree] (c2n5) at (0, -1.414) {$P(a \mid b)$};
    \node[dofree] (c2n6) at (-1.414, 0) {$P(a)$};

    \node[dofree] (c2n7) at (0, 1.414) {$P(a \mid c)$};
    \node[dofree] (c2n8) at (1.414, 0) {$P(a \mid b,c)$};

\begin{pgfonlayer}{background}
    \draw[styleR2] (c2n0) -- (c2n2);
    \draw[styleR3] (c2n0) -- (c2n3);
    \draw[styleR3] (c2n0) -- (c2n4);
    \draw[styleR2] (c2n1) -- (c2n8);
    \draw[styleR3] (c2n1) -- (c2n5);
    \draw[styleR1] (c2n2) -- (c2n4);
    \draw[styleR3] (c2n3) -- (c2n6);
    \draw[styleR3] (c2n4) -- (c2n6);
    \draw[styleR1] (c2n5) -- (c2n8);
\end{pgfonlayer}
\end{scope}

\begin{scope}[scale=2, xshift=2.625cm, yshift=-1cm]
    \node[withdo] (c3n0) at (0, 0) {$P(b \mid \mathrm{do}(a,c))$};
    \node[withdo] (c3n1) at (0.5, -0.5) {$P(b \mid \mathrm{do}(a),c)$};
    \node[withdo] (c3n2) at (-0.5, 0.5) {$P(b \mid \mathrm{do}(c),a)$};
    \node[withdo] (c3n3) at (0.5, 0.5) {$P(b \mid \mathrm{do}(a))$};
    \node[withdo] (c3n4) at (-0.5, -0.5) {$P(b \mid \mathrm{do}(c))$};
    \node[dofree] (c3n5) at (0, -1.414) {$P(b \mid c)$};
    \node[dofree] (c3n6) at (-1.414, 0) {$P(b)$};

    \node[dofree] (c3n7) at (0, 1.414) {$P(b \mid a)$};
    \node[dofree] (c3n8) at (1.414, 0) {$P(b \mid a,c)$};
    \node[dofree] (solo) at (0, 1.78) {$P(a,b,c)$};

\begin{pgfonlayer}{background}
    \draw[styleR2] (c3n0) -- (c3n2);
    \draw[styleR3] (c3n0) -- (c3n3);
    \draw[styleR2] (c3n1) -- (c3n8);
    \draw[styleR3] (c3n2) -- (c3n7);
    \draw[styleR2] (c3n3) -- (c3n7);
    \draw[styleR3] (c3n4) -- (c3n6);
\end{pgfonlayer}
\end{scope}

\begin{scope}[scale=2, xshift=5.25cm]
    \node[withdo] (c4n0) at (0, 0) {$P(c \mid \mathrm{do}(a,b))$};
    \node[withdo] (c4n1) at (0.5, -0.5) {$P(c \mid \mathrm{do}(b),a)$};
    \node[withdo] (c4n2) at (-0.5, 0.5) {$P(c \mid \mathrm{do}(a),b)$};
    \node[withdo] (c4n3) at (0.5, 0.5) {$P(c \mid \mathrm{do}(b))$};
    \node[withdo] (c4n4) at (-0.5, -0.5) {$P(c \mid \mathrm{do}(a))$};
    \node[dofree] (c4n5) at (0, -1.414) {$P(c \mid a)$};
    \node[dofree] (c4n6) at (-1.414, 0) {$P(c)$};

    \node[dofree] (c4n7) at (0, 1.414) {$P(c \mid b)$};
    \node[dofree] (c4n8) at (1.414, 0) {$P(c \mid a,b)$};

\begin{pgfonlayer}{background}
    \draw[styleR2] (c4n0) -- (c4n1);
    \draw[styleR2] (c4n0) -- (c4n2);
    \draw[styleR3] (c4n0) -- (c4n4);
    \draw[styleR2] (c4n1) -- (c4n8);
    \draw[styleR3] (c4n1) -- (c4n5);
    \draw[styleR1] (c4n2) -- (c4n4);
    \draw[styleR2] (c4n2) -- (c4n8);
    \draw[styleR2] (c4n3) -- (c4n7);
    \draw[styleR2] (c4n4) -- (c4n5);
    \draw[styleR1] (c4n5) -- (c4n8);
\end{pgfonlayer}
\end{scope}

\begin{scope}[xshift=2.25cm, yshift=2.3cm, scale=1.3]
    \node[withdo] (doC) at (0,.866) {$P(a,b \mid \mathrm{do}(c))$};
    \node[dofree] (givenC) at (.5,0) {$P(a,b \mid c)$};
    \node[dofree] (marginal) at (-.5,0) {$P(a,b)$};

\begin{pgfonlayer}{background}
    \draw[styleR3] (doC) -- (marginal);
\end{pgfonlayer}
\end{scope}

\begin{scope}[xshift=5.15cm, yshift=2.3cm, scale=1.3]
    \node[withdo] (doA) at (0,.866) {$P(b,c \mid \mathrm{do}(a))$};
    \node[dofree] (givenA) at (.5,0) {$P(b,c \mid a)$};
    \node[dofree] (marginal) at (-.5,0) {$P(b,c)$};

\begin{pgfonlayer}{background}
    \draw[styleR2] (doA) -- (givenA);
\end{pgfonlayer}
\end{scope}

\begin{scope}[xshift=8.05cm, yshift=2.3cm, scale=1.3]
    \node[withdo] (doC) at (0,.866) {$P(a,c \mid \mathrm{do}(b))$};
    \node[dofree] (givenC) at (.5,0) {$P(a,c \mid b)$};
    \node[dofree] (marginal) at (-.5,0) {$P(a,c)$};

\begin{pgfonlayer}{background}
\end{pgfonlayer}
\end{scope}

\end{tikzpicture}
\caption{Derivation graph for the ADMG $\mathcal{G}: A \rightarrow B \rightarrow C$. Blue nodes denote expressions involving interventions, while orange nodes denote observational (do-free) expressions. 
Edges represent atomic applications of the do-calculus rules (one variable at a time): gray solid lines for Rule $\mathcal{R}1$, orange dashed lines for Rule $\mathcal{R}2$, and blue dotted lines for Rule $\mathcal{R}3$.}
\label{fig:derivation_lattice_do_bc}
\end{figure*}

\begin{definition}
Let $\mathcal{G} = (\mathcal{V}_\mathcal{G}, \mathcal{E}_\mathcal{G})$ be an ADMG. 
Its \emph{derivation graph} is defined as the following undirected graph $\mathcal{D}[\mathcal{G}] = (\mathcal{V}_\mathcal{D}, \mathcal{E}_\mathcal{D})$, where:
\begin{itemize}
        \item The set of vertices $\mathcal{V}_\mathcal{D}$ consists of all expressions of the form $P(\mathbf{y} \mid \Do(\mathbf{x}), \mathbf{w})$, with pairwise disjoint subsets  $(\mathbf{Y}, \mathbf{X},  \mathbf{W}) \subseteq \mathcal{V}_\mathcal{G}$, with $\mathbf{Y} \neq \emptyset$.
    \item There is an edge $(Q_1, Q_2) \in \mathcal{E}_\mathcal{D}$ if $Q_2$ can be obtained from $Q_1$ by applying a single do-calculus rule  that is valid in $\mathcal{G}$.
\end{itemize}
\end{definition}

Figure~\ref{fig:derivation_lattice_AfBfC} illustrates the derivation graph for the  causal graph $\mathcal{G}_0 = (\{A,B,C\}, \emptyset)$, while Figure~\ref{fig:derivation_lattice_do_bc} is the derivation graph for the ADMG $\mathcal{G} = (\{A,B,C\}, \{A\rightarrow B, B \rightarrow C\})$. 

The derivation graph exhibit several structural properties.


\begin{restatable}[]{proposition}{propemptyADMG}
\label{prop:emptyADMG}
    Let $\mathcal{G}_0 = (\V_{\G_0}, \emptyset)$ be the causal graph with no edges.
    Then its derivation graph $\mathcal{D}[\mathcal{G}_0]$ consists of $2^{\left|\V_{\G_0}\right|}-1$ connected components, each corresponding to a non-empty subset of outcome variables $\Y$, and expressions of the form
        $
        \left\{ P(\mathbf{y} \mid \Do(\mathbf{x}), \mathbf{w}) 
        : \mathbf{X}, \mathbf{W} \subseteq \V_{\G_0} \setminus \mathbf{Y}, \X \cap \W = \emptyset \right\}$.
        
    
     Moreover, for any causal graph $\mathcal{G}$ defined on the same vertex set $\V_{\G_0}$, the derivation graph $\mathcal{D}[\mathcal{G}]$ is a subgraph of $\mathcal{D}[\mathcal{G}_0]$.
\end{restatable}

The connected components of $\mathcal{D}[\mathcal{G}_0]$ are visible in Figure~\ref{fig:derivation_lattice_AfBfC}.
The relationship between $\mathcal{D}[\mathcal{G}]$ and $\mathcal{D}[\mathcal{G}_0]$ is seen by comparing Figure~\ref{fig:derivation_lattice_AfBfC} with Figure~\ref{fig:derivation_lattice_do_bc}: while the set of nodes remains unchanged, adding causal edges removes derivations, resulting in more connected components in $\mathcal{D}[\mathcal{G}]$.


Given a causal graph $\mathcal{G}$ and its associated derivation graph $\mathcal{D}[\mathcal{G}]$, 
the connected component of $\mathcal{D}[\mathcal{G}]$ that contains a specific expression $\mathcal{Q}$ {is of special interest: it represents all expressions which are equivalent to $\mathcal{Q}$ under the do-calculus.}

\begin{definition}[Equivalent Expressions]
    Let $\mathcal{G} = (\mathcal{V}_\mathcal{G}, \mathcal{E}_\mathcal{G})$ be an ADMG, 
    and let $\mathcal{D}[\mathcal{G}] = (\mathcal{V}_\mathcal{D}, \mathcal{E}_\mathcal{D})$ be its derivation graph.
    For a given expression $\mathcal{Q} \in \mathcal{V}_\mathcal{D}$, 
    the set of \emph{equivalent expressions} 
    $\mathcal{A}^\mathcal{G}_\mathcal{Q}$ is defined as the set of nodes in the connected component of $\mathcal{D}$ that contains $\mathcal{Q}$.
\end{definition}

Note that being identifiable is a stable property within the connected component: the set of equivalent interventional queries in a connected component are either all identifiable, or none of them are.

\begin{example}
    In $\mathcal{G}_0 = (\{A,B,C\}, \emptyset)$, one can see in Figure~\ref{fig:derivation_lattice_AfBfC} that:
\[
\mathcal{A}^{\mathcal{G}_0}_{P(a,b \mid \Do(c))} =
\left\{P(a,b \mid \Do(c)), P(a,b), P(a,b\mid c)\right\}.
\]

In the ADMG $\mathcal{G} = (\{A,B,C\}, \{A\rightarrow B, B \rightarrow C\})$ (see Figure~\ref{fig:derivation_lattice_do_bc}), one has:
\[
\mathcal{A}^{\mathcal{G}}_{P(b \mid a)} =
\left\{
\begin{aligned}
& P(b \mid a), P(b \mid \Do(c), a),\\
& P(b \mid \Do(a,c)), P(b \mid \Do(a))
\end{aligned}
\right\}.
\]
\end{example}

The next proposition bounds the complexity of the exploration of the derivation graph.

\begin{restatable}[]{proposition}{proptailleconnectedcomponent}
    Let $\mathcal{G} = (\mathcal{V}_\mathcal{G}, \mathcal{E}_\mathcal{G})$ be an ADMG, and let $\Y,\X,\W$ be pairwise disjoint subsets of~$\V_\G$, with $\Y \neq \emptyset$. Then, 
    \begin{align*}
     1\leq  \left|  \mathcal{A}^{\mathcal{G}}_{{P}(\mathbf{y}\mid \Do(\mathbf{x}),\mathbf{w})} \right|  \leq 3^{\left| \mathcal{V}_{\mathcal{G}} \setminus \Y \right| }.
    \end{align*}
\end{restatable}

Although this bound does not depend on $\X$ or $\W$, it is tight in the worst case, as equality is attained for the empty causal graph. Hence, the connected component $\mathcal{A}^{\mathcal{G}}_{{P}(\mathbf{y}\mid \Do(\mathbf{x}),\mathbf{w})}$ can grow exponentially with the number of variables outside $\Y$, and may therefore contain exponentially many derivation paths. This highlights the need for structured representations and efficient navigation strategies within the derivation graph. 
Similarly, in the particular case of identification, the number of distinct identification formulae may itself grow exponentially, even before considering additional probabilistic rewritings outside the derivation graph.

\section{Commutativity of Do‑Calculus Rules and Simplification of Causal Reasoning}
\label{sec:commutativity}

In this section, we study how the do-calculus rules interact using derivation graphs, leading to a simple procedure that requires at most four applications of do-calculus rules to decide whether two expressions are equivalent. We first show that Rule~$\mathcal{R}1$ is redundant, as it can be systematically derived from Rules~$\mathcal{R}2$ and~$\mathcal{R}3$. We then focus on these two rules and analyze how their applications may commute.
To distinguish the type of rule application, we use the symbols $\uparrow$ and $\downarrow$: $\Ru_1(\Z)$ (resp. $\Ru_2(\Z)$ and $\Ru_3(\Z)$) denotes the insertion of an observation (resp. action) $\Z$, while $\Rd_1(\Z)$ (resp. $\Rd_2(\Z)$ and $\Rd_3(\Z)$) denotes the deletion of an observation (resp. action) $\Z$.

Throughout this section, we use the Napkin graph introduced in Figure~\ref{fig:derivation_napkin} \citep{pearl2018bookofwhy, guo2025causalinferencenapkingraph} as a running example. Its associated derivation graph is also shown in Figure~\ref{fig:derivation_napkin}. An additional example supporting our analysis is provided in Appendix~\ref{app:examples}.

\begin{figure}[t]
    \centering
\begin{tikzpicture}[baseline=-0.5ex, node distance=10mm,
      node/.style = {inner sep=.4pt, minimum size=9mm,},
      >={Stealth[length=6pt]}, every edge/.style={->, semithick}]
      
      \node[node] (W) {$W$};
      \node[node] (Z) [right=of W] {$Z$};
      \node[node] (X) [right=of Z] {$X$};
      \node[node] (Y) [right=of X] {$Y$};
      
      \draw[->] (W) -- (Z);
      \draw[->] (Z) -- (X);
      \draw[->] (X) -- (Y);
      
      \draw[dashed,<->,bend left=17] (W) to (X);
      \draw[dashed,<->,bend left=25] (W) to (Y);

    \end{tikzpicture}
    
\centering
\begin{tikzpicture}[
    dofree/.style = {
      draw=coolorange!60!black,
      fill=coolorange!80,
      line width=0.6pt,
      inner sep=1pt,
      font=\scriptsize
    },
    withdo/.style = {
      draw=blue!60!black,
      fill=blue!25,
      line width=0.6pt,
      inner sep=1pt,
      font=\scriptsize
    },
    styleR1/.style = {thick, -, gray, line width=1.2pt},
    styleR2/.style = {thick, -, orange!60!brown, dashed, line width=1.2pt},
    styleR3/.style = {thick, -, blue!55, dotted, line width=1.5pt},
    scale=3
]
    \node[withdo] (n0) at (1.55, -.42) {$P(y \mid \Do(x))$};
    \node[withdo] (n1) at (1.85, -.12) {$P(y \mid \Do(x,z))$};
    \node[withdo] (n2) at (.6, -.3) {$P(y \mid \Do(w,x))$};
    \node[withdo] (n3) at (0, 0) {$P(y \mid \Do(w,x),z)$};
    \node[withdo] (n4) at (0.6, 0.7) {$P(y \mid \Do(w),x)$};
    \node[withdo] (n5) at (1, 0) {$P(y \mid \Do(w,x,z))$};
    \node[withdo] (n6) at (1, 1) {$P(y \mid \Do(w,z),x)$};
    \node[withdo] (n7) at (0, 1) {$P(y \mid \Do(w),x,z)$};
    \node[withdo] (n8) at (1.85, .88) {$P(y \mid \Do(z),x)$};

\begin{pgfonlayer}{background}
    \draw[styleR3] (n0) -- (n1);
    \draw[styleR3] (n0) -- (n2);
    \draw[styleR3] (n1) -- (n5);
    \draw[styleR2] (n1) -- (n8);
    \draw[styleR1] (n2) -- (n3);
    \draw[styleR2] (n2) -- (n4);
    \draw[styleR3] (n2) -- (n5);
    \draw[styleR2, line width = .9pt] (n3) -- (n5);
    \draw[styleR2] (n3) -- (n7);
    \draw[styleR3] (n4) -- (n6);
    \draw[styleR1] (n4) -- (n7);
    \draw[styleR2] (n5) -- (n6);
    \draw[styleR2] (n6) -- (n7);
    \draw[styleR3] (n6) -- (n8);
\end{pgfonlayer}
\end{tikzpicture}
\caption{\textbf{The Napkin Graph} (top) and its derivation graph restricted to the connected component of $P(y \mid \Do(x))$ (bottom). In the derivation graph, blue nodes denote expressions involving interventions. 
Edges represent atomic applications of the do-calculus rules (one variable at a time): gray solid lines for Rule $\mathcal{R}1$, orange dashed lines for Rule $\mathcal{R}2$, and blue dotted lines for Rule $\mathcal{R}3$.}
\label{fig:derivation_napkin}
\end{figure}
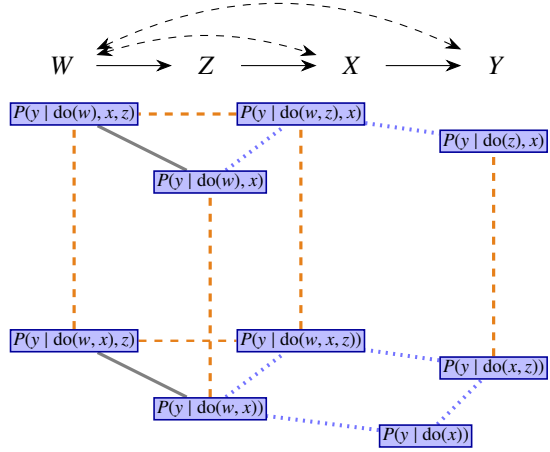

\subsection{Redundancy of $\mathcal{R}_1$}

Following Lemma~4 in \citep{huang_pearls_2006}, Rule~$\mathcal{R}_1$ is not primitive: its effect can be reproduced by a sequence of applications of Rules~$\mathcal{R}_2$ and~$\mathcal{R}_3$.  This gives one direction. Theorem~\ref{thm:r1_as_r23} establishes the converse.

\begin{restatable}[]{theorem}{thRun}
    \label{thm:r1_as_r23}
    Let $\G=(\V_\G,\E_\G)$ be an ADMG, and let $\Y,\X,\W,\Z$ be pairwise disjoint subsets of~$\V_\G$, with $\Y \neq \emptyset$. Consider an expression of the form $P(\mathbf{y} \mid \Do(\mathbf{x}), \mathbf{w})$. Then the following statements hold:
    \begin{itemize}
        \item $\mathcal{G} \models \Ru_1(\Z)
        \;\Leftrightarrow\;
        \mathcal{G} \models \Ru_3(\Z)\,\Rd_2(\Z)$,
        \item $\mathcal{G} \models \Rd_1(\Z)
        \;\Leftrightarrow\;
        \mathcal{G} \models \Ru_2(\Z)\,\Rd_3(\Z)$.
    \end{itemize}
\end{restatable}
The two statements are dual, corresponding to applying the rule in either direction. In the remainder of the paper, we state only one of them, the other following by symmetry.

This result admits a geometric interpretation in the derivation graph. Each application of Rule~$\mathcal{R}1$ (gray solid edge in Fig.~\ref{fig:derivation_napkin}) can be viewed as a shortcut for a two-step derivation through an intermediate node, consisting of Rule~$\mathcal{R}2$ (orange dashed edge) followed by Rule~$\mathcal{R}3$ (blue dotted edge). These three edges form a triangle in the derivation graph.

\begin{example}
    Consider the causal graph in Figure~\ref{fig:derivation_napkin}, for which the following derivations hold:
    \begin{align*}
        P(y \mid \Do(w),x,z) &= P(y \mid \Do(w),x) && \text{by }\Rd_1(Z).
    \end{align*}
    Similarly, the following derivations hold:
    \begin{align*}
        P(y \mid \Do(w),x,z)  &= P(y \mid \Do(w,z),x) && \text{by }\Ru_2(Z)\\
        &= P(y \mid \Do(w),x) && \text{by }\Rd_3(Z).
    \end{align*}
This can be read in its derivation graph: there is a triangle between $P(y \mid \Do(w),x)$, $P(y \mid \Do(w,z),x)$, $P(y \mid \Do(w),x,z)$. The same analysis holds for all gray edges. 

\end{example}

\subsection{Internal Commutativity of $\mathcal{R}_2$}

Rule~$\mathcal{R}_2$ governs the exchange between an observation and an intervention. In this section, we show that applications of this rule commute, meaning that the order in which observations are converted into interventions (or vice versa) does not affect the validity of the derivation.

\begin{restatable}{theorem}{thRtwocommute}
\label{th:r2_commute}
Let $\G=(\V_\G, \E_\G)$ be an ADMG and let $\Y,\X,\W,\Z,\widetilde\Z$ be pairwise disjoint subsets of~$\V_\G$, with $\Y \neq \emptyset$. Consider an expression of the form $P(\mathbf{y} \mid \Do(\mathbf{x}), \mathbf{w})$. The following properties hold:
\begin{itemize}
    \item 
        $\mathcal{G} \models \Ru_2(\Z)\,\Ru_2(\widetilde\Z)        \;\Leftrightarrow\;
\mathcal{G} \models \Ru_2(\Z,\widetilde\Z) \\       \;\Leftrightarrow\;
\mathcal{G} \models \Ru_2(\widetilde\Z)\,\Ru_2(\Z)$.
    \item 
     $\mathcal{G} \models \Rd_2(\Z)\,\Ru_2(\widetilde\Z)        \;\Leftrightarrow\;
\mathcal{G} \models \Ru_2(\widetilde\Z)\,\Rd_2(\Z)$.
\end{itemize}
The dual properties hold as well, obtained by applying the rules in the reverse direction.
\end{restatable}

These equivalences show that applications of Rule~$\mathcal{R}_2$ commute: transforming multiple observations into interventions can be performed in any order, yielding equivalent derivations. However, in contrast with the first case, the second item does not admit a simultaneous application of the two transformations; the commutativity holds only at the level of sequential rule applications.

In terms of the derivation graph, successive applications of Rule~$\mathcal{R}_2$ correspond to consecutive orange dashed edges. The {commutativity of these applications} implies that these edges form quadrilateral patterns: two distinct paths connect the same initial and final expressions.

\begin{example}
Consider the causal graph in Figure~\ref{fig:derivation_napkin}, for which the following derivations hold:
\begin{align*}
    P(y\mid \Do(w), x,z) &= P(y\mid \Do(w,x), z) && \text{by }\Ru_2(X)\\  
    &= P(y\mid \Do(w,x,z)) &&\text{by }\Ru_2(Z).
\end{align*}
Similarly, the following derivations hold:
\begin{align*}
      P(y\mid \Do(w), x,z) &= P(y\mid \Do(w,x,z)) && \text{by }\Ru_2(X,Z).
\end{align*}
This can also be observed in its derivation graph: there is a orange dashed quadrilateral formed by $P(y\mid \Do(w), x,z)$, $P(y\mid \Do(w,x), z)$, $ P(y\mid \Do(w,x,z))$ and $ P(y\mid \Do(w,z),x)$.

\end{example}

\subsection{Internal Commutativity of $\mathcal{R}_3$}

Rule~$\mathcal{R}_3$ governs the insertion or deletion of an action. We  examine how this rule composes with itself. In contrast with $\mathcal{R}_2$,  applications of Rule~$\mathcal{R}_3$ are not always commutative.

\begin{restatable}[]{theorem}{thRtrois}
\label{th:r3_commute}
Let $\G = (\V_\G,\E_\G)$ be an ADMG, and let $\Y,\X,\W,\Z,\widetilde\Z$ be pairwise disjoint subsets of~$\V_\G$, with $\Y \neq \emptyset$. Consider an expression of the form $P(\mathbf{y} \mid \Do(\mathbf{x}), \mathbf{w})$. The following properties hold:
    \begin{itemize}
        \item  $\mathcal{G}\models \mathcal{R}_3^{\uparrow}(\Z)\, \mathcal{R}_3^{\uparrow}(\widetilde\Z) \;\Leftrightarrow\; \mathcal{G}\models \mathcal{R}_3^{\uparrow}(\Z,\widetilde\Z) \\
        \;\Leftrightarrow\; \mathcal{G}\models \mathcal{R}_3^{\uparrow}(\widetilde\Z)\, \mathcal{R}_3^{\uparrow}(\Z)$.
        \item $\mathcal{G}\models  \mathcal{R}_3^{\downarrow}(\widetilde\Z)\, \mathcal{R}_3^{\uparrow}(\Z) \Rightarrow  \mathcal{G}\models \mathcal{R}_3^{\uparrow}(\Z)\, \mathcal{R}_3^{\downarrow}(\widetilde\Z) $.
    \end{itemize}
    The dual properties hold as well, obtained by applying the rules in the reverse direction.
\end{restatable}

The first item shows that insertions of multiple actions commute: simultaneous or sequential insertions lead to equivalent derivations. 
In the derivation graph, this corresponds to the presence of blue dotted quadrilateral-shaped patterns.

\begin{example}
Consider the causal graph in Figure~\ref{fig:derivation_napkin}, for which the following derivations hold:
\begin{align*}
    P(y \mid \Do(x,z)) &= P(y \mid \Do(x))&& \text{by }\Rd_3(Z)\\
    &= P(y \mid \Do(w,x)) &&\text{by }\Ru_3(W).
\end{align*}
Similarly, the following derivations hold:
\begin{align*}
      P(y \mid \Do(x,z)) &= P(y \mid \Do(w,x,z)) && \text{by }\Ru_3(W)\\
      &= P(y \mid \Do(w,x)) &&\text{by }\Rd_3(Z).
\end{align*}
This is also observed in its derivation graph: there is a blue-dotted quadrilateral formed by $P(y \mid \Do(x,z))$, $P(y \mid \Do(x))$, $P(y \mid \Do(w,x))$, $P(y \mid \Do(w,x,z))$.
\end{example}

The second item establishes only a one-way implication, as the converse does not hold. Example~\ref{ex:R3_commute_pas} shows that insertion and deletion of actions do not commute in general. In the derivation graph, unlike in the case of Rule~$\mathcal{R}_2$, not every pair of blue dotted edges forms a valid quadrilateral, and only one of the two possible diagonals corresponds to a valid derivation. This asymmetry reflects the non-commutativity of action insertion and deletion. It also contributes to the increased complexity of do-calculus reasoning.

\begin{example}
\label{ex:R3_commute_pas}
    Let us consider the graph in Figure \ref{fig:derivation_napkin}. One has the following identities:
    \begin{align*}
        P(y \mid \Do(w),x) &= P(y \mid \Do(w,z),x) && \text{by }\Ru_3(Z)\\
        &=P(y \mid \Do(z),x) && \text{by }\Rd_3(W).
    \end{align*}
    However, the following identities \emph{does not hold}:
    \begin{align*}
        P(y \mid \Do(w),x) &= P(y \mid x) && \text{by }\Rd_3(W)\\
        P(y \mid \Do(z),x) &= P(y \mid x) && \text{by }\Rd_3(Z).
    \end{align*}


\begin{minipage}{0.42\linewidth}
\resizebox{\linewidth}{!}{%
\begin{tikzpicture}[
    dofree/.style = {
      draw=coolorange!60!black,
      fill=coolorange!80,
      line width=0.6pt,
      inner sep=1pt,
      font=\footnotesize
    },
    withdo/.style = {
      draw=blue!60!black,
      fill=blue!25,
      line width=0.6pt,
      inner sep=1pt,
      font=\footnotesize
    },
    styleR1/.style = {thick, -, gray, line width=1.2pt},
    styleR2/.style = {thick, -, orange!60!brown, dashed, line width=1.2pt},
    styleR3/.style = {thick, -, blue!55, dotted, line width=1.5pt},
    warning/.style = {thick, -, red!70!white},
    scale=3.4
]
    \node[withdo] (n0) at (1.55, -.42) {$P(y \mid \Do(x))$};
    \node[withdo] (n1) at (1.85, -.12) {$P(y \mid \Do(x,z))$};
    \node[withdo] (n2) at (.6, -.3) {$P(y \mid \Do(w,x))$};
    \node[withdo] (n4) at (0.6, 0.7) {$P(y \mid \Do(w),x)$};
    \node[withdo] (n5) at (1, 0) {$P(y \mid \Do(w,x,z))$};
    \node[withdo] (n6) at (1, 1) {$P(y \mid \Do(w,z),x)$};
    \node[withdo] (n8) at (1.85, .88) {$P(y \mid \Do(z),x)$};
    \node[draw, fill = red!30, line width=0.6pt,
      inner sep=1pt,font=\footnotesize] (n9) at (1.55, .58) {$P(y \mid x)$};

\begin{pgfonlayer}{background}
    \draw[styleR3] (n0) -- (n1);
    \draw[styleR3] (n0) -- (n2);
    \draw[styleR3] (n1) -- (n5);
    \draw[styleR2] (n1) -- (n8);
    \draw[styleR2] (n2) -- (n4);
    \draw[styleR3] (n2) -- (n5);
    \draw[styleR3] (n4) -- (n6);
    \draw[styleR2] (n5) -- (n6);
    \draw[styleR3] (n6) -- (n8);
    \draw[warning] (n9) -- (n8) node[pos=0.5, red, font=\Large] {$\times$};
    \draw[warning] (n9) -- (n4) node[pos=0.5, red, font=\Large] {$\times$};
\end{pgfonlayer}
\end{tikzpicture}
}
\end{minipage}
\hfill
\begin{minipage}{0.57\linewidth}
This can also be seen in the derivation graph: there is no quadrilateral formed by 
$P(y \mid \Do(w),x)$, $P(y \mid \Do(w,z),x)$, $P(y \mid \Do(z),x)$, because
$P(y \mid x)$ does not belong to the same connected component.
\end{minipage}
\end{example}

\subsection{Weak Commutativity Between $\mathcal{R}_2$ and $\mathcal{R}_3$ }

We now study the interaction between Rules~$\mathcal{R}_2$ and~$\mathcal{R}_3$.
Compositions of Rules~$\mathcal{R}_2$ and~$\mathcal{R}_3$ commute only when the two rules are applied in the same direction, according to the $\uparrow/\downarrow$ convention.

\begin{restatable}[]{theorem}{thRdeuxRtrois}
\label{th:r2R3_commute}
Let $\G = (\V_\G,\E_\G)$ be an ADMG, and let $\Y,\X,\W,\Z,\widetilde\Z$ be pairwise disjoint subsets of~$\V_\G$, with $\Y \neq \emptyset$. Consider an expression of the form $P(\mathbf{y} \mid \Do(\mathbf{x}), \mathbf{w})$. The following properties hold:
    \begin{itemize}
            \item $\mathcal{G}\models \Rd_3(\widetilde\Z) \,  \Rd_2(\Z) \; \Leftrightarrow\; \mathcal{G}\models \Rd_2(\Z) \,  \Rd_3(\widetilde\Z)$.
            \item $\mathcal{G}\models \Rd_3(\widetilde\Z) \,  \Ru_2(\Z) \;\Rightarrow\; \mathcal{G}\models \Ru_2(\Z) \,  \Rd_3(\widetilde\Z)$.
        \end{itemize}
    The dual properties hold as well, obtained by applying the rules in the reverse direction.
\end{restatable}

The equivalence of the first statements shows that deleting an intervention and transforming an intervention into an observation commute: the order in which the two deletions are performed does not affect the validity of the derivation.
In the derivation graph, this corresponds to the presence of 
two-coloured quadrilateral-shaped patterns. 

\begin{example}
Consider the causal graph in Figure~\ref{fig:derivation_napkin}, for which the following derivations hold:
\begin{align*}
    P(y\mid \Do(w),x) &= P(y\mid \Do(w,z),x)&& \text{by }\Ru_3(Z)\\
    &= P(y\mid \Do(w,x,z)) &&\text{by }\Ru_2(X).
\end{align*}
Similarly, the following derivations hold:
\begin{align*}
      P(y\mid \Do(w),x) &= P(y\mid \Do(w,x))&& \text{by }\Ru_2(X)\\
      &= P(y\mid \Do(w,x,z)) &&\text{by }\Ru_3(Z).
\end{align*}
 This is also observed in its derivation graph: there is a two-coloured  quadrilateral formed by $P(y \mid \Do(w),x)$, $P(y \mid \Do(w,z),x)$, $P(y \mid \Do(w,x,z)$ and $P(y\mid \Do(w,x))$. \end{example}
 
In contrast, the second item establishes only a one-way implication. The converse does not hold in general (see Example~\ref{ex:r2R3_commute_pas}), revealing an additional source of asymmetry in do-calculus derivations.
In the derivation graph, this behavior translates into partial commutativity patterns between orange dashed edges (corresponding to Rule~$\mathcal{R}_2$) and blue dotted edges (corresponding to Rule~$\mathcal{R}_3$).

\begin{example}
\label{ex:r2R3_commute_pas}
    Let us consider the graph in Figure \ref{fig:derivation_napkin}. We have the following identities:
    \begin{align*}
        P(y \mid \Do(x)) &= P(y \mid \Do(x,z)) && \text{by }\Ru_3(Z)\\
        &=P(y \mid \Do(z),x) && \text{by }\Rd_2(X).
    \end{align*}
    However, the following identities \emph{does not hold}:
    \begin{align*}
        P(y \mid \Do(x)) &= P(y \mid x) && \text{by }\Rd_2(X)\\
        P(y \mid \Do(z),x) &= P(y \mid x) && \text{by }\Ru_3(Z).
    \end{align*}

\begin{minipage}{0.42\linewidth}
\resizebox{\linewidth}{!}{%
\begin{tikzpicture}[
    dofree/.style = {
      draw=coolorange!60!black,
      fill=coolorange!80,
      line width=0.6pt,
      inner sep=1pt,
      font=\footnotesize
    },
    withdo/.style = {
      draw=blue!60!black,
      fill=blue!25,
      line width=0.6pt,
      inner sep=1pt,
      font=\footnotesize
    },
    styleR1/.style = {thick, -, gray, line width=1.2pt},
    styleR2/.style = {thick, -, orange!60!brown, dashed, line width=1.2pt},
    styleR3/.style = {thick, -, blue!55, dotted, line width=1.5pt},
    warning/.style = {thick, -, red!70!white},
    scale=3.4
]
    \node[withdo] (n0) at (1.55, -.42) {$P(y \mid \Do(x))$};
    \node[withdo] (n1) at (1.85, -.12) {$P(y \mid \Do(x,z))$};
    \node[withdo] (n2) at (.6, -.3) {$P(y \mid \Do(w,x))$};
    \node[withdo] (n4) at (0.6, 0.7) {$P(y \mid \Do(w),x)$};
    \node[withdo] (n5) at (1, 0) {$P(y \mid \Do(w,x,z))$};
    \node[withdo] (n6) at (1, 1) {$P(y \mid \Do(w,z),x)$};
    \node[withdo] (n8) at (1.85, .88) {$P(y \mid \Do(z),x)$};
    \node[draw, fill = red!30, line width=0.6pt,
      inner sep=1pt,font=\footnotesize] (n9) at (1.55, .58) {$P(y \mid x)$};

\begin{pgfonlayer}{background}
    \draw[styleR3] (n0) -- (n1);
    \draw[styleR3] (n0) -- (n2);
    \draw[styleR3] (n1) -- (n5);
    \draw[styleR2] (n1) -- (n8);
    \draw[styleR2] (n2) -- (n4);
    \draw[styleR3] (n2) -- (n5);
    \draw[styleR3] (n4) -- (n6);
    \draw[styleR2] (n5) -- (n6);
    \draw[styleR3] (n6) -- (n8);
    \draw[warning] (n9) -- (n0) node[pos=0.5, red, font=\Large] {$\times$};
    \draw[warning] (n9) -- (n8) node[pos=0.5, red, font=\Large] {$\times$};
\end{pgfonlayer}
\end{tikzpicture}
}
\end{minipage}
\hfill
\begin{minipage}{0.57\linewidth}
This is also observed in the derivation graph: there is no quadrilateral formed by 
 $P(y \mid \Do(x,z))$, $P(y \mid \Do(x))$, $P(y \mid \Do(z),x)$, because
$P(y \mid x)$ does not belong to the same connected component.
\end{minipage}

\end{example}

\subsection{Normal Forms of Do-Calculus Derivations}

We conclude our study of the interactions between the do-calculus rules by showing that any valid derivation can be represented in a canonical sequence\footnote{Rules~$\Ru_2$ and~$\Ru_3$ commute, as do their downward counterparts.
As a result, permuting these blocks yields four equivalent canonical sequences.
We fix one ordering here.}. 

\begin{restatable}[]{theorem}{thformenormal}
\label{th:formenormal}
    Let $\G=(\V_\G,\E_\G)$ be an ADMG, and let $\Y,\X,\W$ and $\Y,\widetilde\X,\widetilde\W$ be pairwise disjoint subsets of~$\V_\G$, with $\Y \neq \emptyset$. Then $P(\mathbf{y}\mid \Do(\mathbf{x}),\mathbf{w}) $ and $ P(\mathbf{y} \mid \Do(\mathbf{\widetilde x}),\mathbf{\widetilde w})$ are equivalent if and only if  the sequential application
    $$\Ru_2(\Z_2^\uparrow)\, \Ru_3(\Z_3^\uparrow)\,
    \Rd_3(\Z_3^\downarrow)\,
    \Rd_2(\Z_2^\downarrow),$$
    with $\Z_2^\uparrow \coloneqq \W \setminus \widetilde \W$, $\Z_3^\uparrow \coloneqq (\widetilde \X \cup \widetilde \W) \setminus (\X \cup  \W)$, $\Z_2^\downarrow \coloneqq \widetilde \W \setminus \W$ and $\Z_3^\downarrow \coloneqq (\X \cup  \W) \setminus (\widetilde \X \cup \widetilde \W)$, 
    is valid in $\G$ {and transforms $ P(\mathbf{y}\mid \Do(\mathbf{x}),\mathbf{w})$ into $ P(\mathbf{y} \mid \Do(\mathbf{\widetilde x}),\mathbf{\widetilde w})$}. 
\end{restatable}

In other words, any valid derivation between two interventional distributions can be reordered, without loss of validity, into a canonical sequence: observations are first converted into interventions, then interventions are introduced or eliminated, and finally interventions are converted back into observations. This follows directly from the commutativity properties proved above.

Building on this result, the next corollary gives a graphical criterion to check whether two interventional distributions are equivalent.
\begin{restatable}[]{corollary}{corcriteregraphique}
\label{cor:criteregraphique}
    Let $\G=(\V_\G,\E_\G)$ be an ADMG, and let $\Y, \X, \W$ and $\Y, \widetilde\X, \widetilde\W$ be respectively pairwise disjoint subsets of~$\V_\G$, with $\Y\neq \emptyset$. Let $\Z_2^\uparrow$, $\Z_3^\uparrow$, $\Z_2^\downarrow$ and $\Z_3^\downarrow$ be defined as in Theorem~\ref{th:formenormal}.
    
    Then, the following graphical criterion:
    \begin{align*}
        &\Y \ind_{\G_{\overline{\X}\underline{\Z_2^\uparrow}}} \Z_2^\uparrow \mid \X, \W \cap \widetilde\W \text{ and } 
        \Y \ind_{\G_{\overline{\X,\Z_2^\uparrow}\,\overline{\Z_3^\uparrow(\W \cap \widetilde\W)}}} \Z_3^\uparrow \mid \X, \W \\
        \text{ and } &
        \Y \ind_{\G_{\overline{\widetilde\X}\underline{\Z_2^\downarrow}}} \Z_2^\downarrow \mid \widetilde\X ,\W \cap \widetilde\W
        \text{ and }
        \Y \ind_{\G_{\overline{\widetilde\X,\Z_2^\downarrow}\,\overline{\Z_3^\downarrow(\W \cap \widetilde\W)}}} \Z_3^\downarrow \mid \widetilde\X, \widetilde\W ,
    \end{align*}
    is both sound and complete to determine whether $P(\mathbf{y}\mid \Do(\mathbf{x}),\mathbf{w})$ and $ P(\mathbf{y} \mid \Do(\mathbf{\widetilde x}),\mathbf{\widetilde w})$ are equivalent.
\end{restatable}

The above graphical
criterion amounts to performing four d-separation tests on mutilated
graphs. As all the sets $\Z_2^\uparrow$, $\Z_3^\uparrow$, $\Z_2^\downarrow$ and
$\Z_3^\downarrow$ are explicitly defined, the corresponding
graph mutilations are straightforward to construct,  and since d-separation can be checked in linear time \citep{geiger_identifying}, this criterion is computationally efficient.

Causal reasoning typically relies on successive applications of do-calculus rules. Corollary~\ref{cor:criteregraphique} instead provides a necessary and sufficient graphical criterion that captures the validity of such sequences as a whole, allowing practitioners to identify the minimal assumptions underlying a causal argument, rather than relying on \textit{ad hoc} or overly strong hypotheses.

\begin{example}
Let us consider the graph in Figure \ref{fig:derivation_napkin}. We have:
$$P(y\mid \Do(w), x,z) = P(y\mid \Do(w,x,z)) = P(y \mid do(x))$$
by using $\Ru_2(X,Z) \, \Rd_3(W,Z)$. 
\end{example}

Finally, the canonical structure allows us to bound the distance between expressions in the derivation graph, yielding the following structural result.
\begin{restatable}[]{corollary}{cordiametre}
Let $\G$ be an ADMG over $\V$, and let $\mathcal{D}$ be its derivation graph.
Then, the diameter\footnote{The graph diameter is the length of the longest shortest path between any two vertices in a connected graph.} of each connected component in $\mathcal{D}$ is at most 4.
\end{restatable}
Equivalently, any two derivations of the same interventional distribution are connected by a path of length at most four in the derivation graph.
This bound is not directly visible in Figures~\ref{fig:derivation_lattice_AfBfC} and~\ref{fig:derivation_lattice_do_bc}, which show only atomic applications of the do-calculus rules. Moreover, the bound is tight. Indeed, consider an ADMG with no edges among
$Y,W_1,W_2,X_1,$ and $X_2$. Transforming
$P(y \mid \mathrm{do}(x_1), w_1)$ into
$P(y \mid \mathrm{do}(x_2), w_2)$ requires exactly four applications of do-calculus:
two to remove $x_1$ and $w_1$, and two to introduce $x_2$ and $w_2$.
Therefore, the corresponding derivations are at distance four in the derivation graph.

\section{Application: Deriving Several Identification Formulae From an Interventional Query}
\label{sec:app}

We have studied the set of expressions equivalent under do-calculus. Here, we illustrate how this perspective can be used for identification and, more importantly, estimation, by comparing several equivalent formulae for the same causal query. 

\subsection{Estimation in a Simple Case}

We return to the example introduced in Figure~\ref{fig:1} and Example~\ref{ex:1}, and consider the query $\mathbb{E}(Y \mid \Do(x), w)$\footnote{The paper is written with discrete probability, but all results directly translate to continuous random variables, which we consider in this section.}.
From the derivation graph given in Appendix \ref{app:example_estimation:1}, this query admits the following set of equivalent interventional probabilities:
$$\mathcal{A}^{\mathcal{G}}_{ \mathbb{E}(Y\mid \Do(w,z))}
= \{\mathbb{E}(Y \mid \Do(w),z),\mathbb{E}(Y\mid \Do(w,z)), \mathbb{E}(Y\mid \Do(z))\}.$$

We generate data from a linear Gaussian model based on Figure~\ref{fig:1} and assume linearity and Gaussian noise for estimation. Details on the data generation, estimators, and estimation procedure are provided in Appendix~\ref{app:example_estimation:1}.


We first estimate $ \mathbb{E}(Y\mid \Do(z))$: 
\begin{align*}
    \mathbb{E}(Y\mid \Do(z)) 
    & = \int_y y \int_w f(y \mid z,w) f(w) dw dy\\
    &= \gamma + \alpha \mathbb{E}(W) + \beta z,
\end{align*}
where we fit a regression model $\mathbb{E}(Y \mid z,w) = \gamma + \alpha w +\beta z$. {The total effect of Z on Y, defined by $\frac{\partial}{\partial z} \mathbb{E}(Y \mid \Do(z))$, is $\beta$.}

We then estimate $\mathbb{E}(Y\mid \Do(w,z))$:
\begin{align*}
\mathbb{E}(Y\mid \Do(w,z)) 
     =& \int_y y \int_x f(y \mid w,z,x) f(x \mid w,z) f(w) dx dy\\
    =&  \int_x f(x \mid w,z) f(w) (\widetilde\gamma + \widetilde\alpha x + \widetilde\beta w + \widetilde\delta z) dx\\
    %
    %
    =& ~{\widetilde\gamma} + \widetilde\alpha \mathbb{E}(X) + \widetilde\beta \mathbb{E}(W) + \widetilde\delta(\gamma' + \alpha' X + \beta' w),
\end{align*}
where we fit two regression models, $\mathbb{E}(Y \mid x,w,z) = \widetilde\gamma + \widetilde\alpha x +\widetilde\beta w + \widetilde\delta z$ and 
 $\mathbb{E}(Z \mid x,w) = \gamma' + \alpha' x +\beta' w$. The total effect of $X,W$ on $Y$ is $(\widetilde\delta \alpha', \widetilde\delta \beta')$.

Although these expressions are causally equivalent, they produce estimators with markedly different statistical properties. Figure~\ref{fig:boxplot} compares the two estimators over $1000$ runs. While both are unbiased for the causal effect of $Z$ on $Y$, the frontdoor-based estimator (middle) exhibits substantially larger variance than the backdoor estimator (left). Similarly, although the causal effect of $X$ on $Y$ is zero, the estimator derived from the frontdoor-type formula (right) shows a large variance.
This illustrates that, even when identification is guaranteed, the choice among equivalent formulae can have a major impact on estimation performance.

\begin{figure}
    \centering
    \includegraphics[width=0.49\textwidth]{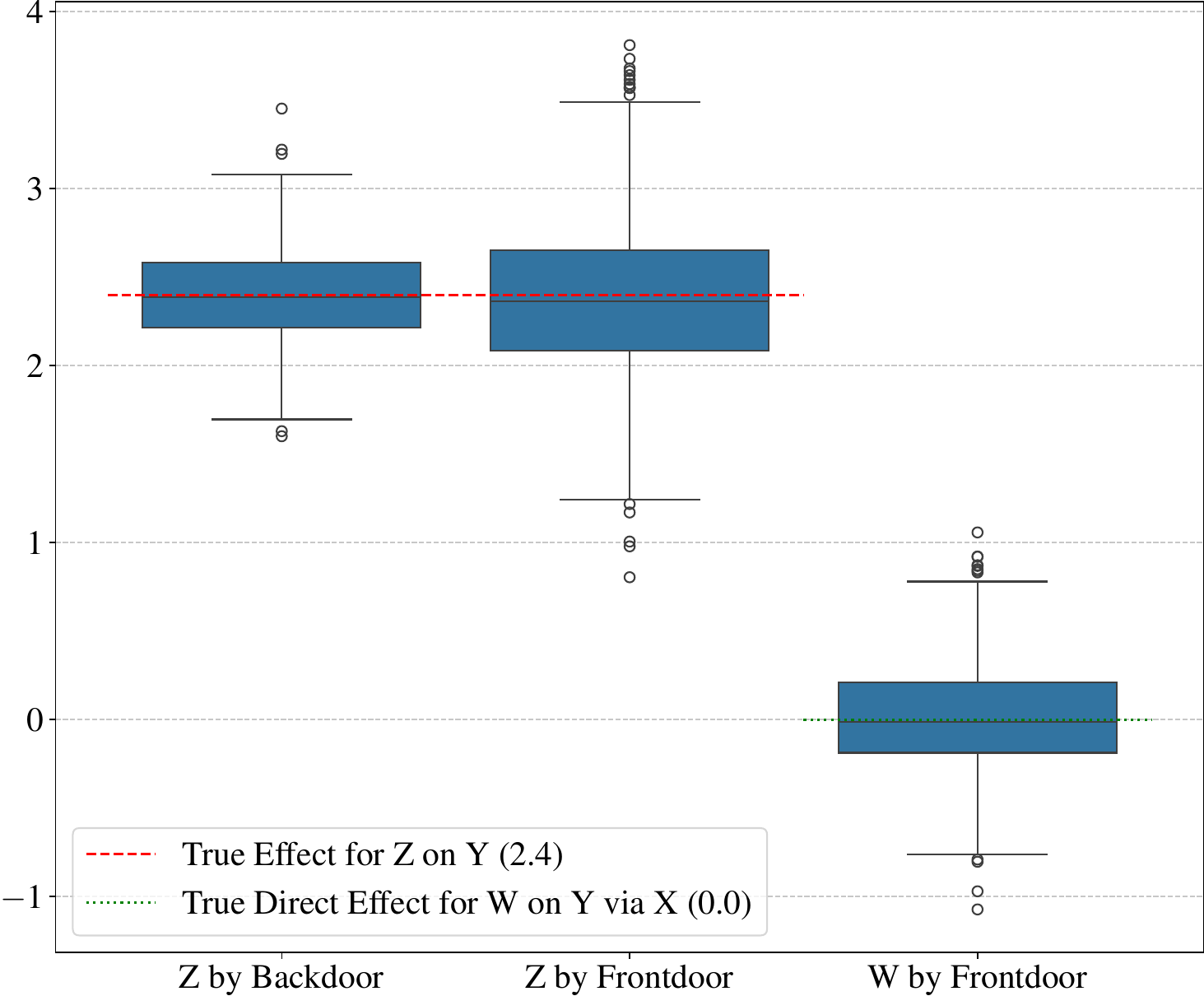}
    \caption{Comparison of two estimators associated with two equivalent identification formulae over $1000$ runs.}
\label{fig:boxplot}
\end{figure}

\subsection{Application to a Real Biological Dataset}
\label{sec:real_data}

\begin{table}[t]
    \centering
    \caption{Estimators obtained from the equivalent identification formulae for the causal effect $P(\mathrm{P38} \mid do(\mathrm{Mek}))$ on the Sachs protein-signaling dataset.}
    \label{tab:sachs_results}
    \begin{tabular}{l r r}
    \hline
    Density & Prediction & Variance \\
    \hline
    $P(\mathrm{P38} \mid do(\mathrm{Mek}))$ & 0.02272 & 0.000572 \\
    $P(\mathrm{P38})$ & -0.018090 & 0.001029 \\
    $P(\mathrm{P38} \mid do(\mathrm{Mek},\mathrm{Raf}))$ & -0.00616 & 0.000265 \\
    $P(\mathrm{P38} \mid do(\mathrm{Akt},\mathrm{Mek}))$ & 0.02149 & 0.000574 \\
    $P(\mathrm{P38} \mid do(\mathrm{Erk}))$ & -0.00445 & 0.000065 \\
    $P(\mathrm{P38} \mid do(\mathrm{Akt},\mathrm{Erk}))$ & -0.00495 & 0.000064 \\
    $P(\mathrm{P38} \mid do(\mathrm{Akt},\mathrm{Erk},\mathrm{Mek},\mathrm{Raf}))$ & -0.00723 & 0.000262 \\
    $P(\mathrm{P38} \mid do(\mathrm{Akt},\mathrm{Erk},\mathrm{Jnk}))$ & -0.00181 & \textbf{0.000016} \\
    $P(\mathrm{P38} \mid do(\mathrm{Akt},\mathrm{Jnk}))$ & -0.00730 & 0.000269 \\
    $P(\mathrm{P38} \mid do(\mathrm{Akt}))$ & -0.01993 & 0.001032 \\
    $P(\mathrm{P38} \mid do(\mathrm{Erk},\mathrm{Jnk}))$ & -0.00154 & \textbf{0.000017} \\
    \hline
    \end{tabular}
\end{table}

We further apply our approach on the protein signaling dataset of \citep{Sachs2005}, a standard benchmark in causal inference. This dataset consists of single-cell measurements of protein and phosphoprotein expression levels collected under several experimental conditions, together with a biologically validated causal graph that we use as ground truth. In this experiment, we focus on the causal effect $P(\mathrm{P38} \mid do(\mathrm{Mek}))$, which is known to be null according to the reference DAG. 
Starting from the causal graph, the derivation graph produced by our method contains 32 equivalent interventional densities. Since the underlying graph is a DAG, each density leads to an identifiable adjustment formula via parental adjustment. Several of these densities induce the same adjustment set and therefore correspond to the same estimator. For each distinct adjustment set, we estimate the causal effect and bootstrap variances using 500 resamples. Table~\ref{tab:sachs_results} reports the distinct estimators obtained from our approach, together with their  point estimates and bootstrap variances. All estimators return values close to zero, in agreement with the ground truth, and the different equivalent formulae exhibit noticeably different variances. More details are provided in  Appendix~\ref{app:sachs}.

\subsection{Discussion Over More Complex Cases}
In more complex graphs, multiple identification formulae may coexist, and it is often unclear which one is preferable from an estimation
perspective. Understanding the structure and diversity of equivalent formulae, as captured by the derivation graph, can therefore guide the choice of estimators. We provide in Appendix~\ref{app:details} an example of a more complex graph and a class of interventional queries, together with their identification formulae obtained using the ID algorithm via the \texttt{causaleffect} R package~\cite{JSSv076i12}.

\section{Conclusion and Perspectives}
\label{sec:conc}

We introduced in this work derivation graphs, which provide a simple view on how observational and interventional probabilities are related through do-calculus rules. By organizing derivations into canonical sequences, this representation clarifies how the rules of do-calculus interact and explains the combinatorial structure underlying equivalent expressions, derived from one another by the recursive application of any number of do-calculus rules. We proved in particular that equivalent expressions can be related by only two applications of Rule $\mathcal{R}2$ and two applications of Rule $\mathcal{R}3$, and we introduced an efficient graphical criterion to decide whether two expressions are equivalent or not.

An important implication of this perspective concerns estimation. While equivalent expressions are causally identical, they may yield identification formulae with very different statistical behavior. The derivation graph makes this multiplicity explicit and offers a principled way to enumerate and compare alternative representations of an interventional query. This underscores that identification alone is not sufficient: choosing one formula among several equivalent ones can significantly impact practical estimation.

An interesting direction for future work is to investigate whether an analogue of derivation graphs can be developed for the reductions underlying the ID algorithm. While the ID algorithm outputs a single derivation leading to an identification formula, one may instead consider the larger space of all valid rewritings obtained through probabilistic manipulations and do-calculus rules. In such a framework, nodes would correspond to algebraic expressions obtained from conditional and interventional distributions through marginalizations, products, and divisions, while edges would represent valid transformations between expressions. The derivation produced by ID would then correspond to a particular path from the initial query to a do-free formula.

A major challenge, however, is that the resulting graph would generally be infinite, even within a single connected component, since probabilistic identities can always be inserted into a formula. For instance, one may multiply an expression by identities such as
$\sum_{\mathbf{z}} P(\mathbf{z}\mid do(\mathbf{x}),\mathbf{w}) = 1$
without changing its value. Consequently, unlike the derivation graphs studied in this work, such extended derivation spaces would not admit finite enumeration in general. Understanding whether meaningful finite representations, canonical reductions, or tractable subclasses can nevertheless be characterized remains an open problem.

\section*{Acknowledgements}

This work was partly supported by fundings from the French government, managed by the National Research Agency (ANR), under the France 2030 program, reference ANR-23-IACL-0006 and ANR-23-PEIA-0007.

\section*{Impact Statement}
This paper presents work whose goal is to advance the field of Causality. There are many potential societal consequences of our work, none
which we feel must be specifically highlighted here.

\newpage
\bibliography{References}
\bibliographystyle{icml2026}

\newpage
\appendix

\setcounter{theorem}{0}
\setcounter{lemma}{0}
\setcounter{definition}{0}
\setcounter{proposition}{0}
\setcounter{example}{0}

\renewcommand{\thetheorem}{\Alph{theorem}}
\renewcommand{\thelemma}{\Alph{lemma}}
\renewcommand{\thedefinition}{\Alph{definition}}
\renewcommand{\theproposition}{\Alph{proposition}}
\renewcommand{\theexample}{\Alph{example}}

\onecolumn
\vspace{1cm}

\begin{center}
{ \Large  Supplementary Materials }
\end{center}
\vspace{1cm}

\section{Proofs of the Results of Section~\ref{sec:derivationGraph}: Derivation Graphs and Equivalent Expressions}

\propemptyADMG*

\begin{proof}
    We first show the claim on the number of connected components.
    Since do-calculus rules operate only on the intervention and conditioning variables, no sequence of do-calculus rules can transform an expression with outcome variables $\Y$ into an expression with different outcome variables $\Y' \neq \Y$.
    As a consequence, expressions involving different non-empty outcome sets belong to distinct connected components of $\mathcal{D}[\mathcal{G}_0]$.
    
    There are exactly $2^{|\V_{\G_0}|}-1$ non-empty subsets of $\V_{\G_0}$, hence $\mathcal{D}[\mathcal{G}_0]$ contains at least $2^{|\V_{\G_0}|}-1$ connected components.
    
    Conversely, since $\mathcal{G}_0$ contains no edges, all variables are mutually independent.
    Therefore, for any fixed non-empty outcome set $\Y$, all expressions of the form
    \[
    P(\mathbf{y} \mid \Do(\mathbf{x}), \mathbf{w}),
    \qquad
    \mathbf{X}, \mathbf{W} \subseteq \V_{\G_0} \setminus \mathbf{Y}, \X\cap\W = \emptyset,
    \]
    Hence, all such expressions belong to the same connected component.
    This shows that $\mathcal{D}[\mathcal{G}_0]$ has exactly $2^{|\V_{\G_0}|}-1$ connected components, each corresponding to a non-empty outcome set $\Y$.
    
    Finally, let $\mathcal{G}$ be any causal graph with vertex set $\V_{\G_0}$.
    Any conditional independence or graphical criterion that holds in $\mathcal{G}$ also holds in the empty graph $\mathcal{G}_0$.
    As a consequence, every valid do-calculus rule in $\mathcal{D}[\mathcal{G}]$ is also valid in $\mathcal{D}[\mathcal{G}_0]$. Therefore,
    \[
    \mathcal{D}[\mathcal{G}] \subseteq \mathcal{D}[\mathcal{G}_0].
    \]
\end{proof}

\proptailleconnectedcomponent*
\begin{proof}
    Consider the ADMG $\G_0 \coloneqq (\V_\G,\emptyset)$. By Proposition~\ref{prop:emptyADMG}, \(\mathcal{D}[\mathcal{G}] \subseteq \mathcal{D}[\mathcal{G}_0] \) and we have:
    
    \[
        \left|\mathcal{A}^{\mathcal{G}}_{{P}(\Y\mid \Do(\X),\W)} \right| 
        \leq \left|\mathcal{A}^{\mathcal{\G}_0}_{{P}(\Y\mid \Do(\X),\W)} \right| 
        = \left| \left\{ P(\mathbf{y} \mid \Do(\mathbf{x}), \mathbf{w}) : \mathbf{X}, \mathbf{W} \subseteq \V_{\G_0} \setminus \mathbf{Y}, \X\cap\W =\emptyset \right\}\right| 
        = 3^{\left| \mathcal{V}_{\mathcal{G}} \setminus \Y \right| } 
    \]
        
    Finally, every connected component contains at least one distribution, which concludes the proof.
\end{proof}

\newpage
\section{Proofs of the Results of Section~\ref{sec:commutativity}: Commutativity of Do‑Calculus Rules and Simplification of Causal Reasoning}

\paragraph{Notations.} To distinguish the type of rule application, we use the symbols $\uparrow$ and $\downarrow$: $\Ru_1(\Z)$ (resp. $\Ru_2(\Z)$ and $\Ru_3(\Z)$) denotes the insertion of an observation (resp. action) $\Z$, while $\Rd_1(\Z)$ (resp. $\Rd_2(\Z)$ and $\Rd_3(\Z)$) denotes the deletion of an observation (resp. action) $\Z$.

\subsection{Validity of a Shortcut Do-Calculus Rule}

The proofs of Section~\ref{sec:commutativity} rely on Theorem~\ref{th:cycle_do_calculus} which simplifies and accelerates the subsequent proofs. We define a \emph{shortcut rule} as a \emph{syntactically valid} single do-calculus rule that, when it exists, directly maps an interventional distribution \(Q_1\) to \(Q_2\), without assuming its graphical validity \emph{a priori}. Roughly speaking, Theorem~\ref{th:cycle_do_calculus} states that if there exists a valid sequence of do-calculus rules transforming \(Q_1\) into \(Q_2\), then the corresponding shortcut rule, whenever it exists, is itself graphically valid.

\begin{theorem}[Validity of a shortcut do-calculus rule from a graphically valid derivation]
\label{th:cycle_do_calculus}
Let $\G$ be an ADMG over $\V_\G$, and let 
\(
Q_1 \coloneqq P(\mathbf{y} \mid \Do(\mathbf{x_1}), \mathbf{w_1})),~
Q_2 \coloneqq P(\mathbf{y} \mid \Do(\mathbf{x_2}), \mathbf{w_2}))
\) 
be two interventional distributions. Suppose there exists a single do-calculus rule $\mathcal{R}$ that \emph{syntactically}\footnote{meaning that the corresponding edge exists in the derivation graph $\mathcal{D}[(\V_\G, \emptyset)]$} maps $Q_1$ to $Q_2$. 

If there exists a sequence of do-calculus rules $\mathcal{R}_{i_1} \cdots \mathcal{R}_{i_n}$ that is valid in $\G$ mapping $Q_1$ to $Q_2$, then $\mathcal{R}$ is valid in $\G$.
\end{theorem}

\begin{proof}
The proof is inspired by the construction given in the proof of Theorem 1.2.5(ii) from \citet{pearl_causality_2009} (see the proof of Theorem 3 in \citet{10.5555/647231.719429}) adapted for the SCM case and accounting for mutilations.

Assume, for contradiction, that the shortcut rule $\mathcal R$, allowing to go from query $Q_1$ to $Q_2$ is not valid in $\mathcal G$. Without loss of generality, we assume that $\mathcal R$ is applied in the downward direction ($\mathcal R_1^{\downarrow}$, $\mathcal R_2^{\downarrow}$, or $\mathcal R_3^{\downarrow}$), as the opposite case can be reduced to this one by swapping $Q_1$ and $Q_2$, 
We will construct an SCM $\mathcal M$ such that ${Q_1}_{\mathcal M} \neq {Q_2}_{\mathcal M}$. 

Let us rewrite $Q_1$ and $Q_2$ by introducing common sets $\mathbf X$, $\mathbf W$, and $\mathbf Z$ as follows:
\begin{itemize}
    \item If $\mathcal R$ is $\mathcal R_1^{\downarrow}$: set $Q_1 = P(\mathbf y \mid \mathrm{do}(\mathbf x), \mathbf z, \mathbf w)$, $Q_2 = P(\mathbf y \mid \mathrm{do}(\mathbf x), \mathbf w)$, where $\mathbf X := \mathbf X_1$, $\mathbf W := \mathbf W_2$, and $\mathbf Z := \mathbf W_1 \setminus \mathbf W_2$.
    
    \item If $\mathcal R$ is $\mathcal R_2^{\downarrow}$: set $Q_1 = P(\mathbf y \mid \mathrm{do}(\mathbf x), \mathrm{do}(\mathbf z), \mathbf w)$, $Q_2 = P(\mathbf y \mid \mathrm{do}(\mathbf x),\mathbf z, \mathbf w)$, where $\mathbf X := \mathbf X_1 \cap \mathbf X_2$, $\mathbf W := \mathbf W_1$, and $\mathbf Z := \mathbf X_1 \setminus \mathbf X_2$.
    
    \item If $\mathcal R$ is $\mathcal R_3^{\downarrow}$: set $Q_1 = P(\mathbf y \mid \mathrm{do}(\mathbf x), \mathbf w)$, $Q_2 = P(\mathbf y \mid \mathrm{do}(\mathbf x \setminus \mathbf z), \mathbf w)$, where $\mathbf X := \mathbf X_1 \cap \mathbf X_2$, $\mathbf W := \mathbf W_1$, and $\mathbf Z := \mathbf X_1 \setminus \mathbf X_2$.
\end{itemize}

We can assume that  $\mathbf X = \emptyset$, as it is intervened upon identically in both quantities. We distinguish two cases:
\begin{itemize}
    \item If $\mathcal R$ is $\mathcal R_1^{\downarrow}$ or $\mathcal R_2^{\downarrow}$. We construct the SCM $\mathcal M$ as follows:
    
    \textbf{First step.} $\mathcal R$ is not valid in $\mathcal G$, hence $\mathbf Y \;\not\!\perp\!\!\!\perp_{\tilde{\mathcal G}}\; \mathbf Z \mid \mathbf W$, where $\tilde{\mathcal G} = \mathcal G$ if $\mathcal R$ is $\mathcal R_1^{\downarrow}$ and $\tilde{\mathcal G} = \mathcal G_{\underline{\mathbf Z}}$ if $\mathcal R$ is $\mathcal R_2^{\downarrow}$. Thus, in all cases, there exists a path $\pi$ from $Y \in \mathbf Y$ to $Z \in \mathbf Z$, such that $\pi$ encounters $\mathbf Y$ only at its first vertex and $\mathbf Z$ only at its last vertex. For each collider $C$ on $\pi$, there exist a (possibly empty) directed path $\pi_C$ from $C$ to $W_C \in \mathbf W$. Moreover, without loss of generality (see claim 1 and 2 from \citet{10.5555/647231.719429}), we may assume that all $\pi_C$ are pairwise distinct and that for all $C$, $\pi$ and $\pi_C$ only meet at $C$. Then consider the graph $\G^\star$ on $\mathcal V$ whose edges are the edges of $\displaystyle \pi \;\cup\; \bigcup_{C} \pi_C$. Hence, by construction, $\G^\star \subseteq\tilde{\mathcal{G}} \subseteq \mathcal G$ and every element of $\mathbf Z \setminus \{Z\}$ is isolated in $\G^\star$. See Figure~\ref{fig:forthmA} for a representation of $\G^\star$ when $\pi$ contains colliders.

    \textbf{Second step.} We construct a linear Gaussian SCM $\mathcal M$ corresponding to $\G$ as follows (where all exogenous variables are mutually independent of each other):
     \begin{enumerate}
         \item For each bidirected edge connecting nodes $A$ and $B$ in $\mathcal G^\star$, introduce an exogenous variable $U_{\{A,B\}} \sim \mathcal N(0,1)$.
         \item For every node $V \in \mathcal V$, define an exogenous variable $\epsilon_V \sim \mathcal N\left(0,\sigma^2_V \right)$.
         \item For every node $V \in \mathcal V$, set $\displaystyle V \gets \rho \cdot\left(\sum_{P \in \mathrm{Pa}(V,\G^\star)} P + \sum_{W \in \mathrm{Sp}(V,\G^\star)} U_{\{V,W\}}\right)  + 0 \cdot \left(\sum_{P \in \mathrm{Pa}(V,\G) \setminus \mathrm{Pa}(V,\G^\star)} P \right) + \epsilon_V$, where $\mathrm{Sp}(V,\G^\star) \coloneqq \left\{ W \in \mathcal V \mid V \dashleftrightarrow W \subseteq \G^\star \right\}$ and $\rho > 0$.
     \end{enumerate}
     
    We choose $\rho$ small enough so that, for every $V \in \mathcal V$, there exists $\sigma_V^2$ such that $\mathrm{Var}(V)=1$, and we fix such a choice for each $V \in \mathcal V$. Note that $\mathcal M$ induces $\G$. Moreover, every edge in $\G \setminus \G^\star$ has coefficient $0$ in $\mathcal M$. Therefore, these additional edges do not affect any observational or interventional distribution induced by $\mathcal M$, and the distributions generated by $\mathcal M$ coincide with those obtained from the subgraph $\G^\star$. Hence, in the remainder of the proof, we may reason on $\mathcal M$ using do-calculus on $\G^\star$.
    
    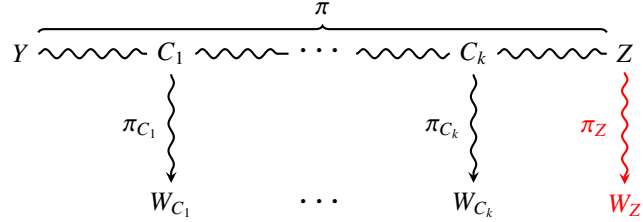
\begin{figure}[H]
\centering
\begin{tikzpicture}[node distance=2cm, thick,->, >=stealth]

    \node (Y) {$Y$};
    \node (C1) [right of=Y] {$C_1$};
    \node (Cdots) [right of=C1] {\Large $\cdots$};
    \node (Ck) [right of=Cdots] {$C_k$};
    \node (Z) [right of=Ck] {$Z$};
    \node (WC1) [below of=C1] {$W_{C_1}$};
    \node (WCdots) [below of=Cdots] {\Large $\cdots$};
    \node (WCk) [below of=Ck] {$W_{C_k}$};
    \node (WZ) [below of=Z, red] {$W_Z$};
    
    \draw[-, decorate, decoration={snake, amplitude=0.6mm, segment length=3mm}] (Y) -- (C1);
    \draw[-,decorate, decoration={snake, amplitude=0.6mm, segment length=3mm}]  (C1) -- (Cdots);
    \draw[-, decorate, decoration={snake, amplitude=0.6mm, segment length=3mm}] (Cdots) -- (Ck);
    \draw[-, decorate, decoration={snake, amplitude=0.6mm, segment length=3mm}] (Ck) -- (Z);

    \draw[decorate, decoration={snake, amplitude=1pt, segment length=3mm}, red]
    (Z) -- node[red, midway, xshift= -4mm] {$\pi_Z$} (WZ);
    
    \draw[decorate, decoration={snake, amplitude=1pt, segment length=3.6mm}]
    (C1) -- node[midway, xshift= -4mm] {$\pi_{C_1}$} (WC1);
    
    \draw[decorate, decoration={snake, amplitude=1pt, segment length=3.6mm}]
    (Ck) --  node[midway, xshift= -4mm] {$\pi_{C_k}$} (WCk);
    
    \draw[-, decorate, decoration={brace, raise=8pt}] (Y) -- node[midway, yshift=6mm] {$\pi$} (Z);

    \end{tikzpicture}
    \caption{\textbf{Graphs in the Proof of Theorem~\ref{th:cycle_do_calculus}.}
    $\G^\star$ is shown in black. $\G^\triangle$ is obtained by adding the path $\pi_Z$ in red. Wavy edges denote arbitrary paths and wavy directed edges denote directed paths. Each $C_i$ is a collider on $\pi$. The paths $\pi_{C_i}$ are pairwise distinct, intersect $\pi$ only at $C_i$, and are disjoint from $\pi_Z$. Moreover, $\pi_Z$ intersects $\pi$ only at $Z$.
    }
    \label{fig:forthmA}
\end{figure}
    
    \textbf{Third step.} Check for each rule that ${Q_1}_{\mathcal M} \neq {Q_2}_{\mathcal M}$. When working in the matrix setting, we denote $\Sigma_{\mathbf A \mathbf B}$ the covariance matrix between $\mathbf A$ and $\mathbf B$. To lighten notation, we omit the dependance in $\mathcal M$, all subsequent calculations ar calculated in $\mathcal M$.
    
    First, let us show that $\mathrm{Cov}(Y,Z \mid \mathbf W)\neq 0$. We distinguish two cases:
    \begin{itemize}
        \item  If $\pi$ contains no colliders, the path-tracing rule (see \citet{Wright1921CorrelationAndCausation}) gives $\mathrm{Cov}(Y,Z \mid \mathbf W) = \mathrm{Cov}(Y,Z) = \rho^{|\pi|} \neq 0$, where $|\pi|$ denotes the effective path length, counting bidirected edges twice. 
        
        \item  If $\pi$ contains $k > 0$ colliders (see Figure~\ref{fig:forthmA}), then $\Sigma_{YZ} = 0$, hence $\mathrm{Cov}(Y,Z \mid \mathbf W) = - \Sigma_{Y\mathbf W} \cdot \Sigma_{\mathbf W\mathbf W}^{-1} \cdot\Sigma_{\mathbf W Z}$. To perform the computations, we order the elements of $\mathbf W$ as follows: first, all variables $W_C$ corresponding to colliders $C$ on $\pi$, in the order induced by the path (i.e., as indexed in Figure~\ref{fig:forthmA}); the remaining elements are then ordered arbitrarily. With this ordering, only the first component of $\Sigma_{Y\mathbf W}$ (i.e., corresponding to $\mathrm{Cov}(Y, W_{C_1})$ in Figure~\ref{fig:forthmA}) and the $k$-th component of $\Sigma_{\mathbf W Z}$ (i.e., corresponding to $\mathrm{Cov}(W_{C_k}, Z)$ in Figure~\ref{fig:forthmA}) are nonzero. Hence $\mathrm{Cov}(Y,Z \mid \mathbf W) = - [\Sigma_{Y\mathbf W}]_1 [\Sigma_{\mathbf W\mathbf W}^{-1}]_{1k} [\Sigma_{\mathbf W Z}]_k$. By the cofactor formula, $[\Sigma_{\mathbf W\mathbf W}^{-1}]_{1k} = (-1)^{1+k} \det(\Sigma_{\mathbf W\mathbf W}^{(k,1)}) / \det(\Sigma_{\mathbf W\mathbf W})$, where $\Sigma_{\mathbf W\mathbf W}^{(k,1)}$ is obtained by removing the last row and the first column of $\Sigma_{\mathbf W\mathbf W}$. With the ordering along the colliders, this matrix is lower triangular, with diagonal entries equal to the covariances between consecutive $W_C$ along the path and $1$ for all remaining variable in $\mathbf W$. Therefore, the determinant of $\Sigma_{\mathbf W\mathbf W}^{(k,1)}$ is the product of these nonzero diagonal entries, and we conclude that $[\Sigma_{\mathbf W\mathbf W}^{-1}]_{1k} \neq 0$. Therefore, $\mathrm{Cov}(Y,Z \mid \mathbf W) \neq 0$.
    \end{itemize}
     
    In both cases, we have $\mathrm{Cov}(Y,Z \mid \mathbf W) \neq 0$. We distinguish two cases:
        
    \begin{itemize}
        \item If $\mathcal R$ is $\mathcal R_1^{\downarrow}$, we know that $\mathrm{Cov}(Y,Z \mid \mathbf W)\neq 0$. Therefore $P(y \mid z, \mathbf w) \neq P(y \mid \mathbf w)$. Since every element of $\mathbf Z \setminus \{Z\}$ is isolated in $\G^\star$, by $\mathcal R_1^\downarrow (\mathbf z \setminus z)$, we have that $P(y \mid \mathbf z, \mathbf w) = P(y \mid z, \mathbf w) \neq P(y \mid \mathbf w)$. Therefore $P(\mathbf y \mid \mathbf z, \mathbf w) \neq P(\mathbf y \mid \mathbf w)$.
    
        \item If $\mathcal R$ is $\mathcal R_2^{\downarrow}$, since $\G^\star \subseteq \tilde{\G} = \G_{\underline{\mathbf Z}}$, we know that  $Z$ is isolated in $\G^\star_{\overline{Z(\{Y\} \cup \mathbf W)}}$. Hence, $\{Y\} \cup \mathbf W \ind_{\G^\star_{\overline{Z(\{Y\} \cup \mathbf W)}}} Z$. By $\mathcal R_3^\downarrow (z)$, we have $P(y , \mathbf w \mid \mathrm{do}(z)) = P(y, \mathbf w)$. Therefore, $P(y \mid \mathrm{do}(z), \mathbf w) = P(y \mid \mathbf w)$ and $\mathrm{Var}(Y\mid   \mathrm{do}(Z = z), \mathbf{W}) = \mathrm{Var}(Y\mid \mathbf{W})$. Moreover, we have $\mathrm{Var}(Y\mid Z, \mathbf W) = \mathrm{Var}(Y\mid \mathbf W)- \frac{\mathrm{Cov}(Y,Z \mid \mathbf W)^2}{\mathrm{Var}(Z\mid \mathbf W)}$. Moreover, we know that  $\mathrm{Cov}(Y,Z \mid \mathbf W) \neq 0$. Hence, $\mathrm{Var}(Y \mid \mathrm{do}(Z = z), \mathbf{W}) \neq \mathrm{Var}(Y \mid Z=z, \mathbf{W})$. Thus,  $P(y \mid \mathrm{do}(z), \mathbf{w}) \neq P(y \mid z, \mathbf{w})$. Since every element of $\mathbf Z \setminus \{Z\}$ is isolated in $\G^\star$, by $\mathcal R_1^\uparrow (\mathbf z \setminus z)$ and $\mathcal R_2^\uparrow (\mathbf z \setminus z)$, we have that $P(y \mid \mathrm{do}(\mathbf z), \mathbf{w}) \neq P(y \mid \mathbf z, \mathbf{w})$. Therefore $P( \mathbf y \mid \mathrm{do}(\mathbf z), \mathbf{w}) \neq P( \mathbf y \mid \mathbf z, \mathbf{w})$.
    \end{itemize}
        
    \textbf{Conlusion for $\mathcal R_1^{\downarrow}$ and $\mathcal R_2^{\downarrow}$.} In both cases, we have shown that there exists an SCM $\mathcal M$ inducing $\mathcal G$ such that 
    $
    P_{\mathcal M}(\mathbf y \mid \mathrm{do}(\mathbf x_1), \mathbf w_1)
    \neq
    P_{\mathcal M}(\mathbf y \mid \mathrm{do}(\mathbf x_2), \mathbf w_2).
    $

    \item  If $\mathcal R$ is $\mathcal R_3^{\downarrow}$, then $\mathbf Y \;\not\!\perp\!\!\!\perp_{\tilde{\mathcal G}}\; \mathbf Z \mid \mathbf W$, where $\tilde{\mathcal G} = \mathcal G_{\overline{\mathbf Z(\mathbf W)}}$. We consider $\G^\star$ as in the first step of the previous case. This gives us $Y\in \mathbf Y$ and $Z \in \mathbf Z$ that are connected under $\mathbf W$ in $\G^\star \subseteq \tilde{\G}$ via a path $\pi$ along with a directed path $\pi_{C}$ from $C$ to $W_C \in \mathbf W$ for each collider $C$ on $\pi$. We distinguish two cases:
        \begin{itemize}
            \item If $\pi$ contains an arrow pointing away from $Z$ (i.e $\pi = \dots \leftarrow Z $). We construct $\mathcal M$ as in the second step of the previous case. Since all $\pi_{C}$ does not encounter $Z$, we know that $Z$ does not have an incoming arrow in $\G^\star$, we know that $Z$ is isolated in $\G^\star_{\underline{\mathbf Z}}$. Therefore $\mathcal R_2^\downarrow(z)$ is valid in $\G^\star$ and $\mathrm{Var}(Y\mid \mathrm{do}(Z = z), \mathbf W) = \mathrm{Var}(Y \mid Z, \mathbf W) $. Moreover, we have $\mathrm{Var}(Y\mid Z, \mathbf W) = \mathrm{Var}(Y\mid \mathbf W)- \frac{\mathrm{Cov}(Y,Z \mid \mathbf W)^2}{\mathrm{Var}(Z\mid \mathbf W)}$. By the same computation as before, we know that  $\mathrm{Cov}(Y,Z \mid \mathbf W) \neq 0$. Hence $\mathrm{Var}(Y\mid \mathrm{do}(Z = z), \mathbf W) \neq \mathrm{Var}(Y\mid \mathbf W)$ and $P(y\mid \mathrm{do}(z), \mathbf w) \neq P(y\mid \mathbf w)$. As before, we conclude that $P(\mathbf y\mid \mathrm{do}(\mathbf z), \mathbf w) \neq P(\mathbf y\mid \mathbf w)$.
            
            \item Otherwise, $\pi$ contains an incoming arrow on $Z$ (i.e., $\pi =\dots \rightarrow Z$). Necessarily, $\tilde{\G}$ contains a directed path $\pi_Z$ from Z to $W_Z \in \mathbf W$. We consider $\G^\star$ as before. We may assume that every vertex on $\pi_Z$ except $Z$ is isolated in $\G^\star$. Indeed otherwise, one may find a $\mathbf W$-active path from $Z$ to $Y$ in $\G^\star \cup \pi_Z$ that starts from an arrow pointing away from $Z$. Instead of constructing $\mathcal M$ corresponding to $\G^\star$, we construct $\mathcal M$ corresponding to $\G^\triangle \coloneqq \G^\star \cup \pi_Z$ (see Figure~\ref{fig:forthmA}). Since the unique path from $Y$ to $W_Z$ in $\G^\triangle$ encounters $Z$, we have
            \(
            P(y \mid \mathrm{do}(\mathbf z), \mathbf w) = 
            P(y \mid \mathrm{do}(\mathbf z), \mathbf w \setminus w_Z) 
            \) by $\mathcal R_1^\downarrow(w_Z)$. Since the unique element of $\mathbf Z$ non isolated in $\G^\star$ is $Z$ and since $\pi =\dots \rightarrow Z$, we know that 
            \(
            P(y \mid \mathrm{do}(\mathbf z), \mathbf w \setminus w_Z) =
            P(y \mid \mathbf w \setminus w_Z)
            \), by $\mathcal R_3^\downarrow(\mathbf z)$.
            Hence, we have
            \(
            \operatorname{Var}(Y \mid \mathrm{do}(\mathbf Z = \mathbf z), \mathbf w) = 
            \operatorname{Var}(Y \mid \mathbf W \setminus \{W_Z\}).
            \)
            However, we have 
            \(
            \operatorname{Var}(Y \mid \mathbf W ) = 
            \operatorname{Var}(Y \mid \mathbf W \setminus \{W_Z\}) - 
            \frac{\operatorname{Cov}(Y, W_Z \mid \mathbf W \setminus \{W_Z\})^2}{\operatorname{Var}(W_Z\mid \mathbf W \setminus \{W_Z\})}
            \). By the same reasoning as for $R_1^\downarrow$ and $R_2^\downarrow$, we know that $\operatorname{Cov}(Y, W_Z \mid \mathbf W \setminus \{W_Z\}) \neq 0$. Therefore we conclude that $ \operatorname{Var}(Y \mid \mathrm{do}(\mathbf Z = \mathbf z), \mathbf W) \neq  \operatorname{Var}(Y \mid \mathbf W)$ and $ P(\mathbf y\mid \mathrm{do}(\mathbf z), \mathbf w) \neq P(\mathbf y\mid \mathbf w)$.
        \end{itemize}
        
        \textbf{Conlusion for $\mathcal R_3^{\downarrow}$.} We have shown that there exists an SCM $\mathcal M$ inducing $\mathcal G$ such that 
    $
    P_{\mathcal M}(\mathbf y \mid \mathrm{do}(\mathbf x_1), \mathbf w_1)
    \neq
    P_{\mathcal M}(\mathbf y \mid \mathrm{do}(\mathbf x_2), \mathbf w_2).
    $
    \end{itemize}

Therefore, in all cases, we have shown that there exists an SCM $\mathcal M$ inducing $\mathcal G$ such that 

$$
P_{\mathcal M}(\mathbf y \mid \mathrm{do}(\mathbf x_1), \mathbf w_1)
\neq
P_{\mathcal M}(\mathbf y \mid \mathrm{do}(\mathbf x_2), \mathbf w_2).
$$

However, by assumption, there exists a valid sequence of do-calculus rules $\mathcal R_{i_1}\cdots \mathcal R_{i_n}$ mapping $Q_1$ to $Q_2$. By soundness of do-calculus, this implies that for every SCM inducing $\mathcal G$,
$$
P_{\mathcal M}(\mathbf y \mid \mathrm{do}(\mathbf x_1), \mathbf w_1)
= P_{\mathcal M}(\mathbf y \mid \mathrm{do}(\mathbf x_2), \mathbf w_2),
$$
contradicting the previous inequality. Therefore, $\mathcal R$ must be valid in $\mathcal G$.  
\end{proof}

\begin{example}
    Consider four rules $\mathcal{R}_{i_1}, \mathcal{R}_{i_2}, \mathcal{R}_{i_3}$, and $\mathcal{R}_{i_4}$, 
    which form a quadrilateral (i.e., a 4-cycle) in the diagram $\mathcal{D}[(\V, \emptyset)]$. 
    Assume that in a graph $\G$, the rules $\mathcal{R}_{i_1}, \mathcal{R}_{i_2}$, and $\mathcal{R}_{i_3}$ are valid. 
    Then, by Theorem~\ref{th:cycle_do_calculus}, the remaining rule $\mathcal{R}_{i_4}$ 
    (i.e., the corresponding shortcut rule of  $\mathcal{R}_{i_1}\, \mathcal{R}_{i_2}\, \mathcal{R}_{i_3}$) is also valid in $\G$.
    
    We illustrate this in the derivation graph in Figure~\ref{fig:exemple_annex}. Necessarily, both undirected edges
\[
P(y \mid \Do(w), x) \;\text{---}\; P(y \mid x)
\quad \text{and} \quad
P(y \mid x) \;\text{---}\; P(y \mid \Do(w), x)
\]
are absent, since the presence of either one would necessarily imply the presence of the other.
    
\begin{figure}[t]
    \centering
\begin{tikzpicture}[
    dofree/.style = {
      draw=coolorange!60!black,
      fill=coolorange!80,
      line width=0.6pt,
      inner sep=1pt,
      font=\small
    },
    withdo/.style = {
      draw=blue!60!black,
      fill=blue!25,
      line width=0.6pt,
      inner sep=1pt,
      font=\small
    },
    styleR1/.style = {thick, -, gray, line width=1.2pt},
    styleR2/.style = {thick, -, orange!60!brown, dashed, line width=1.2pt},
    styleR3/.style = {thick, -, blue!55, dotted, line width=1.5pt},
    warning/.style = {thick, -, red!70!white},
    scale=4
]
    \node[withdo] (n0) at (1.55, -.42) {$P(y \mid \Do(x))$};
    \node[withdo] (n1) at (1.85, -.12) {$P(y \mid \Do(x,z))$};
    \node[withdo] (n2) at (.6, -.3) {$P(y \mid \Do(w,x))$};
    \node[withdo] (n4) at (0.6, 0.7) {$P(y \mid \Do(w),x)$};
    \node[withdo] (n5) at (1, 0) {$P(y \mid \Do(w,x,z))$};
    \node[withdo] (n6) at (1, 1) {$P(y \mid \Do(w,z),x)$};
    \node[withdo] (n8) at (1.85, .88) {$P(y \mid \Do(z),x)$};
    \node[draw, fill = red!30, line width=0.6pt,
      inner sep=1pt,font=\footnotesize] (n9) at (1.55, .58) {$P(y \mid x)$};

\begin{pgfonlayer}{background}
    \draw[styleR3] (n0) -- (n1);
    \draw[styleR3] (n0) -- (n2);
    \draw[styleR3] (n1) -- (n5);
    \draw[styleR2] (n1) -- (n8);
    \draw[styleR2] (n2) -- (n4);
    \draw[styleR3] (n2) -- (n5);
    \draw[styleR3] (n4) -- (n6);
    \draw[styleR2] (n5) -- (n6);
    \draw[styleR3] (n6) -- (n8);
    \draw[warning] (n9) -- (n8) node[pos=0.5, red, font=\Large] {$\times$};
    \draw[warning] (n9) -- (n4) node[pos=0.5, red, font=\Large] {$\times$};
\end{pgfonlayer}
\end{tikzpicture}
\caption{Portion of the derivation graph of the Napkin graph.}
\label{fig:exemple_annex}
\end{figure}
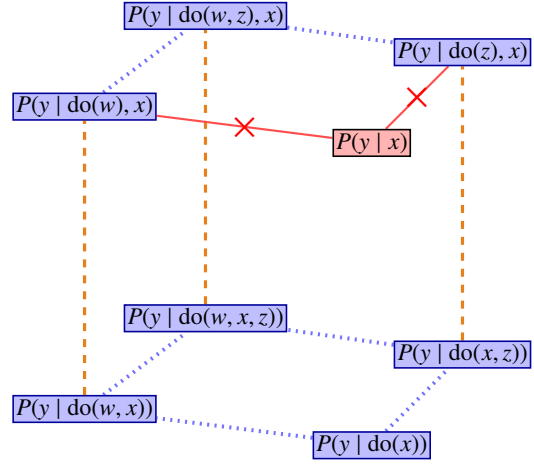
\end{example}

\newpage
\subsection{Redundancy of $\mathcal{R}_1$}
The three rules  $\mathcal{R}_1$,  $\mathcal{R}_2$,  $\mathcal{R}_3$ have a simple visualization on the derivation graph, meaning that redundancy of  $\mathcal{R}_1$ can be visualised on the triangle in Figure \ref{fig:R1_conditions}. 
\begin{figure}[ht!]
    \centering
    \begin{tikzpicture}[
        every node/.style={fill=white, inner sep=5pt, font=\small},
        every edge/.append style={draw, thick},
        >=latex,
        scale=5.5
    ]
        \node (A) at (0,0) {\(P(\mathbf{y} \mid \Do(\mathbf{x}), \mathbf{z}, \mathbf{w})\)};
        \node (B) at (.5,.7071) {\(P(\mathbf{y} \mid \Do(\mathbf{x}), \Do(\mathbf{z}), \mathbf{w})\)};
        \node (C) at (1,0) {\(P(\y \mid \Do(\x), \w)\)};

        \draw[-] (A) -- node[midway, above, sloped, align=center]
            {\(\Y \ind_{\G_{\overline{\X}}} \Z \mid \X,\W\)} (C);
        \draw[-] (A) -- node[midway, sloped, above, align=center]
            {\(\Y \ind_{\G_{\overline{\X} \underline{\Z}}} \Z \mid \X,\W\)} (B);
        \draw[-] (B) -- node[midway, sloped, above, align=center]
            {\(\Y \ind_{\G_{\overline{\X} \, \overline{\Z \setminus \Anc(\W, \G_{\overline{\X}}})}} \Z \mid \X,\W\)} (C);
    \end{tikzpicture}
    \caption{\textbf{Graphical conditions of Rule~1 applications.}  
    Each node represents an expression of the form 
    \(P(\Y \mid \Do(\X), \Do(\cdot), \W, \cdot)\), 
    and each edge corresponds to a valid application of a do-calculus rule 
    under the indicated independence condition.}
    \label{fig:R1_conditions}
\end{figure}
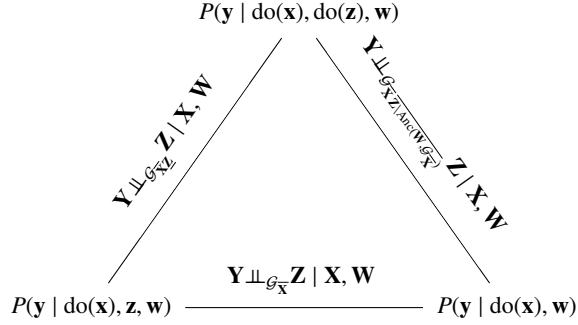

\thRun*

\begin{proof}
Graphically, the theorem corresponds to the equivalence of the following conditions:
    \begin{enumerate}
        \item \(\Y \ind_{\G_{\overline{\X}}} \Z \mid \X,\W\);
        \label{th:r1_sert_a_r1:1}
        
        \item  \(\Y \ind_{\G_{\overline{\X} \underline{\Z}}} \Z \mid \X,\W\) and \(\Y \ind_{\G_{\overline{\X} \, \overline{\Z \setminus \Anc(\W, \G_{\overline{\X}}})}} \Z \mid \X,\W\).
        \label{th:r1_sert_a_r1:2}
    \end{enumerate}
\textbf{We prove the two implications:}
\begin{itemize}
    \item[$\ref{th:r1_sert_a_r1:1} \Rightarrow \ref{th:r1_sert_a_r1:2}$:] Although this direction is already established by Lemma~4 in \citet{huang_pearls_2006}, we include a proof for completeness. 
Let us assume that  \(\Y \ind_{\G_{\overline{\X}}} \Z \mid \X,\W\). Since $\G_{\overline{\X}} \supseteq \G_{\overline{\X}\, \underline{\Z}}$ and $\G_{\overline{\X}} \supseteq \G_{\overline{\X} \, \overline{\Z \setminus \Anc(\W, \G_{\overline{\X}}})}$, we know that \(\Y \ind_{\G_{\overline{\X} \underline{\Z}}} \Z \mid \X,\W\) and \(\Y \ind_{\G_{\overline{\X} \, \overline{\Z \setminus \Anc(\W, \G_{\overline{\X}}})}} \Z \mid \X,\W\).
    
    \item[$\ref{th:r1_sert_a_r1:2} \Rightarrow \ref{th:r1_sert_a_r1:1}$:] By assumption, the derivation:
$
P(\mathbf{y} \mid \mathrm{do}(\mathbf{x}), \mathbf{z}, \mathbf{w}) \quad \mathcal R_2^{\uparrow}(\mathbf z) R_3^{\downarrow}(\mathbf z) \quad P(\mathbf{y} \mid \mathrm{do}(\mathbf{x}), \mathbf{w})
$
is valid in $\mathcal G$. Moreover, $\mathcal R_1^{\downarrow}(\mathbf z)$ syntacticaly maps $P(\mathbf{y} \mid \mathrm{do}(\mathbf{x}), \mathbf{z}, \mathbf{w})$ to $P(\mathbf{y} \mid \mathrm{do}(\mathbf{x}), \mathbf{w})$. By Theorem A, we know that $R_1^{\downarrow}(\mathbf z)$ is valid in $\mathcal G$, i.e. $\textbf{Y} \, \perp\!\!\!\perp_{\mathcal{G}_{\overline{\textbf{X}}}} \, \textbf{Z} \mid \textbf{X}, \textbf{W}$.
\end{itemize}
\end{proof}

\subsection{Internal Commutativity of $\mathcal{R}_2$}

This section demonstrates Theorem~\ref{th:r2_commute}. Definition~\ref{def:R2:mutilated_graphs} defines the mutilation graphs used in the proof of Theorem~\ref{th:r2_commute} and Figure~\ref{fig:R2_conditions} represents all the possible rules.

\begin{definition}[Mutilated Graphs for Rule 2 Transitions]
\label{def:R2:mutilated_graphs}
Let $\G$ be an ADMG over $\V$ and let $\Y,\X,\W,\Zu,\Zd$ be pairwise disjoint subsets of~$\V$.
We define the following mutilated graphs used to apply Rule~2 of the do-calculus:
\begin{align*}
    \Ga &\coloneqq \G_{\overline{\X\Zu}\,\underline{\Zd}}, &
    \Gb &\coloneqq \G_{\overline{\X}\,\underline{\Zu}}, &
    \Gc &\coloneqq \G_{\overline{\X}\,\underline{\Zd}}, \\
    \Gd &\coloneqq \G_{\overline{\X\Zd}\,\underline{\Zu}}, &
    \Ge &\coloneqq \G_{\overline{\X}\,\underline{\Zu\Zd}}.
\end{align*}
\end{definition}

\begin{figure}[ht!]
    \centering
    \begin{tikzpicture}[
        every node/.style={fill=white, inner sep=5pt, font=\small},
        every edge/.append style={draw, thick},
        >=latex,
        scale=3.2
    ]
        \node (A) at (0,2) {\(P(\y \mid \Do(\x), \Do(\zu, \zd), \w)\)};
        \node (B) at (2,2) {\(P(\y \mid \Do(\x), \Do(\zu), \w, \zd)\)};
        \node (C) at (2,0) {\(P(\y \mid \Do(\x), \w, \zu, \zd)\)};
        \node (D) at (0,0) {\(P(\y \mid \Do(\x), \Do(\zd), \w, \zu)\)};

        \draw[-] (A) -- node[midway, above, align=center]
            {\(\Y \ind_{\Ga} \Zd \mid \X,\Zu,\W\)} (B);
        \draw[-] (B) -- node[midway, right, align=center]
            {\(\Y \ind_{\Gb} \Zu \mid \X,\W,\Zd\)} (C);
        \draw[-] (C) -- node[midway, below, align=center]
            {\(\Y \ind_{\Gc} \Zd \mid \X,\W,\Zu\)} (D);
        \draw[-] (D) -- node[midway, left, align=center]
            {\(\Y \ind_{\Gd} \Zu \mid \X,\Zd,\W\)} (A);
        \draw[-] (A) -- node[midway, sloped, above, align=center]
            {\(\Y \ind_{\Ge} \Zu,\Zd \mid \X,\W\)} (C);
    \end{tikzpicture}
    \caption{\textbf{Graphical conditions of Rule~2 applications.}  
    Each node represents an expression of the form 
    \(P(\Y \mid \Do(\X), \Do(\cdot), \W, \cdot)\), 
    and each edge corresponds to a valid application of Rule~2 
    of the do-calculus under the indicated independence condition.  
    The mutilated graphs used in each condition are denoted by 
    \(\Ga,\Gb,\Gc,\Gd,\Ge\), defined in Definition~\ref{def:R2:mutilated_graphs}.}
    \label{fig:R2_conditions}
\end{figure}
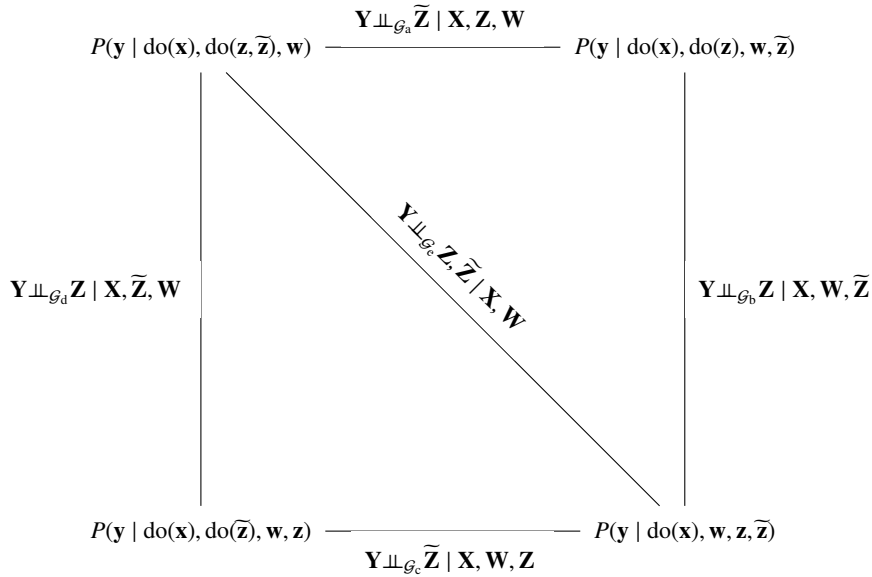

\thRtwocommute*

\begin{proof}
Graphically, using Definition~\ref{def:R2:mutilated_graphs} and Figure~\ref{fig:R2_conditions}, the theorem corresponds to the following properties:
    \begin{enumerate}
        \item \(\Y \ind_{\Ga} \Zd \mid \X,\Zu,\W\) and \(\Y \ind_{\Gb} \Zu \mid \X,\W,\Zd\) $\Leftrightarrow$ \(\Y \ind_{\Ge} \Zu,\Zd \mid \X,\W\) $\Leftrightarrow$ \(\Y \ind_{\Gc} \Zd \mid \X,\W,\Zu\) and \(\Y \ind_{\Gd} \Zu \mid \X,\Zd,\W\)
        \item \(\Y \ind_{\Ga} \Zd \mid \X,\Zu,\W\) and  \(\Y \ind_{\Gd} \Zu \mid \X,\Zd,\W\) $\Leftrightarrow$  \(\Y \ind_{\Gb} \Zu \mid \X,\W,\Zd\) and \(\Y \ind_{\Gc} \Zd \mid \X,\W,\Zu\)
    \end{enumerate}
    
    The first property is a direct consequence of Lemma~\ref{th:r2_commute:1} and the second is a direct consequence of Lemma~\ref{th:r2_commute:2} that we state and prove just below.
\end{proof}

\begin{restatable}[]{lemma}{thRtwocommute:1}
\label{th:r2_commute:1}
    Let $\G=(\V_\G,\E_\G)$ be an ADMG, and let $\Y,\X,\W,\Zu,\Zd$ be pairwise disjoint subsets of~$\V_\G$, with $\Y \neq \emptyset$. The following proposition are equivalent:
    \begin{enumerate}
        \item \(\Y \ind_{\Ga} \Zd \mid \X,\Zu,\W\) and \(\Y \ind_{\Gb} \Zu \mid \X,\W,\Zd\) \label{th:r2_commute:1:1}
        \item  \(\Y \ind_{\Ge} \Zu,\Zd \mid \X,\W\) \label{th:r2_commute:1:2}
        \item \(\Y \ind_{\Gc} \Zd \mid \X,\W,\Zu\) and \(\Y \ind_{\Gd} \Zu \mid \X,\Zd,\W\) \label{th:r2_commute:1:3}
    \end{enumerate}
\end{restatable}

\begin{proof}
By symmetry between $\Zu$ and $\Zd$, we only need to prove $\ref{th:r2_commute:1:1} \Leftrightarrow \ref{th:r2_commute:1:2}$. Let us prove the two implications:
\begin{itemize}
    \item[$\ref{th:r2_commute:1:1} \Rightarrow \ref{th:r2_commute:1:2}$:] We procede by contraposition, let us assume that \(\Y \notind_{\Ge} \Zu,\Zd \mid \X,\W\). There exists a path $\pi$ from $\Y$ to $\Zu \cup \Zd$ in $\Ge$ that is $\X,\W$-active. Without loss of generality we assume that $\pi$ encounters $\Zu \cup \Zd$ only once at its last vertex. We distinguish two cases:
    \begin{itemize}
        \item If $\pi$ encounters $\Zu$. Since $\Ge \subseteq \Gb$, we know that $\pi$ exists in $\Gb$. Moreover, all its non-collider are not in $\X,\W,\Zd$ since $\pi$ is $\X,\W$ active in $\Ge$ and since $\pi$ does not encounter $\Zd$. Moreover, since $\pi$ is $\X,\W$-active in $\Ge \subseteq \Gb$ every collider of $\pi$ belongs to $\Anc(\W, \Ge) \subseteq \Anc(\W, \Gb) \subseteq \Anc(\X \cup \W \cup \Zd, \Gb)$. Therefore, $\pi$ is $\X,\W,\Zd$-active in $\Gb$. Therefore  \(\Y \notind_{\Gb} \Zu \mid \X,\W,\Zd\)
        
        \item If $\pi$ encounters $\Zd$. Necessarily, $\pi$ exists in $\Ga$ because $\pi$ does not encounter $\Zu$. Moreover, all its non-collider are not in $\X,\W,\Zu$ since $\pi$ is $\X,\W$-active in $\Ge$ and since $\pi$ does not encounter $\Zu$. Let $C$ be a collider of $\pi$. Since $\pi$ is $\X,\W$-active in $\Ge$, then $C \in \Anc(\W, \Ge)$. Thus there exists a directed path  $\pi'$ from $C$ to $\W$ in $\Ge$. Since $\Ge$ does not contain any outgoing arrow from any vertex in $\Zu$, we know that $\pi'$ does not encounter $\Zu$, therefore, $\pi'$ exists in $\Ga$ and $C \in \Anc(\W, \Ga).$ Therefore, \(\Y \notind_{\Ga} \Zd \mid \X,\Zu,\W\).
    \end{itemize}
    Therefore, \(\Y \notind_{\Gb} \Zu \mid \X,\W,\Zd\) or \(\Y \notind_{\Ga} \Zd \mid \X,\Zu,\W\).
    \item[$\ref{th:r2_commute:1:1} \Leftarrow \ref{th:r2_commute:1:2}$:] We proceed by contraposition, let us assume that \(\Y \notind_{\Gb} \Zu \mid \X,\W,\Zd\) or \(\Y \notind_{\Ga} \Zd \mid \X,\Zu,\W\). We distinguish two cases:
    \begin{itemize}
        \item If \(\Y \notind_{\Gb} \Zu \mid \X,\W,\Zd\), then there exists a path $\pi$ in $\Gb$ that connects $\Y$ and $\Zu$ under $\X,\W,\Zd$. Since $\pi$ is $\X,\W,\Zd$-active in $\Gb$, we know that all non collider of $\pi$ are not in $\X,\W,\Zd$. Thus, any encounter of $\pi$ and $\Zd$ is a collider. If such an encounter exists, we restrict $\pi$ to its subpath from $\Y$ to $\Zd$, i.e., we set $\pi \gets \pi_{[\Y, \Zd]}$. Thus $\pi$ exists in $\Ge = {\Gb}_{\underline{\Zd}}$ and all its non collider are not in $\X,\W$. For any collider $C$ on $\pi$, we know that $C \in \Anc(\W \cup \Zd, \Gb)$. We distinguish two cases:
        \begin{itemize}
            \item If all colliders $C$ belong to $\Anc(\W, \Ge)$, then $\pi$ is $\X,\W$-active in $\Ge$. Therefore \(\Y \notind_{\Ge} \Zu,\Zd \mid \X,\W\).
            
            \item Otherwise, let $C$ be the first collider of $\pi$ that is not an ancestor of $\W$ in $\Ge$. Necessarily, since $\Ge = {\Gb}_{\underline{\Zd}}$, we know that $C \in \Anc(\Zd, \Ge \setminus \W)$. Thus, $\Ge \setminus \W$ contains $\pi'$ a directed path from $C$ to $\Zd$. We consider the path $\pi_{[\Y, C]} \cup \pi'$. This path exists in $\Ge$ and is $\X,\W$-active by construction. Therefore \(\Y \notind_{\Ge} \Zu,\Zd \mid \X,\W\).
        \end{itemize}
        In all cases, we conclude that \(\Y \notind_{\Ge} \Zu,\Zd \mid \X,\W\).
        
        \item If \(\Y \notind_{\Ga} \Zd \mid \X,\Zu,\W\), then there exists a path $\pi$ in $\Ga$ that connects $\Y$ and $\Zd$ under $\X,\Zu,\W$. Since $\Ga$ have the mutilation $\overline{\Zu}$, we know that any potential encounter of $\pi$ and $\Zu$ is a fork. However, since $\pi$ is $\X,\Zu,\W$-active in $\Ga$, we can conclude that $\pi \cap \Zu = \emptyset$. Therefore, $\pi$ exists in $\Ge$ and all its non-collider are not in $\X,\W$. Let $C$ be a collider of $\pi$. We know that $\Ga$ contains a directed path $\pi'$ from $C$ to $W$. Moreover, due to the mutilation  $\overline{\Zu}$, $\pi'$ does not encounter $\Zu$. Therefore $\pi'$ exists in $\Ge$. Thus $C \in \Anc(W, \Ge)$. Therefore, \(\Y \notind_{\Ge} \Zu,\Zd \mid \X,\W\).
    \end{itemize}
    Therefore, \(\Y \notind_{\Ge} \Zu,\Zd \mid \X,\W\).
\end{itemize}
\end{proof}

\begin{restatable}[]{lemma}{thRtwocommute:2}
\label{th:r2_commute:2}
    Let $\G=(\V_\G,\E_\G)$ be an ADMG, and let $\Y,\X,\W,\Zu,\Zd$ be pairwise disjoint subsets of~$\V_\G$, with $\Y \neq \emptyset$. The following proposition are equivalent:
    \begin{enumerate}
        \item  \(\Y \ind_{\Gb} \Zu \mid \X,\W,\Zd\) and \(\Y \ind_{\Gc} \Zd \mid \X,\W,\Zu\) \label{th:r2_commute:2:1}
        \item  \(\Y \ind_{\Ga} \Zd \mid \X,\Zu,\W\) and  \(\Y \ind_{\Gd} \Zu \mid \X,\Zd,\W\) \label{th:r2_commute:2:2}
    \end{enumerate}
\end{restatable}

\begin{proof}
Let us prove the two implications:
\begin{itemize}
    \item[$\ref{th:r2_commute:2:1} \Rightarrow \ref{th:r2_commute:2:2}$:] Let us assume that  \(\Y \ind_{\Gb} \Zu \mid \X,\W,\Zd\) and \(\Y \ind_{\Gc} \Zd \mid \X,\W,\Zu\). Since $\Ga \subseteq \Gc$ and $\Gd \subseteq \Gb$, we can conclude that  \(\Y \ind_{\Ga} \Zd \mid \X,\Zu,\W\) and  \(\Y \ind_{\Gd} \Zu \mid \X,\Zd,\W\).
    \item[$\ref{th:r2_commute:2:1} \Leftarrow \ref{th:r2_commute:2:2}$:] We proceed by contraposition. Let us assume that \(\Y \notind_{\Gb} \Zu \mid \X,\W,\Zd\) and \(\Y \notind_{\Gc} \Zd \mid \X,\W,\Zu\). We distinguish two cases:
    \begin{itemize}
        \item If \(\Y \notind_{\Gb} \Zu \mid \X,\W,\Zd\), then there exists a path $\pi$ from $\Y$ to $\Zu$ in $\Gb$ that is $\X,\W,\Zd$-active. Without loss of generality, we assume that $\pi$ encounters $\Zu$ only at its last vertex. We distinguish two cases:
        \begin{itemize}
            \item If $\pi \cap \Zd = \emptyset$, then $\pi$ exists in $\Gd = {\Gb}_{\overline{\Zd}}$. Since, $\pi$ is $\X,\W,\Zd$-active in $\Gb$, we know that all its non collider are not in $\X,\W,\Zd$. For any collider $C$, we know that $C \in \Anc(\W \cup \Zd, \Gb)$. We distinguish two cases:
            \begin{itemize}
                \item If all collider $C \in \Anc(\W, \Gd)$, then $\pi$ is $\X,\Zd,\W$-active in $\Gd$. Therefore \(\Y \notind_{\Gd} \Zu \mid \X,\Zd,\W\).
                
                \item Otherwise, let $C$ be the first collider of $\pi$ that is an ancestor of $\Zd$ in $\Gb$. Let $\pi'$ be a directed path from $C$ to $\Zd$ in $\Gb$. We consider the path $\pi_{[\Y,C]} \cup \pi'$.
                $\pi \cap \Zd = \emptyset$, thus $\pi_{[\Y,C]}$ exists in $\Ga$.
                Moreover, $\pi'$ exists in $\Gb$ and $\C \notin \Zu$, thus $\pi' \cap \Zu = \emptyset$. Moreover, by construction, the mutilation $\underline{\Zd}$ does not affect $\pi'$. Therefore $\pi'$ exists in $\Ga$. Thus, $\pi_{[\Y,C]} \cup \pi'$ exists in $\Ga$. 
                All its non collider are not in $\X,\Zu,\W$. Let $V$ be a collider of $\pi_{[\Y,C]} \cup \pi'$, by construction it is a collider before $C$ in $\pi$. Therefore $V$ is an ancestor of $\W$ in $\Gb$. Thus there exists a directed path $\pi_V$ from $V$ to $\W$ in $\Gb$. Since $C$ is the first collider that is an ancestor of $\Zd$ in $\Gb$, we know that $\pi_V \cap \Zd = \emptyset$. Similarly, since $V \notin \Zu$ and $\pi_V$ exists in $\Gb$, we know that $\pi_V \cap \Zu = \emptyset$. Therefore $\pi_V$ exists in $\Ga$. Therefore $V \in \Anc(\W, \Ga)$. Therefore, \(\Y \notind_{\Ga} \Zd \mid \X,\Zu,\W\)
            \end{itemize}
            Therefore, \(\Y \notind_{\Ga} \Zd \mid \X,\Zu,\W\) or \(\Y \notind_{\Gd} \Zu \mid \X,\Zd,\W\).
            
            \item Otherwise, $\pi \cap \Zd \neq \emptyset$. Since, $\pi$ is $\X,\W,\Zd$-active in $\Gb$, we know that every encounter of $\pi$ and $\Zd$ is a fork. Moreover, since $\pi$ encounter $\Zu$ only at its last vertex, we know that $\pi_{[\Y, \Zd]} \cap \Zu = \emptyset$. Therefore  $\pi_{[\Y, \Zd]}$ exists in $\Ga$ and all its non-collider are not in $\X,\Zu,\W$. We distinguish two cases:
            \begin{itemize}
                \item If all colliders of $\pi_{[\Y, \Zd]}$ are ancestors of $\W$ in $\Ga$, then \(\Y \notind_{\Ga} \Zd \mid \X,\Zu,\W\).
                
                \item Otherwise, let $C$ be the first collider of $\pi_{[\Y, \Zd]}$ that is not an ancestor of $\W$ in $\Ga$. Necessarily, $\Gb$ contains a directed path $\pi'$ from $C$ to $\Zd$. Without loss of generality, we assume that $\pi'$ encounters $\Zd$ only at its last vertex.  Since $\Gb$ have the mutilation $\underline{\Zu}$ and $C \notin \Zu$, we know that $\pi'$ does not encounter $\Zu$. Therefore, $\pi'$ exists in $\Ga$. We consider the path ${\pi_{[\Y, \Zd]}}_{[\Y, C]} \cup \pi'$. It exists in $\Ga$, and by construction it is $\X, \Zu, \W$-active. Therefore, \(\Y \notind_{\Ga} \Zd \mid \X,\Zu,\W\).
            \end{itemize}
            Therefore, in all cases, \(\Y \notind_{\Ga} \Zd \mid \X,\Zu,\W\).
        \end{itemize}
        Therefore, in all cases, \(\Y \notind_{\Ga} \Zd \mid \X,\Zu,\W\) or \(\Y \notind_{\Gd} \Zu \mid \X,\Zd,\W\).
        
        \item If \(\Y \notind_{\Gc} \Zd \mid \X,\W,\Zu\), then we do not have \(\Y \ind_{\Gc} \Zd \mid \X,\W,\Zu\) and \(\Y \ind_{\Gd} \Zu \mid \X,\Zd,\W\). Therefore, by Lemma~\ref{th:r2_commute:1}, we don't have \(\Y \ind_{\Ga} \Zd \mid \X,\Zu,\W\) and \(\Y \ind_{\Gb} \Zu \mid \X,\W,\Zd\). Thus we distinguish two cases:
        \begin{itemize}
            \item  \(\Y \notind_{\Ga} \Zd \mid \X,\Zu,\W\) and therefore, \(\Y \notind_{\Ga} \Zd \mid \X,\Zu,\W\) or \(\Y \notind_{\Gd} \Zu \mid \X,\Zd,\W\)
            
            \item  \(\Y \notind_{\Gb} \Zu \mid \X,\W,\Zd\) and therefore \(\Y \notind_{\Ga} \Zd \mid \X,\Zu,\W\) or \(\Y \notind_{\Gd} \Zu \mid \X,\Zd,\W\) by the previous reasoning.
        \end{itemize}
        Therefore  \(\Y \notind_{\Ga} \Zd \mid \X,\Zu,\W\) or \(\Y \notind_{\Gd} \Zu \mid \X,\Zd,\W\).
    \end{itemize}
    
    Therefore, in all cases, \(\Y \notind_{\Ga} \Zd \mid \X,\Zu,\W\) or \(\Y \notind_{\Gd} \Zu \mid \X,\Zd,\W\).
    
\end{itemize}
\end{proof}

\subsection{Internal Commutativity of $\mathcal{R}_3$}

This section demonstrates Theorem~\ref{th:r3_commute}. Definition~\ref{def:R3:mutilated_graphs} defines the mutilation graphs used in the proof of Theorem~\ref{th:r3_commute} and Figure~\ref{fig:R3_conditions} represents all the possible rules.

\begin{definition}[Mutilated Graphs for Rule 3 Transitions]
\label{def:R3:mutilated_graphs}
Let $\G=(\V_\G,\E_\G)$ be an ADMG, and let $\Y,\X,\W,\Zu,\Zd$ be pairwise disjoint subsets of~$\V_\G$, with $\Y \neq \emptyset$.
We define the following mutilated graphs used to apply Rule~3 of the do-calculus:
\begin{align*}
    \Ga &\coloneqq \G_{\overline{\X\Zu}\,\overline{\Zd \setminus \Anc(\W, \G_{\overline{\X \Zu}})}}, &
    \Gb &\coloneqq \G_{\overline{\X}\,\overline{\Zu \setminus \Anc(\W, \G_{\overline{\X }})}}, &
    \Gc &\coloneqq \G_{\overline{\X}\,\overline{\Zd \setminus \Anc(\W, \G_{\overline{\X }})}}, \\
    \Gd &\coloneqq \G_{\overline{\X\Zd}\,\overline{\Zu \setminus \Anc(\W, \G_{\overline{\X \Zd}})}}, &
    \Ge &\coloneqq \G_{\overline{\X}\,\overline{(\Zu\Zd) \setminus \Anc(\W,\G_{\overline{\X}})}}.
\end{align*}
\end{definition}

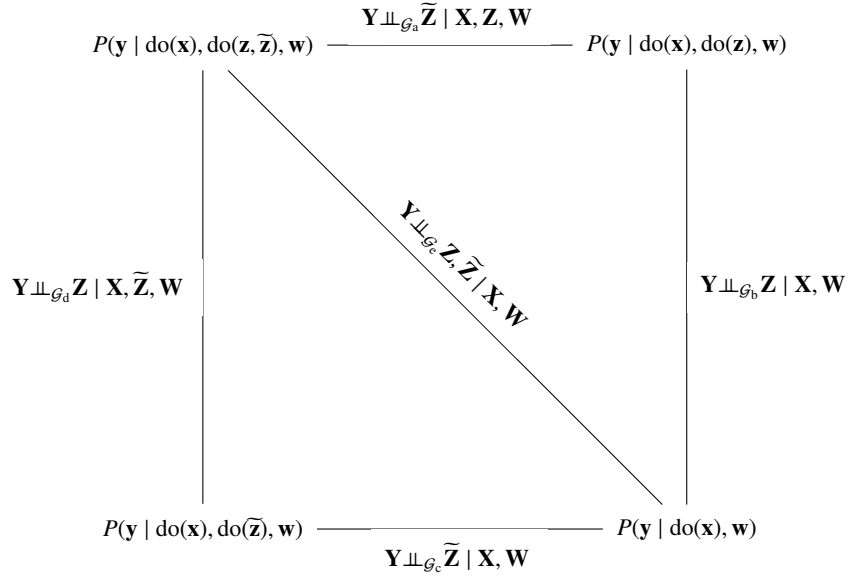
\begin{figure}[ht!]
    \centering
    \begin{tikzpicture}[
        every node/.style={fill=white, inner sep=5pt, font=\small},
        every edge/.append style={draw, thick},
        >=latex,
        scale=3.2
    ]
        \node (A) at (0,2) {\(P(\y \mid \Do(\x), \Do(\zu, \zd), \w)\)};
        \node (B) at (2,2) {\(P(\y \mid \Do(\x), \Do(\zu), \w)\)};
        \node (C) at (2,0) {\(P(\y \mid \Do(\x), \w)\)};
        \node (D) at (0,0) {\(P(\y \mid \Do(\x), \Do(\zd), \w)\)};

        \draw[-] (A) -- node[midway, above, align=center]
            {\(\Y \ind_{\Ga} \Zd \mid \X,\Zu,\W\)} (B);
        \draw[-] (B) -- node[midway, right, align=center]
            {\(\Y \ind_{\Gb} \Zu \mid \X,\W\)} (C);
        \draw[-] (C) -- node[midway, below, align=center]
            {\(\Y \ind_{\Gc} \Zd \mid \X,\W\)} (D);
        \draw[-] (D) -- node[midway, left, align=center]
            {\(\Y \ind_{\Gd} \Zu \mid \X, \Zd,\W\)} (A);
        \draw[-] (A) -- node[midway, sloped, above, align=center]
            {\(\Y \ind_{\Ge} \Zu,\Zd \mid \X,\W\)} (C);
    \end{tikzpicture}
    \caption{\textbf{Graphical conditions of Rule~3 applications.}  
    Each node represents an expression of the form 
    \(P(\Y \mid \Do(\X), \Do(\cdot), \W, \cdot)\), 
    and each edge corresponds to a valid application of Rule~3 
    of the do-calculus under the indicated independence condition.  
    The mutilated graphs used in each condition are denoted by 
    \(\Ga,\Gb,\Gc,\Gd,\Ge\), defined in Definition~\ref{def:R3:mutilated_graphs}.}
    \label{fig:R3_conditions}
\end{figure}

\thRtrois*

\begin{proof}
    Graphically, using Definition~\ref{def:R3:mutilated_graphs} and Figure~\ref{fig:R3_conditions}, the theorem corresponds to the following properties:
    \begin{enumerate}
        \item \(\Y \ind_{\Ga} \Zd \mid \X,\Zu,\W\) and \(\Y \ind_{\Gb} \Zu \mid \X,\W\) $\Leftrightarrow$ \(\Y \ind_{\Ge} \Zu,\Zd \mid \X,\W\) $\Leftrightarrow$ \(\Y \ind_{\Gc} \Zd \mid \X,\W\) and \(\Y \ind_{\Gd} \Zu \mid \X,\Zd,\W\)
        \item \(\Y \ind_{\Gb} \Zu \mid \X,\W\) and \(\Y \ind_{\Gc} \Zd \mid \X,\W\) $\Rightarrow$ \(\Y \ind_{\Ga} \Zd \mid \X,\Zu,\W\) and  \(\Y \ind_{\Gd} \Zu \mid \X,\Zd,\W\) 
    \end{enumerate}

The first point is a direct consequence of Lemma~\ref{th:r3_commute:1} and the second is a direct consequence of Lemma~\ref{lemma:R3:implications}.

\end{proof}

\begin{restatable}{lemma}{thRthreecommute:1}
\label{th:r3_commute:1}
    Let $\G$ be an ADMG over $\V$ and let $\Y,\X,\W,\Zu,\Zd$ be pairwise disjoint subsets of~$\V$. The following proposition are equivalent:
    \begin{enumerate}
        \item \(\Y \ind_{\Ga} \Zd \mid \X,\Zu,\W\) and \(\Y \ind_{\Gb} \Zu \mid \X,\W\) \label{th:r3_commute:1:1}
        \item  \(\Y \ind_{\Ge} \Zu,\Zd \mid \X,\W\) \label{th:r3_commute:1:2}
        \item \(\Y \ind_{\Gc} \Zd \mid \X,\W\) and \(\Y \ind_{\Gd} \Zu \mid \X,\Zd,\W\) \label{th:r3_commute:1:3}
    \end{enumerate}
\end{restatable}

\begin{proof}
By symmetry between $\Zu$ and $\Zd$, we only need to prove $\ref{th:r3_commute:1:1} \Leftrightarrow \ref{th:r3_commute:1:2}$. Let us prove the two implications:
\begin{itemize}
    \item[$\ref{th:r3_commute:1:1} \Rightarrow \ref{th:r3_commute:1:2}$:] It is a direct consequence of Theorem~\ref{th:cycle_do_calculus}.
    \item[$\ref{th:r3_commute:1:1} \Leftarrow \ref{th:r3_commute:1:2}$:] We will prove the contrapositive. Let us assume that  \(\Y \notind_{\Ga} \Zd \mid \X,\Zu,\W\) or \(\Y \notind_{\Gb} \Zu \mid \X,\W\). Let us consider the two cases:
    \begin{itemize}
        \item If \(\Y \notind_{\Ga} \Zd \mid \X,\Zu,\W\), then $\Ga$ contains a path $\pi$ from $\Y$ to $\Zd$ that is $\X,\Zu,\W$-active. Since $\Zu \cup (\Zd \setminus \Anc(\W, \G_{\overline{\X \Zu}})) \supseteq (\Zu \cup \Zd) \setminus \Anc(\W,\G_{\overline{\X}})$, we know that $\Ga \subseteq \Ge$. As a result, $\pi$ exists in $\Ge$. Let $V$ be a non collider of $\pi$. Since, $\pi$ is $\X,\Zu,\W$-active in $\Ga$, we know that $W \notin \X, \W$. Let $C$ be a collider of $\pi$. Since $\pi$ is $\X,\Zu,\W$-active in $\Ga$, we know that $C$ is an ancestor of $\X,\Zu,\W$ in $\Ga$. Since $\Ga$ contains the mutilation $\overline{\Zu}$, we know that $C$ is an ancestor of $\X, \W$ in $\Ga \subseteq \Ge$. Therefore, $\pi$ is active in $\Ge$, therefore \(\Y \notind_{\Ge} \Zu,\Zd \mid \X,\W\).
        
        \item If \(\Y \notind_{\Gb} \Zu \mid \X,\W\), then $\Gb$ contains a path $\pi$ from $\Y$ to $\Zu$ that is $\X, \W$-active. We distinguish two cases:
        \begin{itemize}
            \item If $\pi$ exists in $\Ge$. Then, every non collider of $\pi$ is not in $\X,\W$ as $\pi$ is $\X,\W$-active in $\Gb$. Let $C$ be a collider of $\pi$. Since $\Gb$ contains the mutilation $\overline{\X}$, we know that $C$ is an ancestor of $\W$ in $\Gb \subseteq \G_{\overline{\X}}$. Thus $C$ is an ancestor of $\W$ in $\G_{\overline{\X}}$. Thus $\C$ is an ancestor of $\W$ in $ \G_{\overline{\X}\,\overline{(\Zu\Zd) \setminus \Anc(\W,\G_{\overline{\X}})}} = \Ge $. Therefore, \(\Y \notind_{\Ge} \Zu,\Zd \mid \X,\W\).
            
            \item Otherwise, necessarily $\pi \cap \Zd \neq \emptyset$. Let ${\widetilde Z}$ be the first encounter of $\pi$ and $\Zd$ that prevents $\pi$ from existing in $\Ge$. By definition, ${\widetilde Z}$ belongs to $\Zd \setminus \Anc(\W, \G_{\overline{\X}})$ and ${\widetilde Z}$ have an incoming arrow in $\pi$. Thus, ${\widetilde Z}$ is not a fork and not an ancestor of $\W$ in $\G_{\overline{\X}}$. Thus, ${\widetilde Z}$ is not a collider because, otherwise, $\pi$ would be blocked in $\Gb$. We distinguish two potential cases:
            \begin{itemize}
                \item If $\pi = \cdots \rightarrow {\widetilde Z} \rightarrow \cdots$. Since all colliders on $\pi$ are ancestors of $\W$ in $\G_{\overline{\X}}$, the subpath $\pi_{[{\widetilde Z}, \Zu]}$ cannot contain a collider; otherwise, ${\widetilde Z}$ would be an ancestor of $\W$ in $\G_{\overline{\X}}$. Therefore, $\pi_{[{\widetilde Z}, \Zu]}$ is a directed path whose last vertex is not an ancestor of $\W$ in $\G_{\overline{\X}}$, contradicting the existence of $\pi$ in $\Gb$. Hence, this case is impossible.
                \item If $\pi = \cdots \leftarrow {\widetilde Z} \leftarrow \cdots$. By definition, the subpath $\pi_{[\Y, {\widetilde Z}]}$ exists in $\Ge$. Since all colliders on $\pi$ are ancestors of $\W$ in $\G_{\overline{\X}}$, the subpath $\pi_{[\Y, {\widetilde Z}]}$ cannot contain a collider; otherwise, ${\widetilde Z}$ would be an ancestor of $\W$ in $\G_{\overline{\X}}$. Moreover, every vertex on $\pi_{[\Y, {\widetilde Z}]}$ does not belong to $\X,\W$. Therefore, $\pi_{[\Y, {\widetilde Z}]}$ is $\X,\W$-active in $\Ge$. Therefore, \(\Y \notind_{\Ge} \Zu,\Zd \mid \X,\W\).
            \end{itemize}
        \end{itemize}
    \end{itemize}
    In all cases, we have shown that \(\Y \notind_{\Ge} \Zu,\Zd \mid \X,\W\).
\end{itemize}
\end{proof}

\begin{restatable}{lemma}{thRtimplications}
\label{lemma:R3:implications}
    Let $\G$ be an ADMG over $\V$ and let $\Y,\X,\W,\Zu,\Zd$ be pairwise disjoint subsets of~$\V$. We have the following propositions:
    \begin{enumerate}
        \item \(\Y \ind_{\Gc} \Zd \mid \X,\W \Rightarrow \Y \ind_{\Ga} \Zd \mid \X,\Zu,\W\)  \label{lemma:R3:implications:1}
        \item \(\Y \ind_{\Gb} \Zu \mid \X,\W \Rightarrow \Y \ind_{\Gd} \Zu \mid \X,\Zd,\W\) \label{lemma:R3:implications:2}
    \end{enumerate}
\end{restatable}

\begin{proof}
    By symmetry, we only prove the first implication. We prove the contrapositive. Let us assume that $\Y \notind_{\Ga} \Zd \mid \X,\Zu,\W$, then $\Ga$ contains a path $\pi$ that connects $\Y$ and $\Zd$ and $\pi$ $\X,\Zu,\W$-active. Since $ \Ga \subseteq \Gc$, we know that $\pi$ exists in $\Gc$. Moreover, since $\pi$ is $\X,\Zu,\W$-active in $\Ga$, we know that all its non collider are not in $\X, \Zu, \W$. Thus all its non collider are not in $\X, \W$. Let $C$ be a a collider of $\pi$. Since $\Ga$ contains the mutilation $\overline{\X\Zu}$, we know that $C$ is an ancestor of $\W$ in $\G_{\overline{\X\Zu}} \subseteq \G_{\overline{\X}}$. Therefore, $C$ is an ancestor of $\W$ in $\Gc$. Therefore $\pi$ is $\X,\W$-active in $\Gc$. Hence, $\Y \notind_{\Gc} \Zd \mid \X,\W$
\end{proof}

\subsection{Weak Commutativity Between $\mathcal{R}_2$ and $\mathcal{R}_3$}

This section demonstrates Theorem~\ref{th:r2R3_commute}. Definition~\ref{def:R2R3:mutilated_graphs} defines the mutilation graphs used in the proof of Theorem~\ref{th:r2R3_commute} and Figure~\ref{fig:R2R3_conditions} represents all the possible rules.

\begin{definition}[Mutilated Graphs for Rule 2 and Rule 3 Transitions]
\label{def:R2R3:mutilated_graphs}
Let $\G$ be an ADMG over $\V$ and let $\Y,\X,\W,\Zu,\Zd$ be pairwise disjoint subsets of~$\V$.
We define the following mutilated graphs used to apply Rule~2 and Rule~3 of the do-calculus:
\begin{align*}
    \Ga &\coloneqq \G_{\overline{\X\Zu}\,\overline{\Zd \setminus \Anc(\W, \G_{\overline{\X \Zu}})}}, &
    \Gb &\coloneqq \G_{\overline{\X}\,\underline{\Zu}}, &
    \Gc &\coloneqq \G_{\overline{\X}\,\overline{\Zd \setminus \Anc(\W\cup \Zu, \G_{\overline{\X }})}}, 
    \Gd &\coloneqq \G_{\overline{\X\Zd}\,\underline{\Zu}}.
\end{align*}
\end{definition}

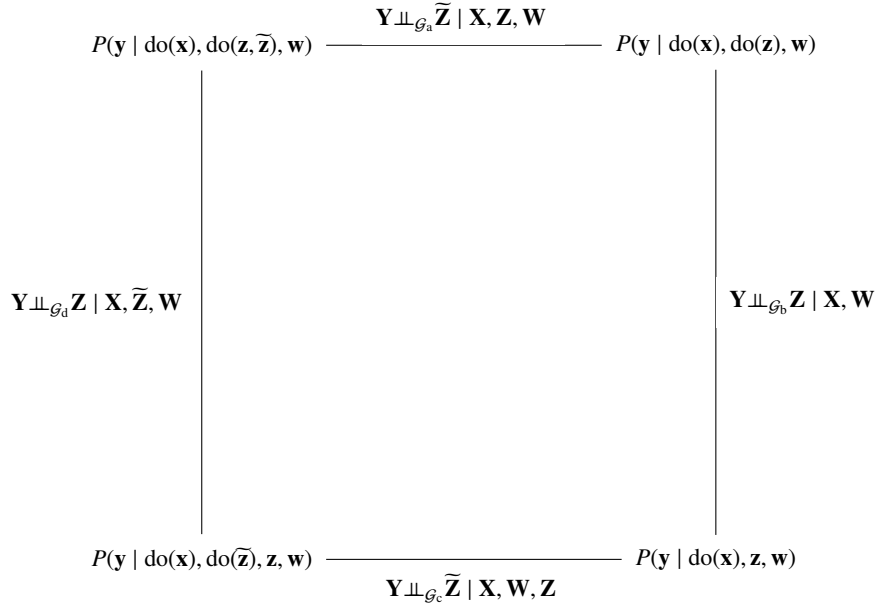
\begin{figure}[ht!]
    \centering
    \begin{tikzpicture}[
        every node/.style={fill=white, inner sep=5pt, font=\small},
        every edge/.append style={draw, thick},
        >=latex,
        scale=3.4
    ]
        \node (A) at (0,2) {\(P(\y \mid \Do(\x), \Do(\zu, \zd), \w)\)};
        \node (B) at (2,2) {\(P(\y \mid \Do(\x), \Do(\zu), \w)\)};
        \node (C) at (2,0) {\(P(\y \mid \Do(\x), \zu, \w)\)};
        \node (D) at (0,0) {\(P(\y \mid \Do(\x), \Do(\zd),\zu, \w)\)};

        \draw[-] (A) -- node[midway, above, align=center]
            {\(\Y \ind_{\Ga} \Zd \mid \X,\Zu,\W\)} (B);
        \draw[-] (B) -- node[midway, right, align=center]
            {\(\Y \ind_{\Gb} \Zu \mid \X,\W\)} (C);
        \draw[-] (C) -- node[midway, below, align=center]
            {\(\Y \ind_{\Gc} \Zd \mid \X,\W, \Zu\)} (D);
        \draw[-] (D) -- node[midway, left, align=center]
            {\(\Y \ind_{\Gd} \Zu \mid \X, \Zd,\W\)} (A);
    \end{tikzpicture}
    \caption{\textbf{Graphical conditions of Rule~2 and Rule~3 applications.}  
    Each node represents an expression of the form 
    \(P(\Y \mid \Do(\X), \Do(\cdot), \W, \cdot)\), 
    each horizontal edge corresponds to a valid application of Rule~3, and each vertical hedge corresponds to a valid application of Rule~2 
    of the do-calculus under the indicated independence condition.  
    The mutilated graphs used in each condition are denoted by 
    \(\Ga,\Gb,\Gc,\Gd,\Ge\), defined in Definition~\ref{def:R2R3:mutilated_graphs}.}
    \label{fig:R2R3_conditions}
\end{figure}

We have the following lemma:

\begin{restatable}{lemma}{thRdRtimplications}
\label{lemma:r2R3:implications}
    Let $\G$ be an ADMG over $\V$ and let $\Y,\X,\W,\Zu,\Zd$ be pairwise disjoint subsets of~$\V$. We have the following propositions:
    \begin{enumerate}
        \item \(\Y \ind_{\Gc} \Zd \mid \X,\W,\Zu \Rightarrow \Y \ind_{\Ga} \Zd \mid \X,\Zu,\W\)  \label{lemma:r2R3:implications:1}
        \item \(\Y \ind_{\Gb} \Zu \mid \X,\W \Rightarrow \Y \ind_{\Gd} \Zu \mid \X,\Zd,\W\) \label{lemma:r2R3:implications:2}
    \end{enumerate}
\end{restatable}

\begin{proof}
Let us prove the two implications:
\begin{itemize}
    \item Let us prove the first implication. We have the following identities:
    \begin{align*}
      & \G_{\overline{\X\Zu}} \subseteq \G_{\overline{\X}}\\
      \Rightarrow\quad&\Anc(\W,\G_{\overline{\X\Zu}}) \subseteq \Anc(\W,\G_{\overline{\X}}) \\
      \Rightarrow\quad &\Anc(\W,\G_{\overline{\X\Zu}}) \subseteq \Anc(\W\cup\Zu,\G_{\overline{\X}}) \\
      \Rightarrow\quad &\Zd \setminus \Anc(\W,\G_{\overline{\X\Zu}}) \supseteq \Zd \setminus \Anc(\W\cup\Zu,\G_{\overline{\X}})\\
      \Rightarrow\quad & \G_{\overline{X} \, \overline{\Zd \setminus \Anc(\W,\G_{\overline{\X\Zu}})}} \subseteq \G_{\overline{X} \, \overline{\Zd \setminus \Anc(\W\cup\Zu,\G_{\overline{\X}})}}\\
      \Rightarrow\quad & \G_{\overline{X\Zu} \, \overline{\Zd \setminus \Anc(\W,\G_{\overline{\X\Zu}})}} \subseteq \G_{\overline{X} \, \overline{\Zd \setminus \Anc(\W\cup\Zu,\G_{\overline{\X}})}}\\
      \Rightarrow\quad & \Ga \subseteq \Gc
    \end{align*}
    
    Therefore, \(\Y \ind_{\Gc} \Zd \mid \X,\W,\Zu \Rightarrow \Y \ind_{\Ga} \Zd \mid \X,\Zu,\W\).
    
    \item Let us prove the second implication by contrapositive. Let us assume that $\Y \notind_{\Gd} \Zu \mid \X,\Zd,\W$. Then, $\Gd \subseteq \Gb$ contains a path $\pi$ from $\Y$ to $\Zu$ that is $\X,\Zd,\W$-active. Thus all non colliders on $\pi$ are not in $\W,\X$. Moreover, since $\Gd$ contains the mutilation $\overline{\Zd}$, we know that all collider on $\pi$ are ancestors of $\X,\W$ in $\G_{\overline{\X}\,\underline{\Zu}} = \Gb$. Therefore $\Y \notind_{\Gb} \Zu \mid \X,\W$.
\end{itemize}

The second implication directly follows from $\Gd \subseteq \Gb$.
\end{proof}

\thRdeuxRtrois*

\begin{proof}
    Graphically, using Definition~\ref{def:R2R3:mutilated_graphs} and Figure~\ref{fig:R2R3_conditions}, the theorem corresponds to the following properties:
    \begin{enumerate}
        \item \(\Y \ind_{\Ga} \Zd \mid \X,\Zu,\W\) and  \(\Y \ind_{\Gb} \Zu \mid \X, \W\) $\Leftrightarrow$  \(\Y \ind_{\Gd} \Zu \mid \X,\Zd,\W\) and \(\Y \ind_{\Gc} \Zd \mid \X,\W,\Zu\)
        
        \item \(\Y \ind_{\Gc} \Zd \mid \X,\W, \Zu\) and \(\Y \ind_{\Gb} \Zu \mid \X,\W\) \(\Rightarrow\) \(\Y \ind_{\Ga} \Zd \mid \X,\Zu, \W\) and \(\Y \ind_{\Gd} \Zu \mid \X,\Zd,\W\)
        
    \end{enumerate}
We prove the three implications:
\begin{itemize}
    \item \(\Y \ind_{\Gc} \Zd \mid \X,\W, \Zu\) and \(\Y \ind_{\Gb} \Zu \mid \X,\W\) \(\Rightarrow\) \(\Y \ind_{\Ga} \Zd \mid \X,\Zu, \W\) and \(\Y \ind_{\Gd} \Zu \mid \X,\Zd,\W\): this implication is a direct consequence of Lemma~\ref{lemma:r2R3:implications}.
    
    \item \(\Y \ind_{\Ga} \Zd \mid \X,\Zu,\W\) and  \(\Y \ind_{\Gb} \Zu \mid \X, \W\) $\Rightarrow$  \(\Y \ind_{\Gd} \Zu \mid \X,\Zd,\W\) and \(\Y \ind_{\Gc} \Zd \mid \X,\W,\Zu\): let us assume that \(\Y \ind_{\Ga} \Zd \mid \X,\Zu,\W\) and  \(\Y \ind_{\Gb} \Zu \mid \X, \W\) holds in $\G$. By Lemma~\ref{lemma:r2R3:implications}, \(\Y \ind_{\Gd} \Zu \mid \X,\Zd,\W\) holds in $\G$. Thus, by Theorem \ref{th:cycle_do_calculus}, \(\Y \ind_{\Gc} \Zd \mid \X,\W,\Zu\) also holds in $\G$. Therefore, we have \(\Y \ind_{\Gd} \Zu \mid \X,\Zd,\W\) and \(\Y \ind_{\Gc} \Zd \mid \X,\W,\Zu\).
    
    \item \(\Y \ind_{\Ga} \Zd \mid \X,\Zu,\W\) and  \(\Y \ind_{\Gb} \Zu \mid \X, \W\) $\Leftarrow$  \(\Y \ind_{\Gd} \Zu \mid \X,\Zd,\W\) and \(\Y \ind_{\Gc} \Zd \mid \X,\W,\Zu\): this case is symmetrical to the previous one.
\end{itemize}
\end{proof}

\subsection{Normal Forms of Do-Calculus Derivations}

The results established in the previous sections show that for any equivalent expressions \(Q_1\) and \(Q_2\), 
\(Q_2\) can be deduced from \(Q_1\) by successive applications of variants of rules 2 and 3. 
A priori, the number of applications is unknown and could be large. 
Here, we provide a more precise statement, showing that this sequence can always be reduced to at most four steps. 
This makes the exploration of any connected component of the derivation graph more computationally feasible, 
and moreover, we provide an explicit expression for each intermediate expression leading from \(Q_1\) to \(Q_2\).

\thformenormal*

\begin{proof}
    If the sequence 
    $\Ru_2(\Z_2^\uparrow)\, \Ru_3(\Z_3^\uparrow)\,
        \Rd_3(\Z_3^\downarrow)\,
        \Rd_2(\Z_2^\downarrow)$
    is valid in $\G$ and transforms 
    $P(\mathbf{y}\mid \Do(\mathbf{x}),\mathbf{w})$ into 
    $P(\mathbf{y} \mid \Do(\mathbf{\widetilde x}),\mathbf{\widetilde w})$,
    then the two distributions are equivalent.
    
    \textbf{Conversely}, assume that
    $P(\mathbf{y}\mid \Do(\mathbf{x}),\mathbf{w})$ and $ P(\mathbf{y} \mid \Do(\mathbf{\widetilde x}),\mathbf{\widetilde w})$
    are equivalent. By definition, there exists a finite sequence of applications of do-calculus rules
    \[
    \mathcal{R}_{i_1}(\mathbf{Z_1})\,\mathcal{R}_{i_2}(\mathbf{Z_2})\cdots \mathcal{R}_{i_n}(\mathbf{Z_n})
    \]
    that transforms $P(\mathbf{y}\mid \Do(\mathbf{x}),\mathbf{w})$ into 
    $P(\mathbf{y} \mid \Do(\mathbf{\widetilde x}),\mathbf{\widetilde w})$.
    
    By Theorem~\ref{thm:r1_as_r23}, every application of $\mathcal{R}_1$ can be replaced by a valid sequence of applications of $\mathcal{R}_2$ and $\mathcal{R}_3$. Hence, without loss of generality, we may assume that the sequence only involves rules $\mathcal{R}_2$ and $\mathcal{R}_3$.     Moreover, by Theorems~\ref{th:r2_commute} and \ref{th:r3_commute}, each  application involving a set of variables can be decomposed into a sequence of applications on single variables. Hence, we may further assume without loss of generality that each $\mathbf{Z_j}$ is a singleton, for $1\leq j \leq n$.
    
    By Theorems~\ref{th:r2_commute}, \ref{th:r3_commute}, and
    \ref{th:r2R3_commute}, downward and then upward applications of rules
    $\mathcal{R}_2$ and $\mathcal{R}_3$ commute (i.e. $\G \models \Rd_{k} \, \Ru_{\ell} \Rightarrow \G \models \Ru_{\ell} \, \Rd_{k}$, for $k,\ell \in \{2,3\}$).
    Therefore, by rearranging the sequence, we may assume that all upward
    applications precede all downward applications, that is, the sequence can be written as
    \[
    \Ru_{i_1}(Z_1)\cdots \Ru_{i_k}(Z_k)\,
    \Rd_{i_{k+1}}(Z_{k+1})\cdots \Rd_{i_n}(Z_n),
    \]
    with $i_j\in\{2,3\}$ for all $j \in \{1,\ldots,n\}$.

    Using again the commutativity of the applications of the rules, applications in the same direction commute (i.e. $\G \models \Ru_{k} \, \Ru_{l} \Rightarrow \G \models \Ru_{l} \, \Ru_{k}$ and $\G \models \Rd_{k} \, \Rd_{l} \Rightarrow \G \models \Rd_{l} \, \Rd_{k}$, for $k,l \in \{2,3\}$). Therefore, any application
    $\Ru_i(Z)$ can be moved arbitrarily far to the right among the upward
    applications $\Ru_2$ and $\Ru_3$, and similarly any application $\Rd_i(Z)$ can be moved arbitrarily far to the left among the downward
    applications of rules $\Rd_2$ and $\Rd_3$.
    Moreover, since any consecutive pair of inverse applications of the form
    $\Ru_i(Z)\,\Rd_i(Z)$ can be removed from the sequence, we may assume that no
    variable $Z$ is acted upon twice by the same rule in opposite directions.

    Using again the commutativity properties of the rules, applications of the same
    rule and in the same direction commute.
    As a consequence, the remaining applications can be sorted by direction and rule type. Hence, the sequence can be written as
    
    \[ \Ru_{2}(Z_1) \dots \Ru_{2}(Z_k) \, \Ru_{3}(Z_{k+1}) \dots \Ru_{3}(Z_{l})\, \Rd_{3}(Z_{l+1}) \dots \Rd_{3}(Z_{m}) \cdots \Rd_{3}(Z_l) \dots \Rd_{3}(Z_n) \]
    
    Moreover, by Theorems~\ref{th:r2_commute} and \ref{th:r3_commute}, multiple applications of the same rule in the same direction can be combined into a single application on the union of the corresponding variables. Consequently, there exist (possibly empty) subsets
    $\Z_2^\uparrow, \Z_2^\downarrow, \Z_3^\uparrow$, and $\Z_3^\downarrow$ of $\V$
    such that the sequential application
    \[
    \Ru_2(\Z_2^\uparrow)\,
    \Ru_3(\Z_3^\uparrow)\,
    \Rd_3(\Z_3^\downarrow)\,
    \Rd_2(\Z_2^\downarrow)
    \]
    is valid in $\G$ and transforms
    $P(\y \mid \Do(\x_1),\x_1)$ into
    $P(\y \mid \Do(\x_2),\w_2)$.
    
    Moreover, since no variable $Z$ was acted upon twice by the same rule in opposite directions we know that $\Z_2^\uparrow \cap \Z_2^\downarrow = \emptyset$ and $\Z_3^\uparrow \cap \Z_3^\downarrow = \emptyset$. Therefore, we have $\Z_2^\uparrow = \W \setminus \widetilde \W$, $\Z_3^\uparrow = (\widetilde \X \cup \widetilde \W) \setminus (\X \cup  \W)$, $\Z_2^\downarrow = \widetilde \W \setminus \W$ and $\Z_3^\downarrow = (\X \cup  \W) \setminus (\widetilde \X \cup \widetilde \W)$.
\end{proof}

\corcriteregraphique*

\begin{proof}
This criterion is the graphical conditions of the sequence of applications given in Theorem~\ref{th:formenormal}.
\end{proof}

\cordiametre*

\begin{proof}
Theorem~\ref{th:formenormal} shows that there is always a sequence of application of length at most four between two equivalent densities.
\end{proof}

\section{Another Example: A Causal Graph With Four Nodes, and Its Derivation Graph Restricted to One Connected Component}
\label{app:examples}

\begin{figure}
    \centering
\begin{tikzpicture}[baseline=-0.5ex, node distance=8mm,
      node/.style = {inner sep=.3pt, minimum size=6mm, font=\small},
      >={Stealth[length=5pt]}, every edge/.style={->, semithick}]
      
      \node[node] (B) {$B$};
      \node[node] (A) [right=of B] {$A$};
      \node[node] (D) [right=of A] {$D$};
      \node[node] (C) [right=of D] {$C$};
      
      \draw[->] (B) -- (A);
      \draw[dashed,<->,bend left=35] (A) to (D);

    \end{tikzpicture}
\begin{tikzpicture}[
    dofree/.style = {
      draw=coolorange!60!black,
      fill=coolorange!80,
      line width=0.6pt,
      inner sep=1pt,
      font=\scriptsize
    },
    withdo/.style = {
      draw=blue!60!black,
      fill=blue!25,
      line width=0.6pt,
      inner sep=1pt,
      font=\scriptsize
    },
    styleR1/.style = {thick, -, gray!65, line width=1.2pt},
    styleR2/.style = {thick, -, orange!60!brown, dashed, line width=1.2pt},
    styleR3/.style = {thick, -, blue!55, dotted, line width=1.5pt},
    scale=1.7,
]
    \node[dofree] (n11) at (0, 0) {$P(a \mid b)$};
    \node[dofree] (n2) at (1, 0) {$P(a \mid b,c)$};
    \node[withdo] (n9) at (0.5, .866) {$P(a \mid \mathrm{do}(c),b)$};
    
    \node[withdo] (n5) at (0, 2.5) {$P(a \mid \mathrm{do}(b,d))$};
    \node[withdo] (n4) at (1, 2.5) {$P(a \mid \mathrm{do}(b,d),c)$};
    \node[withdo] (n7) at (.5, 3.366) {$P(a \mid \mathrm{do}(b,c,d))$};

    \node[withdo] (n1) at (1.5, 1.25) {$P(a \mid \mathrm{do}(b))$};    
    \node[withdo] (n0) at (2.5, 1.25) {$P(a \mid \mathrm{do}(b),c)$};
    \node[withdo] (n3) at (2, 2.116) {$P(a \mid \mathrm{do}(b,c))$};

    \node[withdo] (n10) at (-1.5, 1.25) {$P(a \mid\mathrm{do}(d),b)$};
    \node[withdo] (n6) at (-.5, 1.25) {$P(a \mid \mathrm{do}(d),b,c)$};
    \node[withdo] (n8) at (-1, 2.116) {$P(a \mid \mathrm{do}(c,d),b)$};

\begin{pgfonlayer}{background}
    \draw[styleR1] (n0) -- (n1);
    \draw[styleR2] (n0) -- (n2);
    \draw[styleR2] (n0) -- (n3);
    \draw[styleR3,line width = 1pt] (n0) -- (n4);
    \draw[styleR2,line width = .9pt] (n1) -- (n11);
    \draw[styleR3] (n1) -- (n5);
    \draw[styleR3] (n1) -- (n3);
    \draw[styleR1] (n2) -- (n11);
    \draw[styleR2] (n2) -- (n9);
    \draw[styleR3,line width = 1pt] (n2) -- (n6);
    \draw[styleR3] (n3) -- (n7);
    \draw[styleR2] (n3) -- (n9);
    \draw[styleR1] (n4) -- (n5);
    \draw[styleR2] (n4) -- (n6);
    \draw[styleR2] (n4) -- (n7);
    \draw[styleR3] (n5) -- (n7);
    \draw[styleR2,line width = .9pt] (n5) -- (n10);
    \draw[styleR2] (n6) -- (n8);
    \draw[styleR1] (n6) -- (n10);
    \draw[styleR2] (n7) -- (n8);
    \draw[styleR3] (n8) -- (n9);
    \draw[styleR3] (n8) -- (n10);
    \draw[styleR3] (n9) -- (n11);
    \draw[styleR3] (n10) -- (n11);
\end{pgfonlayer}

\end{tikzpicture}
\caption{The causal graph (left) and its derivation graph restricted to the connected component of $P(a \mid \Do(b))$ (right). Blue nodes denote expressions involving interventions, while orange nodes denote observational (do-free) expressions. 
Edges represent atomic applications of the do-calculus rules (one variable at a time): gray solid lines for Rule $\mathcal{R}1$, orange dashed lines for Rule $\mathcal{R}2$, and blue dotted lines for Rule $\mathcal{R}3$.}
\label{fig:derivation_graph_index3}
\end{figure}
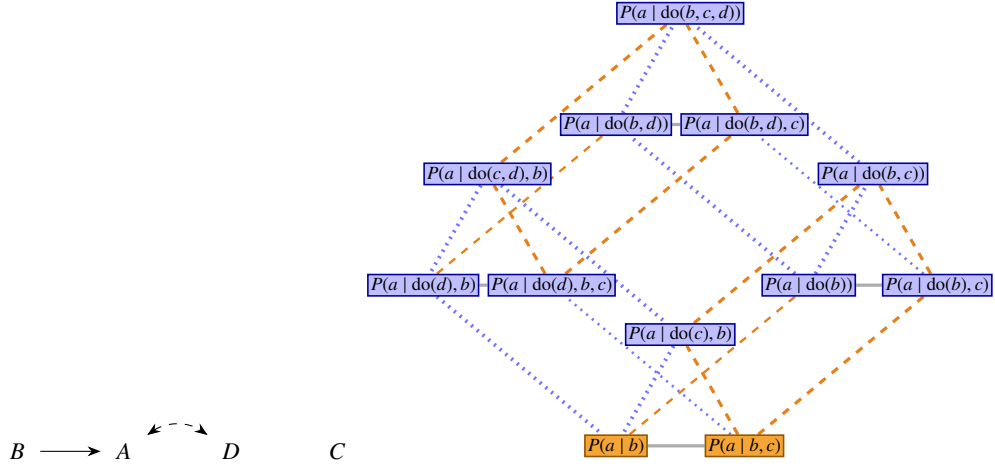

Figure~\ref{fig:derivation_graph_index3} shows the corresponding causal diagram (left) together with its derivation graph (right).

This example illustrates all the favorable configurations: we observe the triangles formed by $\mathcal{R}_1$, $\mathcal{R}_2$, and $\mathcal{R}_3$, as well as the different types of quadrilaterals—namely the blue dotted one, the orange dashed one, and the mixed one.
In this case, all compositions are commutative; consequently, no “hole” appears, in contrast to the napkin graph.
\newpage

\section{Application: Deriving Several Identification Formulae for an Interventional Query}
\label{app:example_estimation}

\subsection{Estimation in a Simple Case: Details}
\label{app:example_estimation:1}

We illustrate our framework on a simple example,
and  compute two interventional quantities that belong
to the same connected component of the derivation graph and estimate them on a generated sample.
This example highlights how different valid derivations correspond
to different decompositions of the causal effect.

\begin{figure}[H]\centering
\begin{tikzpicture}
    \node (A) at (0,0) {W};
    \node (B) at (0.9,0) { Z};
    \node (C) at (1.8,0) { X};
    \node (D) at (2.7,0) { Y};

    \draw[->] (A) -- (B);
    \draw[->] (B) -- (C);
    \draw[->] (C) -- (D);
    \draw[dashed, <->] (A) to[bend left=30] (D);
\end{tikzpicture} 
\hspace{2cm}
\begin{tikzpicture}[
    dofree/.style = {
      draw=coolorange!60!black,
      fill=coolorange!80,
      line width=0.6pt,
      inner sep=1pt,
      font=\small
    },
    withdo/.style = {
      draw=blue!60!black,
      fill=blue!25,
      line width=0.6pt,
      inner sep=1pt,
      font=\small
    },
    styleR1/.style = {thick, -, gray!65, line width=1.2pt},
    styleR2/.style = {thick, -, orange!60!brown, dashed, line width=1.2pt},
    styleR3/.style = {thick, -, blue!55, dotted, line width=1.5pt},
    scale=2.5,
]
    \node[withdo] (n1) at (0, 0) {$P(y \mid \mathrm{do}(w,z))$};
    \node[withdo] (n2) at (1, 0) {$P(y \mid \mathrm{do}(w),z)$};
    \node[withdo] (n3) at (-1, 0) {$P(y \mid \mathrm{do}(z))$};

\begin{pgfonlayer}{background}
    \draw[styleR3] (n1) -- (n3);
    \draw[styleR2] (n1) -- (n2);
\end{pgfonlayer}

\end{tikzpicture}
\caption{A causal diagram over four variables, used in Example~\ref{ex:1}, and the connected component of the expression $P(y \mid \Do(z))$ in the derivation graph.}
\label{fig:1bis}
\end{figure}
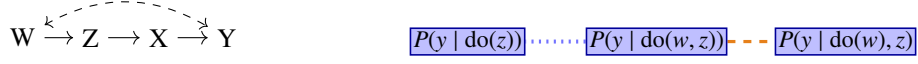

We generate data with the following process: the vector $(W,Z,X,Y)$ is Gaussian, and the relations are linear, with
\begin{align*}
    &\text{Latent confounder:} && U \sim \mathcal{N}(0,1)\\[1mm]
    & && W = a_{UW} \, U + \epsilon_W, && \epsilon_W \sim \mathcal{N}(0,1), \quad a_{UW} = 1.0\\[1mm]
    &\text{Mediator 1:} && Z = a_{WZ} \, W + \epsilon_Z, && \epsilon_Z \sim \mathcal{N}(0,1), \quad a_{WZ} = 1.5\\[1mm]
    &\text{Mediator 2:} && X = a_{ZX} \, Z + \epsilon_X, && \epsilon_X \sim \mathcal{N}(0,1), \quad a_{ZX} = 2.0\\[1mm]
    &\text{Outcome:} && Y = a_{XY} \, X + a_{UY} \, U + \epsilon_Y, && \epsilon_Y \sim \mathcal{N}(0,1), \quad a_{XY} = 1.2, a_{UY} =  1.5 .
\end{align*}

From the derivation graph, we know that 
$$\mathcal{A}^{\mathcal{G}}_{\mathbb{E}(Y \mid \Do(w),z)}
= \{\mathbb{E}(Y \mid \Do(w),z), \mathbb{E}(Y \mid \Do(z)), \mathbb{E}(Y \mid \Do(w,z))\}.$$

We first consider the interventional distribution $\mathbb{E}(Y \mid \Do(z))$, which corresponds to marginalizing over $W$:
\begin{align*}
    \mathbb{E}(Y \mid \Do(z)) 
    &= \int y \, f(y \mid \Do(z)) \, dy
     = \int y \int  f(y \mid w,z) f(w) \, dw \, dy\\
    &= \int  f(w) \, \mathbb{E}(y \mid w,z) \, dw 
     = \gamma + \alpha \, \mathbb{E}(W) + \beta z,
\end{align*}
where we used the regression model $\mathbb{E}(Y \mid w,z) = \gamma + \alpha w + \beta z$.  
The total effect of $Z$ on $Y$, $\frac{\partial}{\partial Z} \mathbb{E}(Y \mid \Do(z))$, is $\beta$.

Next, we consider $\mathbb{E}(Y \mid \Do(w,z))$:
\begin{align*}
    \mathbb{E}(Y \mid \Do(w,z)) 
    &= \int y \, f(y \mid \Do(w,z)) \, dy \\
    &= \int y \int  f(y \mid w,z,x) f(x \mid w,z) \, dx \, dy\\
    &= \int f(x \mid w,z) (\widetilde\gamma + \widetilde\alpha w + \widetilde\beta z + \widetilde\delta x) \, dx\\
    &= \widetilde\gamma + \widetilde\alpha w + \widetilde\beta z + \widetilde\delta \, \mathbb{E}(X \mid w,z) \\
    &= \widetilde\gamma + \widetilde\alpha w + \widetilde\beta z + \widetilde\delta (\gamma' + \alpha' w + \beta' z),
\end{align*}
where the regression models are $\mathbb{E}(Y \mid w,z,x) = \gamma + \alpha w + \beta z + \delta x$ and 
$\mathbb{E}(X \mid w,z) = \gamma' + \alpha' w + \beta' z$.  
Hence, the total effect of $(W,Z)$ on $Y$ is $(\delta \alpha', \delta \beta')$.

We provide a Jupyter notebook in \href{https://gricad-gitlab.univ-grenoble-alpes.fr/yvernesc/do-calculus-derivation-graphs}{this repository}\footnote{https://gricad-gitlab.univ-grenoble-alpes.fr/yvernesc/do-calculus-derivation-graphs} to compute the estimator and get the boxplot in Figure \ref{fig:boxplot}.

This example illustrates how different elements of the same connected component in the derivation graph correspond to distinct but equivalent expressions of causal effects, and how the proposed framework makes these relationships explicit.

\subsection{Experiment on the Sachs Protein‑Signaling Data}
\label{app:sachs}

We applied our method to the protein signaling dataset introduced by \citet{Sachs2005}. The dataset is a widely used benchmark in causal inference and contains single-cell measurements of 11 variables, namely Raf, Mek, Plcg, PIP2, PIP3, Erk, Akt, PKA, PKC, P38, and Jnk, recorded across several thousand cells under multiple experimental conditions. We use the observational data available in the \texttt{bnlearn} R package, together with the biologically validated DAG associated with this dataset~\cite{zzzSachsDAG,zzzSachsDataset} (see Figure~\ref{fig:SAchsDAG}), which we take as ground truth.

We consider the estimation of the causal effect $P(\mathrm{P38} \mid do(\mathrm{Mek}))$. According to the reference graph, this effect is equal to zero. Starting from the causal graph, the derivation graph produced by our method contains 32 equivalent interventional densities. Since the underlying graph is a DAG, each density leads to an identifiable adjustment formula via parental adjustment. Several of these densities induce the same adjustment set and therefore correspond to the same estimator. For each distinct adjustment set, we estimate the causal effect following the procedure described in Section~\ref{sec:app}.

To assess the variability of the resulting estimators, we compute bootstrap variances using 500 resamples. Table~\ref{tab:sachs_resultsd} reports the distinct estimators obtained from our approach, together with their adjustment sets, point estimates, and bootstrap variances. All estimators return values close to zero, in agreement with the ground truth. As in the synthetic example of Section~\ref{sec:app}, the different equivalent formulae exhibit noticeably different variances. In particular, the estimators based on adjustment sets involving Mek, PKA, and PKC, with or without PIP3, achieve the smallest bootstrap variances, indicating better statistical efficiency than alternatives such as the unadjusted estimator.

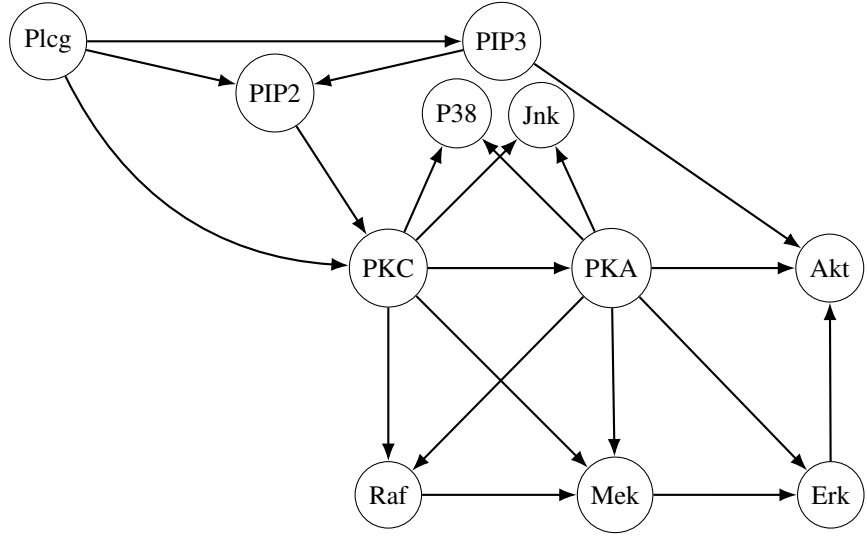
\begin{figure}[t]
    \centering
    \begin{tikzpicture}[
        node distance = 1.9cm,
        every node/.style={circle, draw, minimum size=7mm},
        arrow/.style={-Latex, thick},
        scale=3
    ]
    
    \node (Plcg) at (-0.5,-.772) {Plcg};
    \node (PIP3) at (1.5,-.772) {PIP3};
    \node (PIP2) at (.5,-1) {PIP2};
    \node (PKC) at (1,-1.772) {PKC};
    
    \node (PKA) [right=of PKC] {PKA};
    
    \node (Raf) at (1,-2.772)  {Raf};
    \node (Mek) at (2,-2.772)  {Mek};
    \node (Erk) [right=of Mek] {Erk};
    \node (Akt) [right=of PKA] {Akt};
    
    \node (Jnk) [ above right=of PKC] {Jnk};
    \node (P38) [above left=of PKA] {P38};
    
    \path (Plcg) edge[-Latex, thick,bend right = 30] (PKC);
    \draw[arrow] (Plcg) -- (PIP2);
    \draw[arrow] (Plcg) -- (PIP3);
    
    \draw[arrow] (PIP3) -- (Akt);
    \draw[arrow] (PIP3) -- (PIP2);
    \draw[arrow] (PIP2) -- (PKC);
    
    \draw[arrow] (PKC) -- (Jnk);
    \draw[arrow] (PKC) -- (P38);
    \draw[arrow] (PKC) -- (Mek);
    \draw[arrow] (PKC) -- (Raf);
    \draw[arrow] (PKC) -- (PKA);
    
    \draw[arrow] (PKA) -- (Jnk);
    \draw[arrow] (PKA) -- (P38);
    \draw[arrow] (PKA) -- (Mek);
    \draw[arrow] (PKA) -- (Raf);
    \draw[arrow] (PKA) -- (Erk);
    \draw[arrow] (PKA) -- (Akt);
    
    \draw[arrow] (Raf) -- (Mek);
    \draw[arrow] (Mek) -- (Erk);
    \draw[arrow] (Erk) -- (Akt);
    
    \end{tikzpicture}
    \caption{The DAG in \citet{Sachs2005} biologically validated.}
    \label{fig:SAchsDAG}
\end{figure}

\begin{table}[t]
    \centering
    \caption{Estimators obtained from the equivalent identification formulae for the causal effect $P(\mathrm{P38} \mid do(\mathrm{Mek}))$ on the Sachs protein-signaling dataset.}
    \label{tab:sachs_resultsd}
    \begin{tabular}{l l r r}
    \hline
    Density & Adjustment set & Prediction & Bootstrap variance \\
    \hline
    $P(\mathrm{P38} \mid do(\mathrm{Mek}))$ & $\{\mathrm{PKA}, \mathrm{PKC}, \mathrm{Raf}\}$ & 0.022723 & 0.000572 \\
    $P(\mathrm{P38})$ & $\emptyset$ & -0.018090 & 0.001029 \\
    $P(\mathrm{P38} \mid do(\mathrm{Mek},\mathrm{Raf}))$ & $\{\mathrm{PKA}, \mathrm{PKC}\}$ & -0.006157 & 0.000265 \\
    $P(\mathrm{P38} \mid do(\mathrm{Akt},\mathrm{Mek}))$ & $\{\mathrm{PIP3}, \mathrm{PKA}, \mathrm{PKC}, \mathrm{Raf}\}$ & 0.021494 & 0.000574 \\
    $P(\mathrm{P38} \mid do(\mathrm{Erk}))$ & $\{\mathrm{Mek}, \mathrm{PKA}\}$ & -0.004451 & 0.000065 \\
    $P(\mathrm{P38} \mid do(\mathrm{Akt},\mathrm{Erk}))$ & $\{\mathrm{Mek}, \mathrm{PIP3}, \mathrm{PKA}\}$ & -0.004954 & 0.000064 \\
    $P(\mathrm{P38} \mid do(\mathrm{Akt},\mathrm{Erk},\mathrm{Mek},\mathrm{Raf}))$ & $\{\mathrm{PIP3}, \mathrm{PKA}, \mathrm{PKC}\}$ & -0.007231 & 0.000262 \\
    $P(\mathrm{P38} \mid do(\mathrm{Akt},\mathrm{Erk},\mathrm{Jnk}))$ & $\{\mathrm{Mek}, \mathrm{PIP3}, \mathrm{PKA}, \mathrm{PKC}\}$ & -0.001808 & \textbf{0.000016} \\
    $P(\mathrm{P38} \mid do(\mathrm{Akt},\mathrm{Jnk}))$ & $\{\mathrm{Erk}, \mathrm{PIP3}, \mathrm{PKA}, \mathrm{PKC}\}$ & -0.007302 & 0.000269 \\
    $P(\mathrm{P38} \mid do(\mathrm{Akt}))$ & $\{\mathrm{Erk}, \mathrm{PIP3}, \mathrm{PKA}\}$ & -0.019925 & 0.001032 \\
    $P(\mathrm{P38} \mid do(\mathrm{Erk},\mathrm{Jnk}))$ & $\{\mathrm{Mek}, \mathrm{PKA}, \mathrm{PKC}\}$ & -0.001539 & \textbf{0.000017} \\
    \hline
    \end{tabular}
\end{table}

\newpage
\subsection{Details on a More Complex Causal Graph}
\label{app:details}
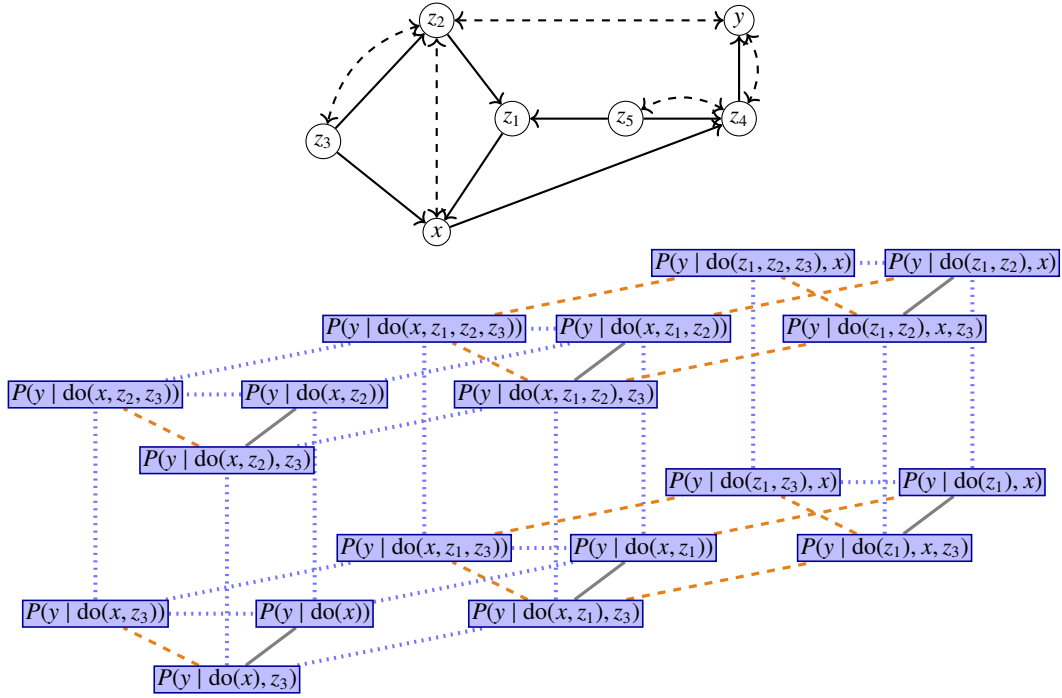
\begin{figure}[h]
\centering
\begin{tikzpicture}[
    node/.style={draw, circle, inner sep=1.5pt, font=\small},
    dir/.style={->, thick},
    bidir/.style={<->, dashed, thick},
    scale=1
]

\node[node] (x)  at (0,0) {$x$};
\node[node] (y)  at (4,2.8) {$y$};

\node[node] (z1) at (1,1.5) {$z_1$};
\node[node] (z2) at (0,2.8) {$z_2$};
\node[node] (z3) at (-1.5,1.2) {$z_3$};

\node[node] (z4) at (4,1.5) {$z_4$};
\node[node] (z5) at (2.5,1.5) {$z_5$};

\draw[dir] (z1) -- (x);
\draw[dir] (x) -- (z4);
\draw[dir] (z4) -- (y);
\draw[dir] (z2) -- (z1);
\draw[dir] (z3) -- (z2);
\draw[dir] (z3) to (x);
\draw[dir] (z5) -- (z1);
\draw[dir] (z5) -- (z4);

\draw[bidir] (x)  to (z2);
\draw[bidir] (z2) to[bend right=30] (z3);
\draw[bidir] (y)  to (z2);
\draw[bidir] (y)  to[bend left=30] (z4);
\draw[bidir] (z4) to[bend right=30] (z5);

\end{tikzpicture}
\begin{tikzpicture}[
    dofree/.style = {
      draw=coolorange!60!black,
      fill=coolorange!80,
      line width=0.6pt,
      inner sep=1pt,
      font=\footnotesize
    },
    withdo/.style = {
      draw=blue!60!black,
      fill=blue!25,
      line width=0.6pt,
      inner sep=1pt,
      font=\footnotesize
    },
    styleR1/.style = {thick, -, gray, line width=1.2pt},
    styleR2/.style = {thick, -, orange!60!brown, dashed, line width=1.2pt},
    styleR3/.style = {thick, -, blue!55, dotted, line width=1.5pt},
    scale=2.9
]
    
    \node[withdo] (n2) at (.6, -.3) {$P(y \mid \Do(x),z_3)$};
    \node[withdo] (n3) at (0, 0) {$P(y \mid \Do(x,z_3))$};
    \node[withdo] (n4) at (0.6, 0.7) {$P(y \mid \Do(x,z_2),z_3)$};
    \node[withdo] (n5) at (1, 0) {$P(y \mid \Do(x))$};
    \node[withdo] (n6) at (1, 1) {$P(y \mid \Do(x,z_2))$};
    \node[withdo] (n7) at (0, 1) {$P(y \mid \Do(x,z_2,z_3))$};

   \node[withdo] (n22) at (2.1, 0) {$P(y \mid \Do(x,z_1),z_3)$};
    \node[withdo] (n32) at (1.5, 0.3) {$P(y \mid \Do(x,z_1,z_3))$};
    \node[withdo] (n42) at (2.1, 1) {$P(y \mid \Do(x,z_1,z_2),z_3)$};
    \node[withdo] (n52) at (2.5, 0.3) {$P(y \mid \Do(x,z_1))$};
    \node[withdo] (n62) at (2.5, 1.3) {$P(y \mid \Do(x,z_1,z_2))$};
    \node[withdo] (n72) at (1.5, 1.3) {$P(y \mid \Do(x,z_1,z_2,z_3))$};
    
        \node[withdo] (n23) at (3.6, 0.3) {$P(y \mid \Do(z_1),x,z_3)$};
    \node[withdo] (n33) at (3, 0.6) {$P(y \mid \Do(z_1,z_3),x)$};
    \node[withdo] (n43) at (3.6, 1.3) {$P(y \mid \Do(z_1,z_2),x,z_3)$};
    \node[withdo] (n53) at (4, 0.6) {$P(y \mid \Do(z_1),x)$};
    \node[withdo] (n63) at (4, 1.6) {$P(y \mid \Do(z_1,z_2),x)$};
    \node[withdo] (n73) at (3, 1.6) {$P(y \mid \Do(z_1,z_2,z_3),x)$};
\begin{pgfonlayer}{background}
     \draw[styleR2] (n2) -- (n3);
    \draw[styleR3] (n2) -- (n4);
    \draw[styleR1] (n2) -- (n5);
    \draw[styleR3] (n3) -- (n5);
    \draw[styleR3] (n3) -- (n7);
    \draw[styleR1] (n4) -- (n6);
    \draw[styleR2] (n4) -- (n7);
    \draw[styleR3] (n5) -- (n6);
    \draw[styleR3] (n6) -- (n7);
    
    \draw[styleR2] (n22) -- (n32);
    \draw[styleR3] (n22) -- (n42);
    \draw[styleR1] (n22) -- (n52);
    \draw[styleR3] (n32) -- (n52);
    \draw[styleR3] (n32) -- (n72);
    \draw[styleR1] (n42) -- (n62);
    \draw[styleR2] (n42) -- (n72);
    \draw[styleR3] (n52) -- (n62);
    \draw[styleR3] (n62) -- (n72);
    
         \draw[styleR2] (n23) -- (n33);
    \draw[styleR3] (n23) -- (n43);
    \draw[styleR1] (n23) -- (n53);
    \draw[styleR3] (n33) -- (n53);
    \draw[styleR3] (n33) -- (n73);
    \draw[styleR1] (n43) -- (n63);
    \draw[styleR2] (n43) -- (n73);
    \draw[styleR3] (n53) -- (n63);
    \draw[styleR3] (n63) -- (n73);
    
    \draw[styleR3] (n6) -- (n62);
    \draw[styleR2] (n62) -- (n63);
    \draw[styleR3] (n7) -- (n72);
    \draw[styleR2] (n72) -- (n73);
        \draw[styleR3] (n4) -- (n42);
    \draw[styleR2] (n42) -- (n43);
        \draw[styleR3] (n2) -- (n22);
    \draw[styleR2] (n22) -- (n23);
        \draw[styleR3] (n3) -- (n32);
    \draw[styleR2] (n32) -- (n33);
        \draw[styleR3] (n5) -- (n52);
    \draw[styleR2] (n52) -- (n53);
\end{pgfonlayer}
\end{tikzpicture}
\caption{Causal graph (top) and connected component of $P(y\mid\Do(x))$ in its derivation graph (bottom).}
\end{figure}

We list all identification formulae obtained by applying the ID algorithm (via the R package \texttt{causaleffect}) to all equivalent interventional queries within the same connected component of the derivation graph. This yields eight distinct formulae, each exhibiting different statistical properties.

\begin{align*}
    P(y \mid \Do(x))&=     P(y \mid \Do(x,z_3))
    =     P(y \mid \Do(x,z_1,z_3))=     P(y \mid \Do(x,z_1))
=
\frac{
\sum_{z_2,z_5}
P(y \mid x,z_1,z_2,z_5)\,
P(x \mid z_1,z_2,z_5)\,
P(z_2 \mid z_5)\,
P(z_5)
}{
\sum_{z_2}
P(x \mid z_1,z_2)\,
P(z_2)
}\\
P(y \mid \Do(x,z_2)) &= P(y \mid \Do(x,z_2,z_3))=
\frac{
\sum_{z_2', z_5} 
P(y \mid z_5, z_2', z_1, x)\,
P(x \mid z_5, z_2', z_1)\,
P(z_2' \mid z_5)\,
P(z_5)
}{
\sum_{z_2'} 
P(x \mid z_5, z_2', z_1)\,
P(z_2')
}\\
P(y \mid \Do(x,z_1,z_2)) &= P(y \mid \Do(x,z_1,z_2,z_3))=
\frac{
\sum_{z_2', z_5} 
P(y \mid x, z_1, z_2', z_5)\,
P(x \mid z_1, z_2', z_5)\,
P(z_2' \mid z_5)\,
P(z_5)
}{
\sum_{z_2'} 
P(x \mid z_1, z_2')\,
P(z_2')
}\\
P(y \mid \Do(x, z_1, z_2), z_3) &= P(y \mid \Do(z_1, z_2), x,z_3)
= P(y \mid \Do(z_1, z_2,z_3), x)\\
&=
\frac{
\sum_{z_3', z_5, z_2', z_4}
P(y \mid z_3', z_5, z_2', z_1, x, z_4)\,
P(z_4 \mid z_3', z_5, z_2', z_1, x)\,
P(x \mid z_3', z_5, z_2', z_1)\,
P(z_2' \mid z_3', z_5)\,
P(z_5 \mid z_3')\,
P(z_3')
}{
\sum_{z_3', z_5, z_2', z_4, y'}
P(y' \mid z_3', z_5, z_2', z_1, x, z_4)\,
P(z_4 \mid z_3', z_5, z_2', z_1, x)\,
P(x \mid z_3', z_5, z_2', z_1)\,
P(z_2' \mid z_3', z_5)\,
P(z_5 \mid z_3')\,
P(z_3')
}\\
P(y \mid \Do(x, z_2), z_3)
&=
\frac{
\sum_{z_5, z_2', z_4}
P(y \mid z_5, z_2', z_1, x, z_4)\,
P(z_4 \mid z_5, z_2', z_1, x)\,
P(x \mid z_5, z_2', z_1)\,
P(z_2' \mid z_5)\,
P(z_5)
}{
\sum_{z_5, z_2', z_4, y'}
P(y' \mid z_5, z_2', z_1, x, z_4)\,
P(z_4 \mid z_5, z_2', z_1, x)\,
P(x \mid z_5, z_2', z_1)\,
P(z_2' \mid z_5)\,
P(z_5)
}\\
P(y \mid \Do(z_1), x)
&=
\frac{
\sum_{z_3, z_5, z_2, z_4}
P(y \mid z_3, z_5, z_2, z_1, x, z_4)\,
P(z_4 \mid z_3, z_5, z_2, z_1, x)\,
P(x \mid z_3, z_5, z_2, z_1)\,
P(z_2 \mid z_3, z_5)\,
P(z_5 \mid z_3)\,
P(z_3)
}{
\sum_{z_3, z_5, z_2, z_4, y'}
P(y' \mid z_3, z_5, z_2, z_1, x, z_4)\,
P(z_4 \mid z_3, z_5, z_2, z_1, x)\,
P(x \mid z_3, z_5, z_2, z_1)\,
P(z_2 \mid z_3, z_5)\,
P(z_5 \mid z_3)\,
P(z_3)
}\\
P(y \mid \Do(z_1, z_2), x)
&=
\frac{
\sum_{z_3, z_5, z_2', z_4}
P(y \mid z_3, z_5, z_2', z_1, x, z_4)\,
P(z_4 \mid z_3, z_5, z_2', z_1, x)\,
P(x \mid z_3, z_5, z_2', z_1)\,
P(z_2' \mid z_3, z_5)\,
P(z_5 \mid z_3)\,
P(z_3)
}{
\sum_{z_3, z_5, z_2', z_4, y'}
P(y' \mid z_3, z_5, z_2', z_1, x, z_4)\,
P(z_4 \mid z_3, z_5, z_2', z_1, x)\,
P(x \mid z_3, z_5, z_2', z_1)\,
P(z_2' \mid z_3, z_5)\,
P(z_5 \mid z_3)\,
P(z_3)
}\\
P(y \mid \Do(x, z_1), z_3) 
&=
P(y \mid \Do(z_1), x, z_3) 
= P(y \mid \Do(x), z_3) 
=
P(y \mid \Do(z_1, z_3), x)\\
&=
\frac{
\sum_{z_3', z_5, z_2, z_4}
P(y \mid z_3', z_5, z_2, z_1, x, z_4)\,
P(z_4 \mid z_3', z_5, z_2, z_1, x)\,
P(x \mid z_3', z_5, z_2, z_1)\,
P(z_2 \mid z_3', z_5)\,
P(z_5 \mid z_3')\,
P(z_3')
}{
\sum_{z_3', z_5, z_2, z_4, y'}
P(y' \mid z_3', z_5, z_2, z_1, x, z_4)\,
P(z_4 \mid z_3', z_5, z_2, z_1, x)\,
P(x \mid z_3', z_5, z_2, z_1)\,
P(z_2 \mid z_3', z_5)\,
P(z_5 \mid z_3')\,
P(z_3')
}.
\end{align*}


\end{document}